\documentclass[11pt, margin=1in]{article}
\usepackage[toc,page]{appendix}
\usepackage{yfonts}
\usepackage{bbm}

\usepackage{latexsym}
\usepackage{amsmath}
\usepackage{amssymb}
\usepackage{amsthm}
\usepackage{epsfig}
\usepackage{xcolor}
\usepackage{graphicx}
\usepackage{bm}
\usepackage[shortlabels]{enumitem}
\usepackage{extarrows}

\usepackage[top=1in, bottom=1in, left=1in, right=1in]{geometry}

\renewcommand{\P}{\mathbb{P}}
\newcommand{\E}{\mathbb{E}}

\newcommand{\cN}{\mathcal{N}}

\newcommand{\R}{\mathbb{R}}

\newcommand{\N}{\mathbb{N}}
\renewcommand{\S}{\mathbb{S}}

\newcommand{\eps}{\varepsilon} 

\def\id{{\mathbf I}}

\newcommand{\<}{\langle}
\renewcommand{\>}{\rangle}
\newcommand{\sign}{\text{sign}}

\newcommand{\op}{{\rm op}}
\newcommand{\ones}{\bm{1}}

\def\sT{{\mathsf T}}
\def\bzero{{\boldsymbol 0}}

\newtheorem{theorem}{Theorem}
\newtheorem*{theorem*}{Theorem}
\newtheorem{lemma}{Lemma}

\newtheorem{proposition}{Proposition}

\newtheorem{corollary}{Corollary}

\theoremstyle{definition}

\newtheorem{remark}{Remark}

\def\cL{\mathcal{L}}

\def\hbSigma{\boldsymbol{\hat{\Sigma}}}

\def\bb{{\boldsymbol b}}

\def\bK{{\boldsymbol K}}

\def\bpsi{{\boldsymbol \psi}}

\def\bSigma{\boldsymbol{\Sigma}}

\def\ev{u}

\def\bphi{\boldsymbol{\phi}}

\def\bPhi{\boldsymbol{\Phi}}

\def\normal{{\sf N}}

\def\bv{{\boldsymbol v}}

\def\bZ{{\boldsymbol Z}}

\def\bM{{\boldsymbol M}}

\def\Ball{{\sf B}}

\def\bH{{\boldsymbol H}}
\def\obH{\overline{\boldsymbol H}}
\def\bX{{\boldsymbol X}}

\def\bw{{\boldsymbol w}}
\def\de{{\rm d}}
\def\bx{{\boldsymbol x}}
\def\by{{\boldsymbol y}}
\def\bW{{\boldsymbol W}}
\def\ba{{\boldsymbol a}}
\def\cF{{\mathcal F}}
\def\Unif{{\rm Unif}}

\def\bz{{\boldsymbol z}}

\def\He{{\rm He}}
\def\cE{{\mathcal E}}

\def\bxi{{\boldsymbol \xi}}

\def\be{{\boldsymbol e}}
\def\bu{{\boldsymbol u}}

\def\bA{{\boldsymbol A}}

\def\blambda{{\boldsymbol \lambda}}

\def\cM{{\mathcal M}}

\def\cV{{\mathcal V}}

\def\bPsi{{\boldsymbol \Psi}}

\def\cZ{\mathcal{Z}}

\def\reals{{\mathbb R}}

\def\cX{\mathcal{X}}
\def\ones{{\boldsymbol 1}}

\def\bb{{\boldsymbol b}}

\def\hblambda{\widehat \blambda}

\def\hy{\hat{y}}
\def\hmu{\hat{\mu}}

\def\ophi{\overline \phi}

\def\obphi{\overline \bphi}

\def\hblambda{{\hat \blambda}}
\def\hf{\hat f}
\def\ha{\hat a}
\def\hba{\hat {\boldsymbol a}}
\def\rz{r_0}

\def\Q{\mathbb{Q}}

\def\NP{\mathsf{NP}}
\def\BPP{\mathsf{BPP}}
\def\Pclass{\mathsf{P}}
\def\Bsf{\mathsf{B}}
\def\ubx{\underline{\bx}}

\usepackage{hyperref}
\hypersetup{
    colorlinks,
    linkcolor={blue!80!black},
    citecolor={green!50!black},
    urlcolor=blue,
}
\colorlet{linkequation}{blue}

\begin{document}

\title{Minimum complexity interpolation in random features models}
\author{Michael Celentano\thanks{Department of Statistics, University of California, Berkeley},\;\;\;
  Theodor Misiakiewicz\footnotemark[1], \;\;\;Andrea
  Montanari\footnotemark[1] \thanks{Department of Electrical Engineering, Stanford
    University},}

\maketitle
\thispagestyle{empty}

\begin{abstract}
  Despite their many appealing properties, kernel methods are heavily affected by the curse of dimensionality.
  For instance, in the case of inner product kernels in $\reals^d$, the Reproducing Kernel Hilbert Space (RKHS) norm is often
  very large for functions that depend strongly on a small subset of directions (ridge functions).
  As a consequence, such functions are difficult to learn using kernel methods.
  This observation has motivated the study of generalizations of kernel methods, whereby the RKHS norm ---which is equivalent to a weighted $\ell_2$ norm--- is replaced by a weighted functional $\ell_p$ norm, which we refer to as $\cF_p$ norm. Unfortunately, tractability of these approaches is unclear. The kernel trick is not available and minimizing these norms requires solving an infinite-dimensional convex problem.

  We study random features approximations to these norms and show that, for $p>1$, the number of
  random features required to approximate the original learning problem is upper bounded by a polynomial in
  the sample size. Hence, learning with $\cF_p$ norms is tractable in these cases. We introduce a proof
  technique based on uniform concentration in the dual, which can be of broader interest in the study of overparametrized models.
  For $p= 1$, our guarantees for the random features approximation break down. We prove instead that learning with the $\cF_1$ norm is $\NP$-hard under a randomized reduction based on the problem of learning halfspaces with noise. 
\end{abstract}

\tableofcontents

\section{Introduction}
\label{sec:Introduction}
  
\subsection{Background: Kernel methods and the curse of dimensionality}
\label{sec:background_introduction}

Kernel methods are among the most fundamental tools in machine learning. Over the last
two years they attracted renewed interest because of their connection with neural networks in the
linear regime (a.k.a. neural tangent kernel or lazy regime) \cite{jacot2018neural,liang2020just,liang2019risk,ghorbani2019linearized}. 

Consider a general covariates space, namely a probability space $(\cX,\P)$
(our results will concern the case $\cX=\reals^d$, but it is useful to start from a more general viewpoint). A reproducing kernel
Hilbert space (RKHS) is usually constructed starting from a positive semidefinite kernel $K:\cX\times\cX\to\reals$.
Here it will be more convenient to start from a weight space, i.e.\ a probability space
$(\cV,\mu)$, and a featurization map $\phi ( \cdot;\bw ) :\cX\to\reals$ parametrized by the weight
vector\footnote{A sufficient condition for this construction to be well defined is $\phi\in L^2(\P\otimes\mu)$.}
$\bw\in\cV$. The RKHS is then defined as the space of functions of the form
\begin{equation}\label{eq:non-random-predictor}
  \hat f ( \bx ; a ) = \int_{\cV} \phi ( \bx ;\bw) \, a (\bw)\,  \mu (\de \bw)\, ,
\end{equation}
 with $a\in L^2(\cV;\mu)$.
To give a concrete example, we might consider $\cV=\reals^d$, and $\phi(\bx;\bw) =\sigma(\<\bw,\bx\>)$:
this is the featurization map arising from two-layers neural networks with random first layer weights.

  The radius-$R$ ball in this space is defined as
\begin{equation}
  \cF_2 (R) = \Big\{ \hat f ( \bx ; a ) =\int_{\cV} \phi ( \bx ;\bw) \, a (\bw)\,  \mu (\de \bw): \;\;
  \int_{\cV}  |a(\bw)|^2 \mu (\de \bw)   \leq R^2\Big\} \, .\label{eq:HilbertBall}
  \end{equation}
  This construction is equivalent to the more standard one,  with associated kernel $K(\bx_1,\bx_2):= \int_{\cV} \phi ( \bx_1 ;\bw) \phi ( \bx_2 ;\bw) \, \mu(\de\bw)$. Vice versa, for any positive semidefinite kernel $K$, we can construct
  a corresponding featurization map $\phi$, and proceed as above.

  It is a basic fact that, although $\cF_2 (R)$ is infinite-dimensional, learning in $\cF_2(R)$ can be performed 
  efficiently \cite{shalev2014understanding}: it is useful to recall the reason here.
  Consider a supervised learning setting in which we are given $n$ samples $(y_i, \bx_i)$, $i \leq n$, with
  $y_i\in \R$ and  $\bx_i\in\cX$. Given a convex loss $\ell:\R\times \R\to\R$,  the RKHS-norm
  regularized empirical risk minimization problem reads:
  \begin{align}
    \mbox{\rm minimize}\;\;\;
    \sum_{i=1}^n\ell\big(y_i,\hf(\bx_i;a)\big)+\lambda \int_{\cV}  |a(\bw)|^2 \mu (\de \bw)\, .\label{eq:GeneralKernel}
    \end{align}
    Conceptually, this problem can be solved in two steps: $(1)$~find the minimum RKHS norm
    subject to interpolation constraints $\hf(\bx_i;a)=\hy_i$, and $(2)$~minimize the sum of this quantity
    and the empirical loss $\sum_{i=1}^n\ell\big(y_i,\hy_i)$. Since step $(2)$ is convex and finite-dimensional, the
    critical problem is the interpolation problem:
  \begin{align*}
    \mbox{\rm minimize}\;\;\;& \int_{\cV}  |a(\bw)|^2 \mu (\de \bw)\, ,\\
     \mbox{\rm subj. to}\;\;\;& \hf(\bx_i;a)=\hy_i,\;\;\; \forall i\le n\, . 
    \end{align*}
    While this is an infinite-dimensional problem, convex duality (the `representer theorem')
    guarantees that the solution belongs to  a fixed
    $n$-dimensional subspace, $a_*(\bw) = \sum_{i=1}^nc_i\phi(\bx_i;\bw)$.

    Unfortunately, kernel methods suffer from the curse of dimensionality.
To give a simple example,
    consider $\bx \sim \Unif( \S^{d-1} (\sqrt{d}))$, $\bw \sim \Unif (\S^{d-1} (1))$, and 
    $\phi (\bx;\bw) = \sigma( \< \bw , \bx \>)$ with $\sigma:\reals\to\reals$ a non-polynomial activation function.
    Consider noiseless data $y_i=f_*(\bx_i)$, with $f_*(\bx) = \sigma(\<\bw_*,\bx\>)$ for a fixed
    $\bw_* \in \S^{d-1} (1)$.   Then,  $\|f_*\|_K=\infty$ (with $\|\,\cdot\,\|_K$ denoting the RKHS norm).
    Correspondingly, \cite{yehudai2019power,ghorbani2019linearized} show that for any fixed  $k$, if
    $n \leq d^k$, then any kernel method of the form \eqref{eq:GeneralKernel} returns $\hf$
    with $\|\hf-f_*\|_{L^2(\P)}$  bounded away from zero. 
    
    The curse of dimensionality suggests to seek functions of the form  \eqref{eq:non-random-predictor}
    where $a$ is sparse, in a suitable sense \cite{bach2017breaking}.  In this paper we consider a generalization in which the RKHS ball of Eq.~\eqref{eq:HilbertBall} is replaced by
\begin{equation}\label{eq:PBall}
    \cF_p (R) = \Big\{  \hf(\bx;a) =\int_{\cV} \phi ( \bx ;\bw) \, a (\bw)\,  \mu (\de \bw): \;\; \Big( \int_{\cV}  |a(\bw)|^p \mu (\de \bw) \Big)^{1/p}  \leq R \Big\} \, .
  \end{equation}
  For $p\in [1,2)$ this comprises a richer function class than the original $\cF_2(R)$, 
  since $\cF_2 (R) \subset \cF_p (R) \subset \cF_1 (R)$, and it is easy to see that the inclusion is strict.
  The case $\rho(x)=|x|$ is also known as `convex neural network' \cite{bach2017breaking}.
  
  Although the penalty $\int_{\cV}  |a(\bw)|^p \mu (\de \bw)$ is convex, it is far from clear that
  the learning problem is tractable. Indeed for $p\neq 2$, the classical representer theorem does not hold anymore
and we cannot reduce the infinite dimensional problem to solving a finite dimensional quadratic program (see Appendix \ref{sec:Representer} for a representer-type theorem for $p>1$).

  By the same reduction discussed above, it is sufficient to consider the following
  \emph{minimum-complexity interpolation problem}: 
  \begin{align}\label{eq:MinCplx}
    \mbox{\rm minimize}\;\;\;& \int_{\cV}  \rho(a(\bw)) \mu (\de \bw)\, ,\\
     \mbox{\rm subj. to}\;\;\;& \hf(\bx_i;a)=y_i,\;\;\; \forall i\le n\, . \nonumber
    \end{align}
    (We denote the values to be interpolated by $y_i$ instead of $\hy_i$ because we will focus on the interpolation problem hereafter.)
    We will take $\rho(x)$ to be a convex function
    minimized at $x=0$. 
    We establish two main results. 
    First, we establish tractability for a subset
    of strictly convex penalties $\rho$ which include, as special cases, $\rho(x) = |x|^p/p$, $p>1$. Our approach is based on a random features approximation which we discuss next. Second,
    we establish $\NP$-hardness under randomized reduction for the case $\rho(x) = |x|$. 

    \subsection{Random features approximation}

   We sample independently $\bw_j \sim \mu$, $j \leq N$, and fit a model 
\begin{equation}\label{eq:non-random-fN}
        \hat f_N ( \bx ; \ba ) = \frac{1}{N} \sum_{j = 1}^N a_j \phi ( \bx;\bw_j )\, .
      \end{equation}
  We determine the coefficients $\ba=(a_i)_{i\le N}$  by solving the interpolation problem
  \begin{align}\label{eq:opt_finite}
      \mbox{\rm minimize}\;\;\;& \sum_{j = 1}^N \rho (a_j )\, ,\\
      \mbox{\rm subj. to}\;\;\;& \hf_N(\bx_i;\ba)=y_i,\;\;\; \forall i\le n\, .\nonumber
  \end{align}
  Notice that this is equivalent to replacing the measure $\mu$ in Eq.~\eqref{eq:MinCplx}
  by its empirical version $\hmu_N=N^{-1}\sum_{i=1}^N\delta_{\bw_i}$.
  Borrowing the terminology from neural network theory, we will refer to \eqref{eq:opt_finite}
  as the ``finite width'' problem, and to the original problem \eqref{eq:MinCplx} as the
  ``infinite width'' problem. We will denote by $\hba_N$ the solution of  the finite width problem \eqref{eq:opt_finite}
and by $\ha$ the solution of the infinite width problem  \eqref{eq:MinCplx}.

  Our main result establishes that,  under suitable assumptions on the penalty $\rho$ and the
  featurization map $\phi$, the random features approach provides a good approximation of the
  minimum complexity interpolation problem \eqref{eq:MinCplx}. Crucially, this is achieved with a number of features that is polynomial
  in the sample size $n$.
 In particular we show that for $\rho(x)=|x|^p/p$, $1<p\leq 2$,
 setting $Q=p/(p-1)$ (and assuming $\|\by\|_2\le C\sqrt{n}$)
\begin{equation}
\| \hat f_N ( \cdot ; \hat \ba_N ) - \hat f ( \cdot ; \hat a) \|_{L^2} \leq C \Big( \sqrt{\frac{n^2 \log(N)}{N}} \vee \frac{(n \log N)^{(Q+1)/2)}}{N} \Big) \,\, .\label{eq:InformalBound}
\end{equation}
Hence, for $N \gg (n\log n )^2 \vee (n \log n)^{(Q+1)/2}$, the
random features approach yields a function that is a good approximation of the original problem
\eqref{eq:MinCplx}.

Let us emphasize that the scaling of the number of random features $N$ in our
bound is not optimal. In particular, for the RKHS case $p=2$, the above result requires $N\gg n^2$,
while we know from \cite{mei2021generalization} that $N\gg n \log n$ features are often sufficient.
We also note that the exponent $Q$ in the polynomial diverges as $p\to1$. Hence, our results
do not guarantee tractability for the case $p=1$.  Indeed this is a fundamental limitation:
as discussed in Section \ref{sec:NPhardness}, a bound such as Eq.~\eqref{eq:InformalBound} with finite $Q$ cannot hold for $p=1$, under some standard hardness assumptions. We show instead that no polynomial time algorithm is guaranteed to achieve accuracy $n^{-C}$ for some absolute constant $C$ in Eq.~\eqref{eq:InformalBound}. This hardness result is based on a randomized reduction from the problem of learning halfspaces with noise, which was proved to be $\NP$-hard in \cite{feldman2006new,guruswami2009hardness}. 

\subsection{Dual problem and its concentration properties}

 Our proof that the random features model $\hat f_N ( \cdot ; \hat \ba_N )$ is
 a good approximation of the infinite width model $\hat f ( \cdot ; \hat a)$ (cf.\ Eq.~\eqref{eq:InformalBound})
 is based on a simple approach which is potentially of independent interest.
 We notice that, while the optimization problems \eqref{eq:MinCplx} and \eqref{eq:opt_finite} are
 overparametrized and hence difficult to control, their duals are underparametrized and can be studied using
 uniform convergence arguments.

 Let $\rho^*$ be the convex conjugate of $\rho$:
\[
\rho^* ( x) = \sup_{y \in \R} \{ xy - \rho(y) \} \, .
\]
Then, the dual problems of \eqref{eq:MinCplx} and \eqref{eq:opt_finite} are given ---respectively--- by
the following optimization problems over $\blambda \in \R^n$:
\begin{align}
  & \mbox{maximize}\;\;\;\;
    \< \blambda , \by \> -  \int_\cV \rho^* ( \< \bphi_{n} (\bw) ,  \blambda \> ) \mu (\de \bw) \, , \label{eq:dual_infinite}\\
  & \mbox{maximize}\;\;\;\;\< \blambda , \by \> - \frac 1N \sum_{j = 1}^N \rho^* ( \< \bphi_{n} (\bw_j) ,  \blambda \> )  \, . \label{eq:dual_finite}
\end{align}
Here we denoted $\by = (y_1 , \ldots , y_n ) \in \R^n$ the vector of responses and by $\bphi_n : \cV \to \R^n$
the evaluation of the feature map at the $n$ data points $\bphi_n (\bw) = ( \phi( \bx_1;\bw) , \ldots , \phi (\bx_n;\bw) ) \in \R^n$.

We will prove that, under suitable assumptions on the penalty $\rho$, the optimizer of
the finite-width dual \eqref{eq:dual_finite} concentrates around the optimizer of the infinite-width dual 
\eqref{eq:dual_infinite}. Our results hold conditionally on the realization of the data $\{(\bx_i,y_i)\}_{i\le n}$
and instead exploit the randomness of the weights $\{\bw_i\}_{i\le N}$.

The rest of the paper is organized as follows. After briefly discussing related work in Section \ref{sec:Related},
we state our assumptions and results for strictly convex penalties in Section \ref{sec:Convergence}. In Section \ref{sec:NPhardness}, we show that the problem with $\cF_1$ norm is $\NP$-hard under randomized reduction. In Section \ref{sec:Examples}, we describe a few
examples in which we can apply our general results. The proof of our main result is outlined in Section \ref{sec:ProofMain},
with most technical  work deferred to the appendices.

\section{Related work}
\label{sec:Related}
  
Random features methods were first introduced as a randomized approximation
to RKHS methods \cite{rahimi2008random,balcan2006kernels}. Given a kernel $K:\cX\times \cX\to\reals$,
the idea of \cite{rahimi2008random} was to replace it by a low rank approximation
\begin{align}
K^{(N)}(\bx_1,\bx_2):= \frac{1}{N}\sum_{i=1}^N\phi(\bx_1;\bw_i)\phi(\bx_2;\bw_i)\, ,\label{eq:KernelRF}
\end{align}
where the random features are such that $\E\{\phi(\bx_1;\bw)\phi(\bx_2;\bw)\}=K(\bx_1,\bx_2)$.
Several papers prove bounds on  the test error of random features methods and compare them with the corresponding kernel approach \cite{rahimi2009weighted, rudi2017generalization, ma2020towards, ghorbani2019linearized, mei2019generalization,ghorbani2020neural,mei2021generalization}.
In particular,  \cite{mei2021generalization}  proves that ---under certain concentration assumptions on the random features--- if $N\ge n^{1+\eps}$ for some $\eps>0$, then the random features approach has nearly the same
test error as the corresponding RKHS method.

Notice that the idea of approximating the kernel via random features implicitly
selects a specific regularization, in our notation $\rho(x) = x^2$. Indeed, for any other regularization, the
predictor $\hf(\bx;\hat a)$ does not depend uniquely on the kernel.
An alternative viewpoint regards the random features model \eqref{eq:non-random-fN} as
a two-layer neural network with random first layer weights. It was shown in \cite{bartlett1998sample}
that the generalization properties of a two-layer neural network can be controlled in terms of the sum of the norms of
the second-layer weights. This opens the way to considering infinitely wide two-layer networks as per
Eq.~\eqref{eq:non-random-predictor}, with $\int |a(\bw)|\, \mu(\de\bw)<\infty$.
These networks represent functions in the class $\cF_1$.

The infinite-width limit  $\cF_1$ was considered in \cite{bengio2006convex}, which propose an incremental algorithm
to fit functions with an increasing number of units. Their approach however has no global optimality guarantees.
Generalization error bounds within $\cF_1$ were proved in \cite{bach2017breaking},
which demonstrates the superiority of this
approach over RKHS methods, in particular to fit ridge functions or other classes of functions that depend
strongly on a low-dimensional subspace of $\reals^d$. The same paper also develops an optimization algorithm based on a conditional-gradient approach. However, in high dimension each iteration of this algorithm requires solving a potentially hard problem.

A line of recent work shows that ---in certain overparametrized training regimes--- neural networks are indeed well
approximated by random features models obtained by linearizing the network around its
(random) initialization. A subset of references include \cite{jacot2018neural, li2018learning, du2018gradient, oymak2019towards,chizat2019lazy}. The implicit bias induced by gradient descent on these
models is a ridge (or RKHS) regularization. However, the bias can change if either the algorithm or the model parametrization is changed \cite{neyshabur2017geometry,gunasekar18a}.

We finally notice that several recent works study the impact of
changing regularization in high-dimensional linear regression,
focusing on the interpolation limit
\cite{liang2020precise,mignacco2020role,chinot2020minimum}.
None of these works addresses the main problem studied here, that is approximating the fully
nonparametric model \eqref{eq:non-random-predictor}.

Our work is a first step towards
understanding the effect of regularization in random features models. It implies that, for certain
penalties $\rho$, as soon as $N$ is polynomially large in $n$, studying the random features model reduces to studying the corresponding infinite width model.

\section{Convergence to the infinite width limit}
\label{sec:Convergence}

\subsection{Setting and dual problem}

We consider a slightly more general setting than the one discussed in
the introduction,  whereby we allow the featurization map  functions to be randomized. More precisely, conditional on the
data $(\bx_i)_{i\le n}$ and weight vectors $(\bw_j)_{j\le N}$, the feature values $\{\phi(\bx_i;\bw_j)\}_{i\le n,j\le N}$ are independent
random variables. Explicitly, such randomized features
can be constructed by letting $\phi(\bx;\bw,z)$ depend on an additional variable $z\in\cZ$,
and setting (with an abuse of notation)
$\phi(\bx_i;\bw_j)=\phi(\bx_i;\bw_j,z_{ij})$ for $\{z_{ij}\}_{i\le
  n,j\le N}$ a collection of independent random variables with common law $z_{ij}\sim \nu$
  (with $\nu$ a probability distribution over $\cZ$).  Without loss of generality\footnote{For instance, this is the case as long as the weights take values in a Polish space, by which the conditional probabilities $\P(\phi(\bx;\bw)\in S|\bw)= P(S;\bw)$ exist.} we can assume $z_{ij}\sim \Unif([0,1])$.

We denote by $\ophi(\bx; \bw) = \E_{\phi} [ \phi (\bx;\bw) ] = \E_{z} [ \phi (\bx;\bw,z) ]$ the expectation of the features with respect
to this additional randomness. In what follows, we will omit to write the dependency on $z$
explicitly, unless required for clarity. The additional freedom afforded by randomized features is useful in multiple scenarios:
\begin{itemize}
\item  We only have access to noisy measurements $\phi(\bx_i;\bw_j)$ of the true features $\ophi(\bx_i;\bw_j)$.
\item We deliberately introduce noise in the featurization
  mechanism. This is known to have a regularizing effect \cite{bishop1995training}.
\item We do not explicitly introduce noise in the featurization mechanism. However, the noiseless
  featurization mechanism is asymptotically equivalent to one
  in which noise has been added. Examples of this type are given in \cite{mei2019generalization,gerace2020generalisation,hu2020universality}
\end{itemize}

At prediction time we use the average features\footnote{Our results do not change significantly
  if we use randomized features also at prediction time.}
\begin{equation}
\hat f ( \bx ; a ) = \int_{\cV} \ophi ( \bx;\bw ) a (\bw) \, \mu (\de \bw) \, , \qquad \hat f_N(\bx;\ba) = \frac1N \sum_{j=1}^N a_j\ophi(\bx;\bw_j) \, .
\end{equation}
The dual of the finite-width problem \eqref{eq:opt_finite} is given by the problem \eqref{eq:dual_finite}, for which we introduce the following notations:
\begin{equation}\label{eq:dual-finite}
\begin{aligned}
    F_N ( \blambda) :=&~\< \blambda , \by \> - \frac 1N \sum_{j = 1}^N \rho^* ( \< \bphi_{n, j}  ,  \blambda \> )   \, , \\
    \hat \blambda_N = &~  \arg\max_{\blambda \in \R^n} F_N (\blambda) \, .
\end{aligned}
\end{equation}
Notice that $\bphi_{n,j} = \bphi_n (\bw_j) = ( \phi( \bx_1;\bw_j) , \ldots , \phi_{\bw} (\bx_n;\bw_j) )^{\sT}$ is now a random vector.
In the case of random features, the dual of the infinite width problem \eqref{eq:MinCplx} has to be modified with respect to
\eqref{eq:dual_infinite}, and takes instead the form 
\begin{equation}\label{eq:dual-infinite}
\begin{aligned}
    F ( \blambda) :=&~\< \blambda , \by \> - \int_\cV \E_\phi [\rho^* ( \< \bphi_{n}(\bw) ,  \blambda \> )]  \mu( \de \bw) \, , \\
    \hat \blambda = &~  \arg\max_{\blambda \in \R^n} F (\blambda) \, .
\end{aligned}
\end{equation}
Note the expectation $\E_{\phi}$   with respect to the randomness in the features (equivalently, this is the expectation with respect to the independent randomness
$z$, which is not noted explicitly). The most direct way to see that this is the correct infinite width dual  is to notice that this is obtained
from Eq.~\eqref{eq:dual-finite} by taking expectation with respect to the weights $\{\bw_j\}_{j\le N}$ and the features randomness.
We will further discuss the connection between \eqref{eq:dual-infinite} and the infinite width primal problem in Section \ref{sec:Primal} below.

In order to discuss the dual optimality conditions, let $s: \R \to \R$ be the derivative of the
convex conjugate of $\rho$, i.e., $s(x) = (\rho^*) ' (x) = (\rho ' )^{-1} (x)$. Since $\rho$ is assumed to be strictly convex, $s(x)$ exists and is continuous and non-decreasing.
With these definitions, the dual optimality condition reads
\begin{align}
  \by = \frac 1N \sum_{j=1}^N \bphi_{n,j} \, s(\<\bphi_{n,j},\hblambda_N\>)\, .
\end{align}
The primal solution is then given by $(\hat \ba_N)_{j} =
s ( \<\bphi_{n,j},\hat \blambda_N \>)$ and the resulting predictor is
\[
 \hf_N ( \bx ; \hba_N ) = \frac{1}{N} \sum_{j=1}^N \ophi( \bx;\bw_j )\, s ( \< \bphi_{n,j} , \hat \blambda_N \> ) \, .
\]
In the following, with an abuse of notation, we will write $\hf_N(\bx;\blambda) = N^{-1}\sum_{j=1}^N \ophi( \bx;\bw_j )\, s ( \< \bphi_{n,j} , \blambda \> )$
for the model at dual parameter $\blambda$. The corresponding infinite width predictor reads
\begin{align}
\hat f( \bx ; \blambda)  : =  \int_{\cV} \ophi ( \bx;\bw ) \,\E_\phi [s ( \< \bphi_n(\bw) , \blambda \> )]\,  \mu (\de \bw) \, .\label{eq:InfWidthPredictorDual}
\end{align}

\subsection{General theorem}

We will show that conditional on the realization of $\by,\bX$,
the distance  between the infinite width interpolating model $\hf(\,\cdot\, ; \hblambda)$ and the finite width one
$\hf_N(\,\cdot\, ;\hblambda_N)$ is small with high probability as soon as $N$ is large enough.
Throughout this section, $\by,\bX$ are viewed as fixed, and we assume certain conditions to hold on
the distribution of the features $\phi_{n} (\bw)$.
In Section \ref{sec:Examples} we will check that these conditions hold for a few models of interest,
for typical realizations of $\by,\bX$.

We first state our assumptions on the feature distribution. We define  the whitened features
\begin{align}
  \bpsi_{n,j} := \bK_n^{-1/2}\bphi_{n,j},\;\;\;\;\;\;\;
  \bK_n := \int \E_{\phi}[\bphi_{n} (\bw)\bphi_{n} (\bw)^\top]\, \mu(\de\bw)\, .
\end{align}
Here  expectation is over the randomization in the features, and  $\bK_n\in \reals^{n \times n}$ is the empirical kernel matrix.
By construction, $\bpsi_{n,j}$ are isotropic: $\E_{\bw,\phi}[\bpsi_{n,j} \bpsi_{n,j}^\top] = \id_n$.

We will assume the following conditions to hold for the featurization map $(\bx,\bw)\mapsto \phi (\bx ; \bw )$ and the penalty $x\mapsto \rho(x)$. Throughout we will follow the convention of denoting by Greek letters
constants that we track of in our calculations, and by $C,c, \rz$ constants that we do not track
(and therefore our statements depend on an unspecified way on the latter). We also recall that a random vector $\bZ$ is
$\gamma^2$-subgaussian if $\E[\exp(\<\bv,\bZ\>)]\le \exp(\gamma^2\|\bv\|^2/2)$ for all vectors $\bv$.
\begin{description}
\item[FEAT1] \emph{(Sub-gaussianity)} For some $0 < \tau \leq n^C$ and any fixed $\| \bx \|_2 \leq \rz \sqrt{d}$,
  $\ophi ( \bx ; \bw )$ is $\tau^2$-sub-Gaussian when $\bw\sim\mu$. Further, the feature vector $\bpsi_n := \bK_n^{-1/2}
  [\phi (x_1;\bw,z_1),\dots, \phi (x_n;\bw,z_n)]^{\sT}$ is
  $\tau^2$-sub-Gaussian when $\bw\sim \mu$, $(z_i)_{i\le n} \sim_{iid}\nu$. 
  Without loss of generality we assume $\tau\ge 1$.
    \item[FEAT2] \emph{(Lipschitz continuity)}
For any $\bw \in \cV$, assume that $\ophi (\bx ; \bw)$ is $L(\bw)$-Lipschitz with respect to $\bx$ and $\bar \phi (\bzero ; \bw) \leq L(\bw)$, where $\P(L(\bw) \geq t) \leq Ce^{-t^2/(2\tau^2)}$, $\tau\ge 1$.

         \item[FEAT3] \emph{(Small ball)}  There exists $\eta \geq n^{-C}$ such that
\begin{align}
  \inf_{\|\bu\|_2=1,\|\bv\|_2=1}\P\big(|\<\bu,\bpsi_n\>|\ge\eta, \, |\<\bv,\bpsi_n\>|\ge\eta  \big)\ge c\, ,
\end{align}
for some strictly positive constants $C,c$. By union bound, this is implied by the stronger condition
    \begin{equation*}
        \sup_{\| \bv \|_2 = 1} \P\big( | \< \bv , \bpsi_n \> | \le \eta \big) \le  \frac{1}{2}(1-c) \, .
      \end{equation*}
      Without loss of generality, we will assume $\eta\le 1$.
    \item[PEN] \emph{(Polynomial growth)} 
    We assume that  $\rho$ is strictly convex and minimized at 0, so that $s(x)$ is continuous and $s(0) = 0$.
    Because $s$ is non-decreasing, we have that $s(x)$ has a derivative almost everywhere.
    We assume there exists $Q_1,Q_2,q_1,q_2 > 1$ and $C,c > 0$ such that for $|x_1|,|x_2| > 0$,
    \begin{gather*}
        c|s(x_2)/x_2| \, (|x_1/x_2|^{q_1-2} \wedge |x_1/x_2|^{q_2-2}) \leq s'(x_1) \leq C|s(x_2)/x_2| \, (|x_1/x_2|^{Q_1-2} \vee |x_1/x_2|^{Q_2-2}),
    \end{gather*}
    and 
    \begin{gather}\label{eq:poly-growth}
        c|s(x_2)| \, (|x_1/x_2|^{q_1-1} \wedge |x_1/x_2|^{q_2-1}) \leq |s(x_1)| \leq C|s(x_2)| \, (|x_1/x_2|^{Q_1-1}\vee |x_1/x_2|^{Q_2-1}).
    \end{gather}
\end{description}

If $s'(x) = (\rho^*)''(x)$ is unbounded as $x\to 0$, we will require a stronger assumption FEAT3' which will ensure that $\nabla^2 F (\blambda)$ exists and is finite.
\begin{description}
         \item[FEAT3'] \emph{(Small ball)} 
    For every $r \in (-1,0)$ and every $\|\bv\|_2 = 1$, $\E[|\langle \bpsi_n,\bv\rangle|^r] \leq C_r\eta^r  < \infty$, $\eta\le 1$.
  \end{description}

  \begin{remark}\label{lem:p-norms}
 Assumption PEN implies that $s(x)$  is locally Lipschitz away from $0$.
It further implies that $s(x)$ and its derivatives are upper and lower bounded by polynomials in $x$.
The most important example is provided by $p$-norms with $1 < p < \infty$,
in which case we set $\rho(x) = \frac1p |x|^p$ for $1 < p < \infty$. It is easy to check
that PEN is satisfied with $Q_1 = Q_2 = q_1 = q_2 = p/(p-1)$ and appropriate choices of $C,c$.
\end{remark}

Our main theorem establishes that the infinite-dimensional problem of Eq.~\eqref{eq:MinCplx} (and its generalization
to randomized featurization maps) is well-approximated by its finite random features counterpart.
For this approximation to be good, it is sufficient that the number of random features scales polynomially with the sample size $n$.
\begin{theorem}\label{thm:gen-conv-to-pop}
  Assume $\|\bx_i\|_2\le \rz\sqrt{d}$ for all $i\le n$, and let $\P$ be a probability distribution supported on
  $\Ball^d_2(\rz\sqrt{d}):=\{\bx\in\reals^d:\, \|\bx\|_2\le \rz\sqrt{d}\}$.
  Assume that conditions FEAT1, FEAT2, FEAT3, and PEN are satisfied. If $Q_1 \wedge Q_2 < 2$, further assume that FEAT3' holds. Then for any $\delta>0$, there exist constants $C',c'$ depending on the constants $C,c,\rz$ in those assumptions,  
  but not on  $\tau$ and $\eta$, and $C''(\delta)$ additionally dependent on $\delta>0$ such that the following holds.
  If $N \geq N_1\vee N_2$, $N_1:=C'(\tau^{Q \vee 2} /\eta^{q\vee 3})^2 n\log n$, $N_2:=C'(\tau^{Q \vee 2} /\eta^{q\vee 3})
  (n \log n)^{Q/2 \vee 1}$ and $n \geq c' d$, then with probability at least $1  - C' N^{-c'n}$, 
    \begin{equation}\label{eq:MainThmBound}
    \begin{aligned}
        \| \hf_N (\,\cdot\,;\hat \blambda_N) - \hf(\,\cdot\,;\hat \blambda)\|_{L_2(\P)}  \leq &~
        M (\delta, \tau , \eta ) \Big( \sqrt{\frac{n\log N}{N}} \vee \frac{(n\log N)^{Q/2}}{N} \Big) \|\bK_n^{- 1/2} \by \|_{2 } \,  ,
        \end{aligned}
    \end{equation}
    where
    \[
    M (\delta, \tau , \eta ) =  C'' (\delta)\frac{\tau^{Q+2 +(2-q'+\delta)_+  }}{\eta^{q \vee 3} } (\tau^{Q-2} \vee  \eta^{Q'-2})  \, ,
    \]
    with $Q = Q_1 \vee Q_2$, $Q' = Q_1 \wedge Q_2$, $q = q_1 \vee q_2$, $q'=q_1\wedge q_2$.
    Further, the bound holds with $\delta=0$, $C''(0)<\infty$ when $q_1=q_2\ge 2$.
  \end{theorem}
In order to interpret Theorem \ref{thm:gen-conv-to-pop}, we remark that we expect typically $\|\bK_n^{-1/2} \by \|_{2 }$ to be
of order $ \sqrt{n } $. In this case $\hat f_N (\,\cdot\, ; \hat \blambda_N)$ differs negligibly from $\hf(\,\cdot\, ; \hblambda)$ when $ N \gg (n \log (n))^{2 \vee (Q+1)/2}$.

\begin{remark}\label{rmk:FEAT3}
The most restrictive among our assumptions  are FEAT3 and FEAT3'. Both conditions imply that the infinite-width dual problem 
  of Eq.~\eqref{eq:dual-infinite} is well behaved. 
  
  In particular, condition FEAT3 ensures that the minimum eigenvalue of the rescaled Hessian
  $\bH_n(\blambda):= \bK_n^{-1/2}  \nabla^2 F ( \blambda )  \bK_n^{-1/2} $
  is bounded away from $0$, as long as $\blambda$ is bounded away from $0$ and $\infty$. Notice indeed that the Hessian is given by
  \begin{align}
  \nabla^2 F ( \blambda )  = -\E_{\bw,\phi} [ s' ( \< \blambda, \bphi_n \>)\, \bphi_n\bphi_n^{\sT}]\, .\label{eq:HessianInfinite}
  \end{align}
  For any $\bv$, with $\|\bv\|_2=1$, we have $\E\{\<\bv,\bpsi_n\>^2\}=1$, whence, using assumption FEAT3
  and Markov's
  inequality  and union bound,
  for any two unit vectors $\bu,\bv$, $\P(|\<\bu,\bpsi_n\>|\in [\eta,C], \, |\<\bv,\bpsi_n\>|\in [\eta,C])\ge c/2$ for some constant $C$.
  We then have,
\[
\begin{aligned}
\inf_{\| \bv \|_2 = 1} \lambda_{\min} ( - \bH_n(\bK_n^{-1/2}\bv)) =&~  \inf_{\| \bu \|_2 = 1 , \| \bv \|_2 = 1} \E_{\bw,\phi} [ \< \bu , \bpsi_n \>^2 s' ( \< \bv , \bpsi_n \>)]\ge  c''\eta^2\inf_{x \in [\eta , 4]} s' (x) \, .
\end{aligned}
\]

Condition FEAT3' ensures that the largest eigenvalue of the Hessian   $\nabla^2 F ( \blambda )$ is bounded for
$\|\blambda\|_2$ bounded below and above. Notice  that, from Eq.~\eqref{eq:HessianInfinite}, no such assumption is required when $s'(x)$ is bounded
as $x\to 0$.
\end{remark}

\begin{remark}
  As mentioned in the introduction, the bound in Eq.~\eqref{eq:MainThmBound} is not optimal. For instance, for the case of
  a penalty $\rho$ that is strongly convex and smooth (covered in Theorem \ref{thm:gen-conv-to-pop} by taking $Q=2$)
  this bound implies that $N\gg (n\log n)^2$
  random features are sufficient to approximate the infinite width problem.
  However, a more careful analysis yields that ---in this case--- $N\gg n\log n$ random features are sufficient.
  This is also what is established in \cite{mei2021generalization} for the case of kernel ridge regression (corresponding to $\rho(x) = x^2$).
\end{remark}

It is instructive to specialize Theorem \ref{thm:gen-conv-to-pop} to the case of $p$-norms, which is
covered by taking $\rho(x) = |x|^p/p$, $p\in (1,2]$. In this case $q_1=q_2=Q_1=Q_2=Q$, with $Q=p/(p-1)\ge 2$
and hence the formulas are simpler.
\begin{corollary}\label{coro:pnorms}
  Assume $\|\bx_i\|_2\le \rz\sqrt{d}$ for all $i\le n$, and let $\P$ be a probability distribution supported on
  $\Ball^d_2(\rz\sqrt{d}):=\{\bx\in\reals^d:\, \|\bx\|_2\le \rz\sqrt{d}\}$.
  Assume that conditions FEAT1, FEAT2, FEAT3 are satisfied, and penalty $\rho(x)=|x|^p/p$, $p\in (1,2]$. 
  Then  there exist constants $C',c'$ depending on the constants $C,c,\rz$ in those assumptions,  
 such that the following holds. If $N \geq C'[(\tau^{Q}/\eta^{Q\vee 3})^2 (n \log n) \vee (\tau^{Q}/\eta^{Q\vee 3})(n \log n)^{Q/2}]$ 
 and $n \geq c' d$, then with probability at least $1 - C' N^{-c'n}$, 
    \begin{equation}
        \| \hf_N (\,\cdot\,;\hat \blambda_N) - \hf(\,\cdot\,;\hat \blambda)\|_{L_2(\P)}
        \leq 
        C' \frac{\tau^{2Q} }{\eta^{Q \vee 3} } \Big( \sqrt{\frac{n\log N}{N}} \vee \frac{(n\log N)^{Q/2}}{N} \Big) \|\bK_n^{- 1/2} \by \|_{2 } \, .\label{eq:CoroPnorm}
    \end{equation}
    where $Q =p/(p-1)$.
  \end{corollary}
  
  Note that the exponent $Q$ on the right hand side of Eq.~\eqref{eq:CoroPnorm} diverges as $p\to 1$. Instead of being a feature of our proof technique, we show in the next section (Section \ref{sec:NPhardness}), that this is unavoidable under standard hardness assumptions.

  \begin{remark}
    In the case of randomized features, we wrote the infinite-width dual problem   (Eq.~\eqref{eq:dual_infinite}), 
    but  we did not write the infinite-width primal problem to which it corresponds.
    Indeed the dual problem entirely defines the predictor via Eq.~\eqref{eq:InfWidthPredictorDual}.
    For the sake of completeness, we present the infinite-width primal problem in Appendix \ref{sec:Primal}.
  \end{remark}

  \section{$\NP$-hardness of learning with $\cF_1$ norm}
 \label{sec:NPhardness}

    The fact that, when $\rho(x) = |x|$,  we are unable to solve efficiently the infinite 
    width problem \eqref{eq:MinCplx} via the random
features approach is not surprising. Indeed, consider the featurization map 
$\phi (\bx;\bw) = ( \< \bw , \bx \>+c)_+$ (ReLu activation), and random weights
$(\bw_{i})_{i\le n} \sim_{iid} \mu=\Unif (\S^{d-1} (1))$. Consider solving either the 
infinite-dimensional problem \eqref{eq:MinCplx} or
its random features approximation \eqref{eq:opt_finite} with data
 $(\bx_i)_{i\le n}\sim \Unif( \S^{d-1} (\sqrt{d}))$, $y_i = f_*(\bx_i)
= ( \< \bw_\star , \bx_i \>+c)_+$ (where $\bw_\star\in \S^{d-1} (1)$ and
$c\neq 0$ are fixed). 

In \cite{ghorbani2019linearized}, it was shown that for any $k \in \N$, if $N \leq d^k$, then $\hat f_N$ has test error bounded away from zero,
namely $\|\hf_N-f_*\|_{L^2}\ge c(k)-o_N(1)$ for any sample size $n$. Indeed this lower bound holds for any function that can be written
as in Eq.~\eqref{eq:non-random-fN}, for some coefficients $(a_i)_{i\le N}$. On the other hand, \cite{bach2017breaking}
proves that minimizing the empirical risk subject to $\hf\in \cF_1 (R)$, cf. Eq.~\eqref{eq:PBall}, for a suitable choice of $R$, 
achieves  test error $\|\hf-f_*\|_{L^2} \le d^{O(1)}/\sqrt{n}$. The last result does not apply directly to min $\cF_1$-norm interpolator.
However, if an analogous result was established for interpolators, it would imply that
$\|\hf_N-\hf\|_{L^2}$ remains bounded away from zero in this case as long as $N = d^{O(1)}$.

In order to provide stronger evidence towards hardness,
  we consider the computational complexity of the interpolation problem 
\begin{equation}\label{eq:density_F1_problem}%
  \begin{aligned}
    \mbox{\rm minimize}\;\;\;& \int_{\cV}  |a(\bw)| \mu (\de \bw)\, ,\\
     \mbox{\rm subj. to}\;\;\;& \hf(\bx_i;a)=y_i,\;\;\; \forall i\le n\, .
    \end{aligned}
    \end{equation}
We show that it is $\NP$-hard under a randomized reduction to 
solve \eqref{eq:density_F1_problem} within an accuracy $n^{-C}$ with $C$ a 
fixed absolute constant. On the contrary, for $p>1$,  Corollary \ref{coro:pnorms} and its 
proof shows that one can obtain accuracy $n^{-K}$ for $K$ fixed 
arbitrary in polynomial time using a random features approximation.

It will be convenient to consider a relaxation of  problem \eqref{eq:density_F1_problem}
 by minimizing over the set of signed measures on $\cV$, which we will denote 
 $\cM (\cV)$. For any $\tau \in \cM (\cV)$, let $\hf ( \bx ; \tau) = \int_{\cV} \phi (\bx ; \bw ) \tau ( \de \bw)$. We have
\begin{equation}\label{eq:F1_problem}%
  \begin{aligned}
    \mbox{\rm minimize}\;\;\;& | \tau | (\cV) \, ,\\
     \mbox{\rm subj. to}\;\;\;& \hf(\bx_i;\tau)=y_i,\;\;\; \forall i\le n\, ,
    \end{aligned}
 \end{equation} 
where the minimization is now over all $\tau \in \cM(\cV)$ and $|\tau|(\cV)$ 
denotes the total variation of $\tau$\footnote{Recall that by Hahn decomposition, 
there exists two non-negative measures $\tau_+$ and $\tau_-$ such that $\tau = \tau_+ - \tau_-$. 
The total variation is equal to $|\tau |(\cV) = \tau_+ (\cV) + \tau_- (\cV)$.}. 
If we assume that the signed measure $\tau$ has a density with respect 
to the fixed probability measure $\mu$, i.e., $\tau (\de \bw) = a (\bw) \mu (\de \bw)$,
 then the total variation is equal to $| \tau | (\cV) = \int_{\cV} |a(\bw) | \mu (\de \bw)$
  and we recover the original problem \eqref{eq:density_F1_problem}. Note that if $\mu$ has 
  full support on $\cV$, then the infinum in the two problems are the same (all measures 
  can be written as limits of measures with densities).
  However the infinum in problem \eqref{eq:density_F1_problem} is not attained in general,
  while the infimum in \eqref{eq:F1_problem} is always achieved for $\cV$ compact (barring degenerate cases in which it is not feasible). 
  Technically this happens because the space  of integrable functions
  $a(\bw)$ with $\int|a(\bw)|\mu(\de\bw)\le C$ is not compact in the weak-* topology, 
  while the set of signed measure with  $| \tau | (\cV) \le C$ is.
  (The optimum $\tau$ can be singular with respect to $\mu$, e.g.
  a sum of Dirac delta functions.)

\begin{remark}
For any $\tau \in \cM(\cV)$ that verifies the equality constraint of problem \eqref{eq:F1_problem}, we have $\by / |\tau|(\cV)$ is in the convex hull of $\{ \bphi_n ( \bw) , - \bphi_n(\bw) : \bw \in \cV \}$. Hence, by Caratheodory theorem, there exists at most $(n+1)$ weights $\{\bw_j\}_{j \in [n+1]}$ such that $\by = \sum_{j \in [n+1]} a_j \bphi_n (\bw_j)$ and $\sum_{j \in [n+1]} | a_j | = |\tau|(\cV)$. In particular, if $\cV$ is compact, the minimizer in problem \eqref{eq:F1_problem} is always attained by a measure that is supported on at most $n+1$ points, i.e., there exists $\{(a_j^* , \bw_j^* )\}_{j \in [n+1]}$ such that $\tau^* = \sum_{j \in [n+1]} a_j^* \delta_{\bw^*_j }$ minimizes \eqref{eq:F1_problem}.
\end{remark}

Let us call the problem \eqref{eq:F1_problem} the $\cF_1$-problem. 
 In order to study 
the computational complexity of solving this problem, we introduce a \textit{weak version} of
 the $\cF_1$-problem, where we allow an error $\eps >0$ on the equality constraint:
\begin{equation}\label{eq:F1_problem_weak}%
  \begin{aligned}
    \mbox{\rm minimize}\;\;\;& | \tau | (\cV) \, ,\\
     \mbox{\rm subj. to}\;\;\;& \hf(\bx_i;\tau)=\hy_i,\;\;\; \forall i\le n\, , \\
     & \| \by - \hat \by \|_2 \leq \eps \, ,
    \end{aligned}
\end{equation}
where the minimization is now over $\hat \by \in \R^n$ and $\tau \in \cM (\cV)$. For concreteness we will consider $\phi ( \bx ; \bw  ) = \min ( \max( \< \bw , \bx \>,0) , 1)$ 
(truncated ReLu), but we believe it is possible to generalize our proofs to other
activations at the cost of additional technical work. 
 We  further restrict $\cV$ to be a rectangle in $\R^d$ possibly with some infinite sides.

Denote $\Q$ the set of rational numbers. We consider the following problem \texttt{W-F1-PB} which depends on a rational number $\eps >0$:

\begin{description}
\item[$\boxed{\texttt{W-F1-PB}(\eps)}$]: Given $\by \in \Q^n$ and $\gamma \in \Q$. Denote $L^*$ the value of the weak $\cF_1$-problem \eqref{eq:F1_problem_weak} with $\eps$ error on the constraints. Either
\begin{itemize}
\item[(1)] Assert that $L^* \leq \gamma + \eps$; or,
\item[(2)] Assert that $L^* \geq \gamma - \eps$. 
\end{itemize}
\end{description}

We can think about \texttt{W-F1-PB} as the \textit{weak validity problem} associated to the $\cF_1$-problem \eqref{eq:F1_problem}. In particular, if we are able to solve the $\cF_1$-problem within an additive error $\eps$ of the optimum and with at most $\eps$ $\ell_2$-error on the equality constraints, we can solve \texttt{W-F1-PB}.

We show in the following theorem that \texttt{W-F1-PB} is hard to solve under the standard assumption that $\BPP$ (\textit{bounded-error probabilistic polynomial time} class) does not contain $\NP$:

\begin{theorem}\label{thm:F1_NPhard}
Let the activation function be the truncated ReLu $\phi ( \bx ; \bw  ) = \min ( \max( \< \bw , \bx \>,0) , 1)$, and $\cV$ be a rectangle in $\R^d$ (possibly with some infinite sides).
  Assuming $\NP \not\subset \BPP$, there exists an absolute constant $C>0$ such that the problem \texttt{W-F1-PB}$(n^{-C})$ is $\NP$-hard.
\end{theorem}

By equality of the infimum of problems \eqref{eq:density_F1_problem} and \eqref{eq:F1_problem}, 
Theorem \ref{thm:F1_NPhard} also implies the hardness of the original problem. 
The proof of Theorem \ref{thm:F1_NPhard} relies on a polynomial time randomized 
reduction from an $\NP$-hard problem, the \textit{Maximum Agreement for Halfspaces} problem. 
 If we only assume $\NP \neq \Pclass$, our reductions can be made deterministic using results
  from \cite{grotschel2012geometric}. However this deterministic reduction only rules out
   precision $\eps$ that are \textit{exponential in the number of bits}, in particular 
   it does not rule out precision $\eps = e^{-n}$.

We denote below the Maximum Agreement for Halfspaces problem \texttt{HS-MA} .  It was shown 
to be $\NP$-hard in \cite{feldman2006new} and \cite{guruswami2009hardness}. We will follow 
the notations of  \cite{guruswami2009hardness}. Consider $(n_+,n_-,d) \in \N^3$ and 
$\{ \bx_1 , \ldots , \bx_{n_+} , \bz_1 , \ldots , \bz_{n_-} \} \subset \{ - 1 , 1 \}^d$. Denote 
\begin{equation}\label{eq:M_function}
M( \bw , a) = \sum_{ i =1}^{n_+} \mathbbm{1}[ \< \bw , \bx_i \> > a] +  \sum_{ i =1}^{n_-} \mathbbm{1}[ \< \bw , \bz_i \> < a] \, .
\end{equation}
The \texttt{HS-MA} problem depends on a rational numbers $\eps >0$ (where we slightly 
simplified the statement in \cite{guruswami2009hardness}):
\begin{description}
\item[$\boxed{\texttt{HS-MA}(\eps)}$]: Distinguish the following two cases
\begin{itemize}
\item[(1)] There exists a half space $(\bw ,a)$ such that $M( \bw , a)  \geq (n_+ + n_- ) (1 - \eps)$; or,
\item[(2)] For any half space $(\bw , a)$, we have $M( \bw , a)  \leq (n_+ + n_- ) (1/2 + \eps)$. 
\end{itemize}
\end{description}

\cite{guruswami2009hardness} showed that for all $0 < \eps < 1/4$, the problem \texttt{HS-MA}$(\eps)$ is $\NP$-hard. 

Below we briefly describe the main ideas of the proof of Theorem \ref{thm:F1_NPhard} and we
 defer the detailed to Appendix \ref{sec:proof_hardness}. In order to reduce 
 \texttt{HS-MA} to \texttt{W-F1-PB}, we use the following intermediate problem:
\begin{equation}\label{eq:inter_problem}%
  \begin{aligned}
    \mbox{\rm maximize}\;\;\;& \< \by , \hat \by \> \, ,\\
     \mbox{\rm subj. to}\;\;\;& \hf(\bx_i;\tau)=\hy_i,\;\;\; \forall i\le n\, , \\
     & | \tau | (\cV) \leq 1 \, .
    \end{aligned}
\end{equation}
We denote \texttt{W-VAL}$(\eps)$ the weak validity problem associated to problem \eqref{eq:inter_problem}.
 First notice that we can rewrite equivalently the constraint set as $\hat \by\in K$ 
 where  $K:=\{ \bz \in \R^n : \exists \tau \in \cM(\cV), \hat f(\bx_i ; \tau ) = z_i, i\leq n\}\subseteq
 \R^n$.
  It is easy to see that \texttt{W-F1-PB} can be used as a weak 
  membership oracle for $K$. \cite{lee2018efficient} shows
   that there exists a polynomial-time randomized algorithm that solves 
   \texttt{W-VAL}$(\eps)$ from a weak membership oracle \texttt{W-F1-PB}$((\eps/n)^C)$ for some
    constant $C>0$. Hence there is a randomized reduction between \texttt{W-VAL} and 
    \texttt{W-F1-PB}. Secondly, the problem \eqref{eq:inter_problem} for $\by = \ones$ (vector of ones) 
    has the same value at optimum as the problem
\begin{equation}\label{eq:inter_problem_equiv}
  \begin{aligned}
    \mbox{\rm maximize}\;\;\;& \< \ones , \bphi_n(\bw) \> \, ,\\
     \mbox{\rm subj. to}\;\;\;& \bw \in \cV \, .
    \end{aligned}
\end{equation}
It is easy to see that we can construct data points and weights such that the problem \eqref{eq:inter_problem_equiv} coincide with $M(\bw,a)$ defined in Eq.~\eqref{eq:M_function} at the optimum. Hence, there is an easy deterministic reduction from \texttt{HS-MA}$(\eps)$ to \texttt{W-VAL}$(\eps')$.

In summary, we used the following two reductions
\begin{equation}
    \boxed{\texttt{W-F1-PB}} \,\,\, \xLongrightarrow{\text{R}} \,\,\, \boxed{\texttt{W-VAL}}  \,\,\, \xLongrightarrow{\text{D}} \,\,\,  \boxed{\texttt{HS-MA}}
\end{equation}
where $\boxed{A} \xLongrightarrow{\text{R}} \boxed{B}$ (resp. $\boxed{A} \xLongrightarrow{\text{D}} \boxed{B}$) means that there exists a polynomial time randomized (resp. deterministic) reduction from $B$ to $A$.

\section{Examples}
\label{sec:Examples}

\subsection{A numerical illustration}

In our first example $\bw \in \cV=\R^d$ and  $\phi (\bx ; \bw) = \sigma ( \< \bw , \bx \>)$ where
$\sigma: \R \to \R$ is an activation function with $\sigma(0) = 0$ (this assumption is only to
simplify some of the formulas below and can be removed).
In this case the featurization map is  deterministic, i.e., $\phi (\bx ; \bw) = \ophi (\bx ; \bw)$.

\begin{figure}
  \includegraphics[width = 0.5\linewidth]{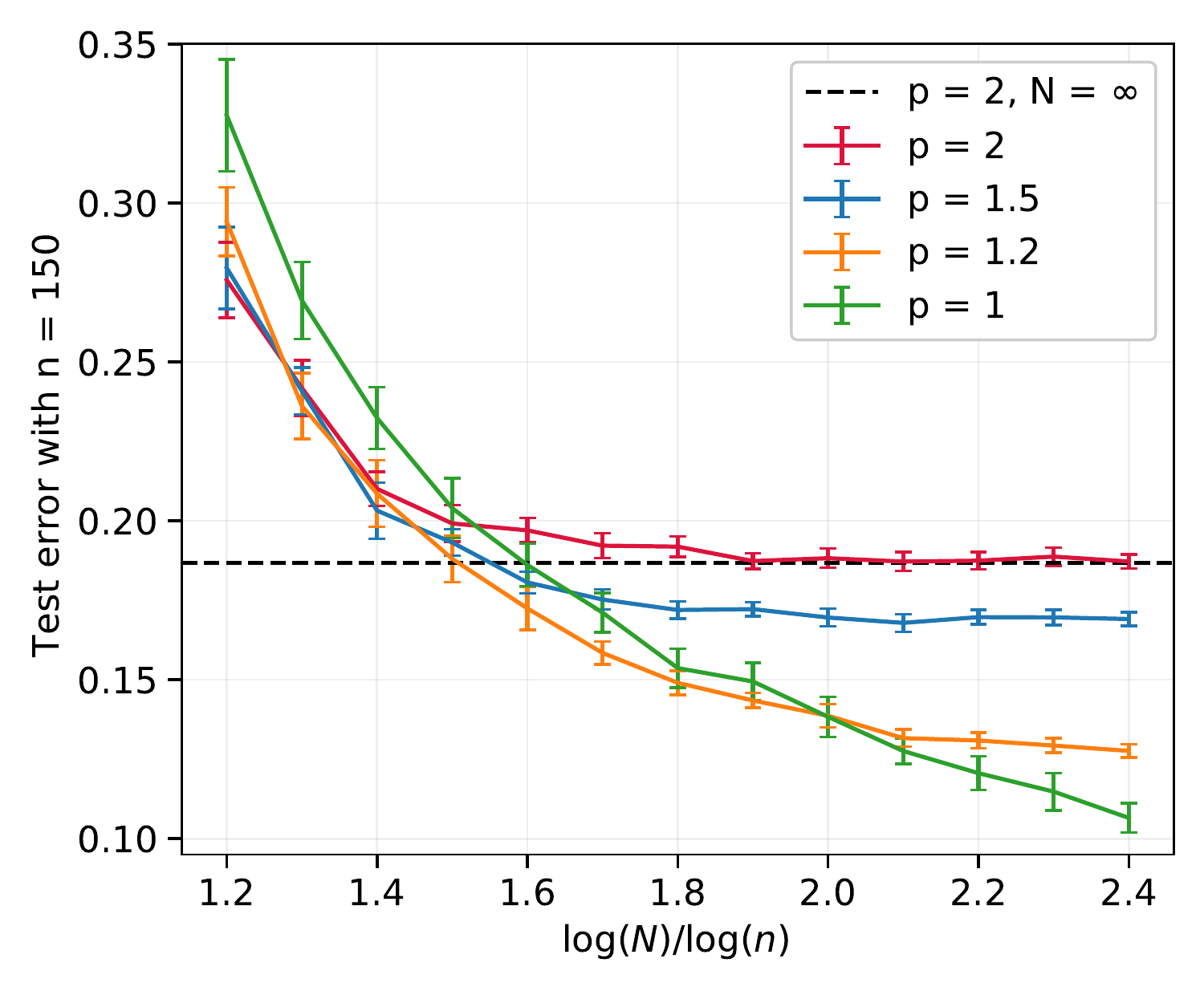}
   \includegraphics[width = 0.5\linewidth]{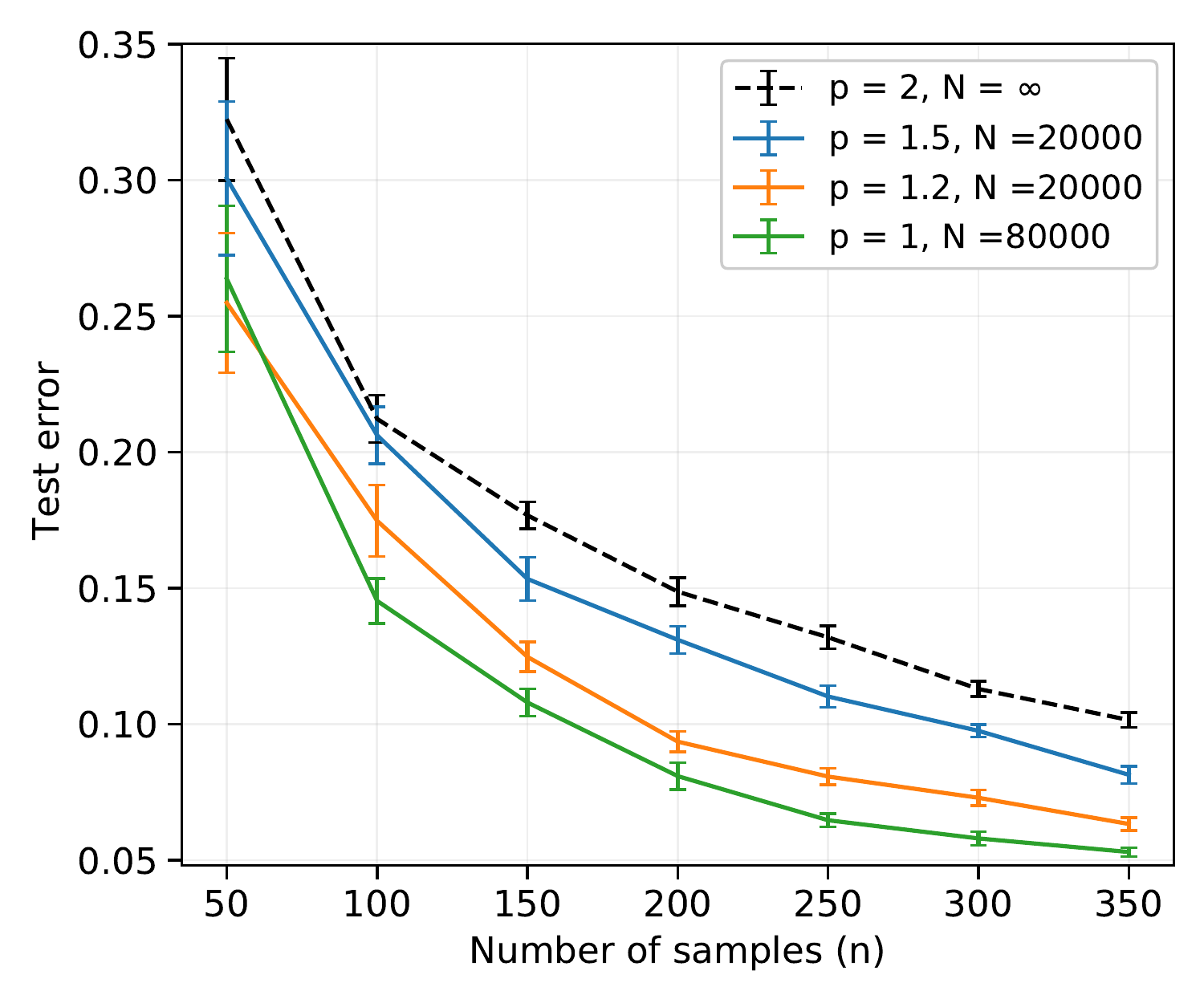}
\vspace{-1cm}
\caption{Minimum complexity random features interpolation for penalty $\rho(x) = |x|^p/p$ and synthetic data
  generated according to the model $y=\sigma(\<\bw_*,\bx\>)$, $\bx\in\reals^d$ (see text). We report the
  average test error over $20$ realizations of this experiment and the $95\%$ confidence intervals.
  Left frame: behavior as a function of the number of features $N$ for fixed sample size $n$. Right frame: behavior as a function
of sample size $n$.}  
\label{fig:ExperimentRidge} 
\end{figure}
Before checking the assumptions of our main theorem in this context, we present a numerical illustration in
Figure \ref{fig:ExperimentRidge}. We generate synthetic data with $(\bx_i)_{i\le n} \sim_{iid} \Unif ( \S^{d-1}(\sqrt{d}))$,
and $y_i = \sigma(\<\bw_*,\bx_i\>)$ where $\bw_*$ is a fixed unit vector, $\|\bw_*\|_2=1$,
and $\sigma(t) = \max(t,0)$ is the ReLU activation.
As mentioned above, we use the featurization map $\phi (\bx ; \bw) = \sigma ( \< \bw , \bx \>)$ 
with weigths  $(\bw_j)_{j\le N} \sim \normal ( 0 , \id_d / d)$.

We fix $d=30$ and solve the minimum complexity interpolation problem
\eqref{eq:MinCplx}, using $\rho(x) = |x|^p/p$. We first fix the sample size $n=150$
and report the average test error as a function of the number of features $N$ (left plot)
for several values of $p$. Notice that in the case $p=2$, the $N=\infty$ limit corresponds to kernel ridge regression,
and hence is directly accessible. We then consider for each $p$ a value of $N$ that is large enough
to obtain a rough approximation of the infinite width limit, and plot the test error as a function of the sample size $n$ (right plot).

A few remarks are in order:
\begin{itemize}
\item[$(i)$] For $p>1$, the test error appears to settle on a limiting value after $N$ becomes large enough $N\gtrsim N_*(n;p)$.
\item[$(ii)$]  The required number of random features $N_*(n;p)$ appears to increase as $p$ decreases. For $p=1$, we are not able to reach
  the $N=\infty$ limit with practical values of $N$.
\item[$(iii)$] As $p$ decreases, the test error achieved by minimum complexity interpolation decreases.
\end{itemize}
Notice that points $(i)$ and $(ii)$ are consistent with our main result Theorem \ref{thm:gen-conv-to-pop}.
Point $(iii)$  is consistent with the notion that the class $\cF_p$ captures better functions that are highly dependent
on low dimensional projections of the covariates vectors.

\subsection{Non-linear random features model}

We next check the assumptions of Theorem \ref{thm:gen-conv-to-pop} for the case of non-linear random features model
$\phi (\bx ; \bw) = \sigma ( \< \bw , \bx \>)$.
\begin{proposition}\label{prop:RF_assumption}
  Assume that $\sigma$ is $L$-Lipschitz and $\sigma(0) = 0$, $\| \bx_i \|_2 \leq \rz\sqrt{d}$ for all $i\le n$, and that the random weights
  $(\bw_i)_{i\le n}\sim \mu$ are mean 0 and satisfy the transportation cost inequality
 \begin{equation}\label{eq:transportation_ineq}
 W_1 ( \nu , \mu ) \leq \sqrt{ 2 (\kappa^2 / d) D ( \nu || \mu )}\,\, \text{ for all probability measures $\nu$ on $\R^d$,}
 \end{equation}
 where $W_1$ is the Wassertein distance and $D$ is the relative entropy (Kullback-Leibler divergence). Then, FEAT1 and FEAT2 are satisfied with constants $L (\bw) = L \| \bw \|_2$ and $\tau = \tau_1 \vee \tau_2$, where
 \begin{align}
   \tau_1 =  C\kappa L \rz, \qquad \tau_2 =  C\kappa L  \| \bK_n^{-1/2} \|_{\op} \frac{\| \bX \|_{\op}}{ \sqrt{d}}\,.
   \label{eq:RoughTauBound}
 \end{align}
\end{proposition}

\begin{proof}[Proof of Proposition \ref{prop:RF_assumption}]
Let us begin with condition FEAT1. Notice that for any $\| \bx \|_2 \leq \rz\sqrt{d}$,
\[
| \sigma ( \< \bw_1 , \bx \>) - \sigma ( \< \bw_2 , \bx \> ) | \leq L \| \bx \|_2 \| \bw_1 - \bw_2 \|_2 \leq L \rz \sqrt{d} \| \bw_1 - \bw_2 \|_2\, .
\]
Hence $\sigma ( \< \bw , \bx \>)$ is $L\rz\sqrt{d}$-Lipschitz. By assumption
\eqref{eq:transportation_ineq} and Bobkov-G\"otze theorem \cite{bobkov1999exponential},
$\sigma ( \< \bw, \bx \>)-\E_{\bw}[\sigma(\< \bw , \bx \>)]$ is $(\kappa L \rz)^2$-sub-Gaussian with respect to $\bw$, for any fixed $\bx$.
Further $|\E_{\bw}[\sigma(\< \bw , \bx \>)] |\leq L \E_{\bw}[|\< \bw , \bx \>|] \leq C\kappa Lr_0$.

Similarly, for any $\bv \in \R^n$, $\| \bv \|_2 = 1$,
\[
| \< \bv , \bpsi_{n} ( \bw_1 ) \> - \< \bv, \bpsi_{n} ( \bw_2 ) \> | \leq L \| \bK_n^{-1/2} \|_2 \| \bX ( \bw_1 - \bw_2 ) \|_2 \leq L  \| \bK_n^{-1/2} \|_{\op} \| \bX \|_{\op} \| \bw_1 - \bw_2 \|_2 \, .
\]
We deduce that $\bpsi_{n} ( \bw)-\E[\bpsi_n(\bw)]$ is $( \kappa L \| \bK_n^{-1/2} \|_{\op} \| \bX \|_{\op} / \sqrt{d} )^2$-sub-Gaussian with respect to $\bw$,
and $\|\E_{\bw}[\bpsi_n(\bw)]\|_2 \leq C \kappa L \| \bK_n^{-1/2} \|_{\op} \| \bX \|_{\op} / \sqrt{d} $.

Next consider condition FEAT2. We have $\sigma ( \< \bx , \bw \>)$ that is Lipschitz with respect to $\bx$ with Lipschitz constant $L (\bw ) := L\| \bw \|_2$. Using that $L (\bw)$ is $L$-Lipschitz with respect to $\bw$, we have
\[
\P ( L (\bw) \geq 4 L \kappa + t ) \leq \exp \{ - d t^2 / (2 L^2 \kappa^2 ) \} \, ,
\]
and therefore there exists an absolute constant $C >0$ that only depend on $L$ and $\kappa$ such that
\[
\P ( L(\bw) \geq t ) \leq C \exp ( - t^2/ (2 L^2 \kappa^2 ) ) \, .
\]
\end{proof}

To make Proposition \ref{prop:RF_assumption} more concrete, let us make the following remarks:
\begin{enumerate}
\item[(a)] The transportation cost inequality \eqref{eq:transportation_ineq} is a necessary and sufficient condition for any $L$-Lipschitz function of $\bw$ to be $(L^2 \kappa^2/d)$-sub-Gaussian. For example, $\bw \sim \normal ( 0 , (\kappa^2 /d) \cdot \id_d )$ verifies this inequality
  \cite[Chapter 4]{van2014probability}.
%
%
    \item[(b)] If $\bx$ is a $\kappa^2$-sub-Gaussian random vector, then by Lemma \ref{lem:matrix_concentration}, there exists constants $C, c>0$ depending only on $\kappa$, such that $\| \bX \|_{\op} \leq C \sqrt{n}$ with probability at least $1 - \exp (- c n)$.
\end{enumerate}
%

Proposition \ref{prop:RF_assumption} only focused on FEAT1 and FEAT2. The last two assumptions FEAT3 and FEAT3'
are more difficult to check. The next proposition provides a class of activation
functions for which FEAT3 is satisfied.
\begin{proposition}
Assume $\phi(\bx;\bw)=\sigma(\<\bw,\bx\>)$ with $\bw\sim\normal(0,\id_d/d)$. 
Let $\mu_k(\sigma) :=\E\{\sigma(G)\He_k(G)\}$ be the $k$-th Hermite coefficient of the 
activation function (where $G\sim\normal(0,1)$, and the normalization 
$\E\{\He_k(G)\He_j(G)\}=k!\delta_{jk}$ is assumed). 

Consider $(\bx_i)_{i\le n}\sim_{iid}\Unif(\S^{d-1}(\sqrt{d}))$ and $d^{\ell+\delta}\le n\le 
d^{\ell+1-\delta}$ for some constant $\delta>0$. If $\mu_{\ell}(\sigma)$ is
bounded away from zero and $|\mu_k(\sigma)|\le C/k^{k+1/2}$ for all $k$ large enough, 
then FEAT 3 is satisfied with high probability with $c=1/4$ and nonrandom
$\eta$ independent of $d,n$.
\end{proposition}
A more general version of this result is proved in Appendix \ref{sec:small_ball} for FEAT3.
While we expect FEAT3 and FEAT3' to hold more generally, we do not have a more general proof at the moment.
We can bypass this difficulty below by considering noisy features.
\begin{proposition}\label{prop:RF_assumption_noisy}
 Consider the same setting as in Proposition \ref{prop:RF_assumption}. Define $\ophi ( \bx ; \bw) = \sigma ( \< \bx , \bw \>)$ and $\phi ( \bx ; \bw ) := \phi ( \bx ; \bw , z) = \sigma ( \< \bx , \bw \>) + z$ where $z \sim \normal (0,\gamma^2)$. Then FEAT1, FEAT2, FEAT3 and FEAT3' are satisfied with constants $L(\bw) = L \| \bw \|_2$, $C_r = \Gamma ((r+1)/2) $ and
 \begin{align}
   \eta  & = \frac{\gamma}{10\lambda_{\max}(\bK_n)^{1/2}}\, ,\label{eq:EtaLB}\\
   \tau &= \tau_1 \vee \tau_2\, ,\;\;\;\;
          \tau_1 =  \kappa L \rz, \qquad \tau_2 =  \kappa^2 L^2  \frac{\| \bX \|^2_{\op}}{\gamma^2 d} + 1\, .
          \label{eq:TauLB}
  \end{align}
\end{proposition}

\begin{proof}[Proof of Proposition \ref{prop:RF_assumption_noisy}]
  First notice that $\bK_n = \overline \bK_n + \gamma^2 \id_n$, where we denoted the matrix $(\overline \bK_n)_{i,j} = \E_{\bw} [ \sigma( \< \bx_i , \bw \>) \sigma (\< \bx_j , \bw\>) ]$. FEAT1 and FEAT2 are verified in Proposition \ref{prop:RF_assumption} with the difference that for $\bv \in \R^n$, $\| \bv \|_2 = 1$,
  \[
  \< \bv , \bpsi_{n} (\bw) \> = \< \bv , \bK_n^{-1/2} \obphi_n ( \bw) \> + \Tilde z \, ,
  \]
  where $\Tilde z \sim \normal ( 0, \gamma^2 \< \bv , \bK_n^{-1} \bv \> )$, and therefore $\bpsi_{n} (\bw)$ is a $ \tau_2^2$-sub-Gaussian random vector with
  \[
  \tau_2^2 = \kappa^2 L^2 \frac{\| \bX \|_{\op}^2}{\gamma^2 d} + 1.
  \]
  
  Let us now check FEAT3. Define $\Delta = \gamma^2 \< \bv , \bK_n^{-1} \bv \>$.
  Recalling that we have $\< \bv , \bpsi_n (\bw ) \> = \< \bv , \bK_n^{-1/2} \obphi_n (\bw) \> + \Tilde z$ with $\Tilde z \sim \normal ( 0, \Delta )$ independent of $\obphi_n (\bw)$, we have
  \begin{align*}
    \P ( |\< \bv , \bpsi_n (\bw ) \>| \leq \eta)
    &\le \sup_{x \in \R} \P ( x - \eta \leq \Tilde z \leq x + \eta) \\
   & \le
    \frac{2\eta}{\sqrt{2 \pi \Delta}} = \big(50\pi\lambda_{\max}(\bK_n)\<\bv,\bK^{-1}_n\bv\>\big)^{-1/2} \le \frac{1}{10}\, .
 \end{align*}
 This shows that FEAT3 is satisfied with the stated value of $\eta$. 
 
 Finally, let us check FEAT3'. Consider $r \in (-1,0)$. Letting $G\sim\normal(0,1)$,
 we have
 \[
 \begin{aligned}
 \E [ |\< \bv , \bpsi_{n} ( \bw ) \>|^r ] \leq &~ \sup_{x \in \R} \E_{\Tilde z} [ | \Tilde z  + x |^r ] \leq \E_{\Tilde z} [ | \Tilde z |^r ] \\
 =&~ \Delta^{r/2} \E[|G|^r]\le \eta^{r}\E[|G|^r]\, ,
 \, .
 \end{aligned}
\]
where in the last inequality we used the fact that $\Delta = \gamma^2 \< \bv , \bK_n^{-1} \bv \>
\ge \gamma^2\lambda_{\max}(\bK_n)^{-1}\ge \eta$.
 FEAT3' is satisfied with $C_r = \E[|G|^r] \leq \Gamma ((r+1)/2) <\infty$. 
\end{proof}

\begin{corollary}\label{coro:RF_randomX}
  Consider the same setting as in Proposition \ref{prop:RF_assumption},
  and Proposition \ref{prop:RF_assumption_noisy}, namely $\ophi ( \bx ; \bw) = \sigma ( \< \bx , \bw \>)$ and $\phi ( \bx ; \bw ) := \phi ( \bx ; \bw , z) = \sigma ( \< \bx , \bw \>) + z$ where $z \sim \normal (0,\gamma^2)$,
  $\gamma\le 1$.

  Further assume $(\bx_i)_{i\le n}$ to be independent centered isotropic, with $\|\bx_i\|_2\le \rz\sqrt{d}$
  almost surely.
  Then there exist constants $C,c$ depending uniquely on $L,\rz,\kappa$ (but not on
  $n,d, \sigma$ or the distribution of the data $\bx_i$), such that conditions FEAT1, FEAT2, FEAT3 and FEAT3' are satisfied with
  probability at least $1-Ce^{-cn}$ for the constants
 \begin{align}
   \eta   = \frac{c\gamma}{\sqrt{n}} \, , \;\;\;\; \;\;\;\;
   \tau &= \frac{C}{\gamma}\Big(\sqrt{\frac{n}{d}}\vee \sqrt{\log d}\Big)\, .
 \end{align}
 In particular, consider the case $\rho(x)=|x|^p/p$, $p\in(1,2]$. If $(y_i)_{i\le n}$ are independent $\E\{y_i^4\}\le C$, then with probability
 at least $1-Cn^{-1}$, we have for $N\ge C'
 \gamma^{-2Q - 2(Q \vee 3)} (\frac{n}{d}\vee\log d)^Q n^{Q \vee 3 +Q/2} (\log n)^{Q/2}$,
  \begin{equation}
        \| \hf_N (\,\cdot\,;\hat \blambda_N) - \hf(\,\cdot\,;\hat \blambda)\|_{L_2(\P)}
        \leq 
        C\frac{n^{(Q\vee3 + 1)/2}}{\gamma^{Q\vee 3 + 2Q +1}}
        \Big(\frac{n}{d}\vee\log d\Big)^{Q}\Big( \sqrt{\frac{n\log N}{N}} \vee \frac{(n\log N)^{Q/2}}{N} \Big)
        \, .\label{eq:RandX_Bound}
      \end{equation}
\end{corollary}
\begin{proof}[Proof of Corollary \ref{coro:RF_randomX}]
  We have
  \begin{align*}
    \lambda_{\max} ( \bK_n ) &= \sup_{\|\bu\|_2=1}\E_{\bw, z} [ \<\bphi_n ( \bw, z ),\bu\> ^2]\\
    & \le \E_{\bw, z} [ \|\bphi_n ( \bw, z )\| ^2]
      \le n \sup_{\|\bx\|_2\le \rz\sqrt{d}} \E_{\bw, z} [ \phi (\bx ; \bw, z )^2 ] \\
                             &\le n \sup_{\|\bx\|_2\le \rz\sqrt{d}} \big\{L^2
                               \E[\<\bw,\bx\>^2]+\gamma^2\big\}\le Cn\, .
  \end{align*}
  Substituting into Eq.~\eqref{eq:EtaLB}, we get the desired bound on $\eta$.

  Considering the estimate on $\tau$, we note that
  $\|\bX\|_{\op}\le C(\sqrt{n}+\sqrt{d\log d})$ with the stated probability by \cite[Theorem 5.6.1]{vershynin2018high}.
  Substituting in Eq.~\eqref{eq:TauLB}, we get the desired estimate of $\tau$.

  Finally, to  prove \eqref{eq:RandX_Bound}, recall that $\bK_n = \overline \bK_n + \gamma^2 \id_n$, where $\overline \bK_n\succeq \bzero$, whence $\|\bK_n^{- 1/2} \by \|_{2 }\le \gamma^{-1}\|\by\|_2\le C\gamma^{-1}\sqrt{n}$,
  where the second step is just Markov inequality. Substituting in Eq.~\eqref{eq:CoroPnorm} yields the
  desired bound \eqref{eq:RandX_Bound}.
\end{proof}
\begin{remark}
  The estimates of $\tau$ and $\eta$ in Proposition \ref{prop:RF_assumption_noisy} and Corollary
  \ref{coro:RF_randomX} are not optimal. First, we expect that the correct order in the second expression in
  Eq.~\eqref{eq:RoughTauBound} to be often $\| \bK_n^{-1/2}  \bX \|_{\op}/ \sqrt{d}$ instead of
  $\| \bK_n^{-1/2} \|_{\op}\|\bX \|_{\op}/ \sqrt{d}$. For many cases of interest the former is of order one, while the latter
  is of order $\sqrt{n/d}$ as we saw.

  Second, we expect the dependence of $\eta$ on $\lambda_{\max}(\bK_n)$ to be often milder (see Appendix \ref{sec:small_ball}).
  Third, we know that in many interesting cases, $\lambda_{\max}(\bK_n)$ is of order $n/d$ instead of order
  $n$. This is for instance the case if $\bw\sim \Unif(\S^{d-1}(1))$, $\bx_i\sim \Unif(\S^{d-1}(\sqrt{d}))$
  \cite{mei2021generalization} (under the assumption $\E[\bphi_n(\bw)]=\bzero$).
\end{remark}

\subsection{The latent linear model}
\label{sec:latent}

This section provides a simple example in which the estimates derived above can be strengthened. 
We consider the following model
\[
\phi ( \bx ; \bw ,z ) = \< \bx , \bw \> + z\, ,
\]
which we will refer to as the `latent linear model.' 

While the latent linear model is extremely simple, it was shown in some settings to have the same asymptotic behavior
as the noiseless nonlinear random features model $\phi(\bx;\bw) = \sigma ( \< \bw , \bx \>)$
\cite{mei2019generalization,hu2020universality}.
For instance, consider the case $\bw\sim \Unif(\S^{d-1}(1))$, $\bx\sim \Unif(\S^{d-1}(\sqrt{d}))$.
Then $\<\bw,\bx\>$ is approximately $\normal(0,1)$.
Decompose $\sigma(t) = \mu_0+\mu_1t  + \sigma_{\perp}(t)$, where $\sigma_{\perp}$ is orthogonal to
linear and constant functions in $L^2(\reals;\gamma)$, with $\gamma$ the standard normal measure.
Then \cite{mei2019generalization} shows that in the proportional asymptotics
$N\asymp n\asymp d$, ridge regression in the  nonlinear random features model behaves as
ridge regression with the latent linear model
$\phi ( \bx ; \bw ,z ) = \mu_0 + \mu_1 \< \bw , \bx \> + \mu_{\star} z$,
with $z\sim\normal(0,1)$ independent of $\bw,\bx$. Here we study the latent linear model from a different perspective
and in a broader context than \cite{mei2019generalization,gerace2020generalisation,hu2020universality}.

We make the following assumptions:
\begin{description}
    \item[A1] \emph{(Covariates distribution)} The covariates $\bx$ are $\kappa^2$-subgaussian random vectors with zero mean and second moment $\E [ \bx \bx^\sT ] = \bSigma_{x}$ and bounded support $\| \bx \|_2 \leq \rz \sqrt{d}$. 
    
    \item[A2] \emph{(Features distribution)} The features $\bw$ follow a multivariate Gaussian distribution with zero mean and covariance $\bSigma_{w}/d$. Furthermore, assume that
      $\lambda_{\max} ( \bSigma_{w} ) \le \kappa$.
    
    \item[A3] \emph{(Features noise distribution)} We assume that $z\sim\normal(0,\gamma^2)$.
    
\end{description}

\begin{proposition}
\label{prop:latant_assumptions}
Assume that conditions A1, A2 and A3 hold. Then there exists $C>0$ such that for $n \geq Cd$, the conditions FEAT1, FEAT2, FEAT3 and FEAT3' hold with probability at least $1 - e^{-c n}$ with constants $\tau,\eta$
depending only on the constants in A1, A2 and A3. In particular, $\tau,\eta$ can be taken to be independent
of $d,n$.
\end{proposition}

\begin{proof}[Proof of Proposition \ref{prop:latant_assumptions}]
  We begin by noticing  that $\bK_n = \E_{\bw, \bz} [ \bphi_n ( \bw , \bz) \bphi_n ( \bw , \bz )] =  \bX (\bSigma_w /d ) \bX^\sT + \gamma^2 \id_n$.

  Consider condition FEAT1. We have $\ophi ( \bx ; \bw ) = \< \bx , \bw \> \sim \normal (0 , \bx^\sT \bSigma_w \bx /d)$ is mean 0. By A1 and A2, $\ophi ( \bx ; \bw )$ is $(\rz \kappa)^2$-subgaussian. Furthermore, for any
  fixed $\bx_1,\dots,\bx_n\in\reals^d$, $\bpsi_n(\bw)$ is by construction isotropic, and Gaussian
  (since $\bw$ and $z$ are). Therefore, it is $1$-subgaussian and mean 0.
  
  FEAT2 is easily verified with $L(\bw) = \| \bw \|_2$ and $\P ( L(\bw) \geq t) \leq C \exp ( - t^2/ (2 \kappa^2) )$.

  For condition FEAT3, note that $\<\bv,\bpsi_n(\bw)\>\sim\normal(0,1)$ for any unit vector
  $\bv$. Therefore we have $\P(|\<\bv,\bpsi_n(\bw)\>|\le \eta) \le (2/\pi)^{1/2}\eta$, whence we can 
  take $\eta=1/10$.

  Finally, for FEAT3',  $\E [ | \< \bv , \bpsi_{n} (\bw) \> |^r]=  \E[|G|^r]=: C_r$ for $G\sim\normal(0,1)$.
\end{proof}

We finally notice that, for the latent linear model, we can improve Theorem \ref{thm:gen-conv-to-pop}.
For the latent linear model, we can use the constraint that the  (randomized) predictor interpolates the data,
to improve the bound \eqref{eq:MainThmBound} by a factor $\sqrt{n}$. We expect this insight to generalize to the non-linear setting.
\begin{proposition}\label{prop:latent-RF-infinite}
    Assume conditions A1, A2 and A3 hold. There exists constants $C', c'$ depending only on the constants in those assumptions such that if $N \geq C (n \log(N))^{Q/2}$ and $n \geq  C' d$, then with probability at least $1 - C' e^{-c' n} - C' N^{-c'n}$,
    \begin{equation}\label{eq:bound_prop_latent}
        \| \hf_N (\cdot;\hat \blambda_N) - \hf_0(\cdot;\hat \blambda_0)\|_{L_2}
        \leq 
        C'\Big( \sqrt{\frac{n\log N}{N}} \vee \frac{(n\log N)^{Q/2}}{N} \Big) \frac{\| \by \|_{2 }}{\sqrt{n}} \, ,
    \end{equation}
    where we denoted $Q = Q_1 \vee Q_2$.
\end{proposition}

In other words, the random features predictor $\hat f_N ( \,\cdot\, ; \hblambda_N )$ differs negligibly from the infinite-width predictor $\hat f ( \cdot ; \hblambda )$ when $N \gg ( n \log(n) )^{1 \vee (Q/2)}$. In particular, for $Q=2$, we obtain that $N \gg n\log(n)$ features are sufficient, which matches the results for kernel methods  \cite{mei2019generalization,mei2021generalization}.
The proof of Proposition \ref{prop:latent-RF-infinite} is deferred to Appendix \ref{app:proof_latent}.

\section{Proof of Theorem \ref{thm:gen-conv-to-pop}: Convergence to the population predictor}
\label{sec:ProofMain}

The proof of Theorem \ref{thm:gen-conv-to-pop} is structured as follows. 
We first define three events on which the finite-width dual objective
$F_N$ and the random features predictor $\hat f_N (\, \cdot\, ;\blambda)$ satisfy certain concentration properties.
Lemma \ref{lem:gen-deterministic-analysis}
shows that the simultaneous occurrence of these events
implies a bound on $\| \hat f_N (\,\cdot\, ; \hat \blambda_N ) - \hat f(\,\cdot\, ; \hat \blambda) \|_{L_2(\P)}$.
We then verify that these events occur with high-probability.

Throughout the proofs, we will use the notation $\| \bu \|_{\bA} = \| \bA^{1/2} \bu \|_2$ for $\bu \in \R^n$ and
$\bA \in \R^{n \times n}$ a positive semidefinite matrix. We also use the standard big-Oh and little-o notations whereby the subscript
is used to indicate the asymptotic variable. For instance, we write $f_N= o_N(g_N)$ if $f_N/g_N\to 0$ as $N\to\infty$.

The three events mentioned above are defined as follows: 
\begin{enumerate}

    \item \emph{(Uniform concentration of the predictor)} 
    Event $\cE_1$ is the event that for all $\|\blambda\|_{\bK_n}/\|\hat \blambda\|_{\bK_n} \in [1/2,2]$
    \begin{equation*}
            \frac{\|\hat f_N (\,\cdot\,;\blambda) - \hat f(\,\cdot\,;\blambda)\|_{L_2}}{s(\|\hat \blambda\|_{\bK_n})} 
            =
            \frac{\Big\|\frac1{N}\sum_{j=1}^N \ophi_{\bw_j} (\,\cdot\, )s(\langle\bphi_{n,j}, \blambda\rangle) - \E_{\bw,\phi}\Big[\ophi_{\bw} (\,\cdot \,)\, s(\< \bphi_{n} (\bw) , \blambda\>)\Big]\Big\|_{L_2}}{s(\|\hblambda\|_{\bK_n})} 
            \leq 
            \eps_1.
    \end{equation*}

    \item \emph{(Concentration of dual gradient at $\hblambda$)}
    Event $\cE_2$ is the event that
    \begin{equation*}
        \frac{\|\nabla F_N (\hblambda)\|_{\bK_n^{-1}}}{s(\|\hblambda\|_{\bK_n})}
        = 
        \frac{\Big\| \frac1N\sum_{j=1}^N \bpsi_{n,j} s(\< \bphi_{n,j}, \hblambda\>) - \E_{\bw , \phi} \Big[\bpsi_n (\bw) s(\<\bphi_n (\bw) , \hblambda\>)\Big]\Big\|_2}{s(\|\hat \blambda\|_{\bK_n})} \leq \eps_2.
    \end{equation*}

    \item \emph{(Uniform lower bound on dual curvature)}
    Event $\cE_3$ is the event that
    for all $\|\blambda\|_{\bK_n}/\|\hblambda\|_{\bK_n} \in [1/2,2]$
    \begin{equation*}
        \frac{\bK_n^{-1/2}\nabla^2 F_N (\blambda)\bK_n^{-1/2}}{s(\|\hblambda\|_{\bK_n }) / \|\hblambda\|_{\bK_n}}  
        =
        -  \frac{\frac 1N \sum_{j=1}^N \bpsi_{n,j}\bpsi_{n,j}^\top s'(\< \bphi_{n,j} , \blambda \>)}{s(\|\hblambda\|_{\bK_n}) / \|\hblambda\|_{\bK_n}}  
        \preceq 
        - \beta \id_n.
    \end{equation*}

\end{enumerate}

The first event $\cE_1$ corresponds to the random features predictor $\hat f_N(\,\cdot\,;\blambda)$ approximating the infinite-width
predictor $\hat f(\,\cdot\,;\blambda)$ uniformly well over $\blambda$ in a region around $\hblambda$.
Events $\cE_2$ and $\cE_3$ relate to local properties of the finite-width dual objective $F_N$ around $\hblambda$:
on $\cE_2$, the gradient of the dual objective $\nabla F_N (\hblambda)$ concentrates around the gradient of the infinite
width dual objective $\nabla F ( \hblambda ) = \bzero$; on $\cE_3$, the Hessian of the dual objective $\nabla^2 F_N (\blambda)$ has maximum eigenvalue uniformly bounded away from $0$ for $\blambda$ in a region around $\hblambda$. 

Because these three events involve concentration or bounds on empirical means over a sample of $N$ features,
it is perhaps not surprising that the preceding bounds can be established with high-probability for $\eps_1,\eps_2 = o(1)$ and $\beta = \Theta(1)$ appropriately chosen when $N$ if sufficiently large compared to $n$.

If the infinite-width predictor $\hf(\,\cdot\,;\blambda)$ satisfies a certain continuity property in $\blambda$,
then the events $\cE_1$, $\cE_2$, $\cE_3$ imply a bound on $\| \hf_N (\,\cdot\,;\hat \blambda_N) - \hf(\,\cdot\,;\hblambda )\|_{L_2}$.
\begin{lemma}\label{lem:gen-deterministic-analysis}
    Assume that for all $\|\blambda - \hblambda\|_{\bK_n} \leq \|\hblambda\|_{\bK_n}/2 $,
    \begin{equation}\label{eq:pop-continuity}
        \frac{\|\hf(\,\cdot\,;\blambda) - \hf(\,\cdot\,;\hblambda)\|_{L_2}}{s(\|\hblambda\|_{\bK_n})}
            \leq 
            K \frac{\|\blambda - \hblambda\|_{\bK_n}}{\|\hblambda\|_{\bK_n}}\, ,
    \end{equation}
    for some $K > 0$.
    If $\eps_2/\beta \leq 1/4$,
    then on events $\cE_1,\cE_2,\cE_3$, we have
    \begin{equation*}
        \| \hf_N (\,\cdot\,;\hat \blambda_N ) - \hf(\,\cdot\,;\hblambda)\|_{L_2}
        \leq 
        \Big(\eps_1 + 2\frac{K \eps_2}{\beta}\Big)  s(\|\hblambda\|_{\bK_n}) \, .
    \end{equation*}
\end{lemma}
\noindent The continuity property for $\hf$ used in Lemma \ref{lem:gen-deterministic-analysis} is much easier to establish than the corresponding continuity property of $\hf_N$.
In the setting of Theorem \ref{thm:gen-conv-to-pop}, we will show that we can choose $K,\beta = \Theta_N(1)$ and $\eps_1,\eps_2 = o_N(1)$ such that the above events hold with high probability. 
Then, we will have $\eps_2 / \beta \leq 1/4$ eventually, and Lemma \ref{lem:gen-deterministic-analysis} can be applied.
\begin{proof}[Proof of Lemma \ref{lem:gen-deterministic-analysis}]
    Consider $\blambda$ with $\|\blambda - \hat \blambda\|_{\bK_n} = (2\eps_2/\beta)\| \hat \blambda \|_{\bK_n} \leq \| \hat \blambda \|_{\bK_n}/2$.
    By Taylor's theorem, there exists $\blambda_\star$ on the line segment between $\blambda$ and $\hat \blambda$, such that
    \begin{equation*}
    \begin{aligned}
    F_N ( \blambda ) 
        &= F_N  ( \hat \blambda ) + ( \blambda - \hat \blambda)^\top \nabla F_N (\hat \blambda) + \frac{1}{2} ( \blambda - \hat \blambda)^\top \nabla^2 F_N (\blambda_\star)  ( \blambda - \hat \blambda)
    \\
        &\leq 
        F_N ( \hat \blambda) + \| \blambda - \hat \blambda\|_{\bK_n} \| \nabla F_N (\hat \blambda) \|_{\bK_n^{-1}} - \frac{1}{2} \lambda_{\min} \big( - \bK_n^{-1/2}\nabla^2 F_N (\blambda_\star)\bK_n^{-1/2} \big) \|\blambda - \hat \blambda\|_{\bK_n}^2.
    \end{aligned}
    \end{equation*}
    When both $\cE_2$ and $\cE_3$ occur and using $\| \blambda_\star - \hat \blambda\|_{\bK_n} \leq \| \blambda - \hat \blambda\|_{\bK_n} \leq \| \hat \blambda \|_{\bK_n} / 2$, we get
    \begin{align*}
        F_N (\blambda)
            &\leq 
            F_N (\hat \blambda ) + \frac{2\eps_2\| \hat \blambda \|_{\bK_n}}{\beta} \eps_2 s(\| \hat \blambda\|_{\bK_n})
            -
            \frac12 \frac{\beta s(\| \hat \blambda\|_{\bK_n})}{\|\hat \blambda\|_{\bK_n}} \frac{4\eps_2^2\| \hat \blambda\|_{\bK_n}^2}{\beta^2}
            \leq 
            F_N (\hat \blambda)\, .
    \end{align*}
    Because $F_N(\blambda) \leq F_N(\hat \blambda)$ for all $\blambda$ with $\|\blambda - \hat \blambda\|_{\bK_n} = (2\eps_2/\beta)\| \hat \blambda \|_{\bK_n}$, and $F_N$ is strictly convex at $\hblambda$ by $\cE_3$we conclude
    \begin{align*}
        \|\hat \blambda_N - \hat \blambda\|_{\bK_n} 
            \leq 
            (2 \eps_2/\beta) \| \hat \blambda\|_{\bK_n} \, .
    \end{align*}
    In this case, by Eq.~\eqref{eq:pop-continuity},
    \begin{equation*}
        \| \hat f(\,\cdot\,;\hat \blambda_N ) - \hat f(\,\cdot\,;\hat \blambda) \|_{L_2}
            \leq 
            K \frac{\|\hat \blambda_N - \hat \blambda\|_{\bK_n}}{\|\hat \blambda\|_{\bK_n}} s(\|\hat \blambda\|_{\bK_n})
            \leq 
            \frac{2 K\eps_2}{\beta}s(\|\hat \blambda\|_{\bK_n}).
    \end{equation*}
    If event $\cE_1$ occurs, then
    \begin{equation*}
        \|\hat f_N(\, \cdot\, ;\hat \blambda_N) - \hat f(\,\cdot \, ; \hat \blambda_N ) \|_{L_2}
            \leq 
            \eps_1 s(\|\hat \blambda\|_{\bK_n}).
    \end{equation*}
    Combining the previous displays with the triangle inequality completes the proof.
\end{proof}

We next state three lemmas implying that events $\cE_1$, $\cE_2$, $\cE_3$ hold with high probability, as well as the continuity
of $\hf(\,\cdot\,;\blambda)$ in the last lemma. Proofs of these lemmas are deferred to the appendices.
We begin with the continuity property of the infinite width predictor.
\begin{lemma}\label{lem:gen-dual-continuity}
    If either $(i)$ assumptions FEAT1 and PEN hold with $Q_1 \wedge Q_2 \geq 2$, or $(ii)$ assumptions FEAT1, FEAT3' and PEN hold with $Q_1 \wedge Q_2  < 2$,
    then
    there exists $C'$ depending only on the constants $c,C,\rz$ in FEAT3' and PEN,
    but not on $\tau,\eta$, such that for all $\|\blambda - \hat \blambda\|_{\bK_n} \leq \| \hat \blambda \|_{\bK_n}/2$,
    \begin{equation}\label{eq:pop-continuity_2}
            \frac{\|\hf(\,\cdot\,;\blambda) - \hf(\,\cdot\,;\hblambda)\|_{L_2}}{s(\|\hblambda\|_{\bK_n})}
            \leq 
             C' (\tau^{Q}\vee \tau^2\eta^{(Q'-2)})   \frac{\|\blambda - \hblambda\|_{\bK_n}}{\|\hblambda\|_{\bK_n}}\, ,
    \end{equation}
    where $Q = Q_1 \vee Q_2$, $Q'=Q_1\wedge Q_2$.
\end{lemma}

Next we state a lemma to check condition $\cE_1$.
\begin{lemma}\label{lem:gen-uniform-concentration-of-predictor}
    Assume FEAT1, FEAT2 and PEN hold. Then there exist $C',c' > 0$
    depending only on the constants $c,C,\rz$ in those assumptions, but not on  $\tau,\eta$, such that
    for $N \geq n \ge c' d$, we have with probability at least $1 - C' N^{-c' n}$,
    \begin{equation}\label{eq:gen-uniform-mean-concentration}
    \begin{aligned}
        \bar S_1
        &:=
        \sup_{\|\bx\|_2 \leq \rz\sqrt{d}}\;
        \sup_{\|\bv\|_2 \leq 2 \| \hblambda \|_{\bK_n}} 
            \Big| 
                \frac1N \sum_{j=1}^N \bar \phi(\bx;\bw_j) s( \< \bpsi_{n,j} , \bv \> )
                -
                \E_{\bw,\phi}[\bar \phi(\bx;\bw) s( \< \bpsi_n(\bw) , \bv \> )]
            \Big|
        \\
        &\leq 
        C' \tau^Q s(\|\hblambda\|_{\bK_n}) \Big( \sqrt{\frac{n\log N}{N}} \vee \frac{(n\log N)^{(Q/2) \vee 1}}{N}\Big)\, ,
    \end{aligned}
    \end{equation}
    where $\bpsi_n(\bw) = \bK_n^{-1/2} \bphi_n(\bw)$ and $Q = Q_1 \vee Q_2$.
\end{lemma}

The next lemma allows us to check event $\cE_2$.
\begin{lemma}\label{lem:point-wise-concentration}
Assume FEAT1 and PEN hold. There exist $C',c' > 0$ depending only on the constants $c,C,\rz$ in those assumptions, but not on $\tau,\eta$, such that
    for $N \geq n \ge c' d$, we have with probability at least $1 - C' N^{-c' n}$,
    \begin{equation}\label{eq:gen-concentration-dual-gradient}
    \begin{aligned}
        \bar S_2
        &:= \Big\| \frac1N\sum_{j=1}^N \bpsi_{n,j} s(\<\bphi_{n,j} ,\hblambda\>) - \E_{\bw,\phi}[\bpsi_n(\bw) s(\< \bphi_n,\hblambda\>)]\Big\|_2 
        \\
        &\leq 
        C' \tau^Q  s(\|\hblambda\|_{\bK_n}) \Big( \sqrt{\frac{n\log N}{N}} \vee \frac{(n\log N)^{(Q/2)\vee 1}}{N}\Big),
    \end{aligned}
    \end{equation}
    where $\bpsi_n(\bw) = \bK_n^{-1/2} \bphi_n(\bw)$ and $Q = Q_1 \vee Q_2$.
\end{lemma}

We finally state a lemma to check event $\cE_3$.
\begin{lemma}\label{lem:lower-hessian-no-small-ball} 
Assume that FEAT1, FEAT3 and PEN hold, and further assume that
$\tau,\eta^{-1}\le n^C$ for some absolute constant $C$. Define $\delta_0(\eta) :=\eta^{3\vee q_1\vee q_2}$.
There exist $C',c'> 0$ depending only on the constants in $c,C,\rz$ in the assumptions, but not on
$\tau,\eta$, such that  for $N \geq C'(\tau^4/\delta_0(\eta)^2) n \log(N) $, we have with probability at
least $1 - C' N^{-c'n}$,
    \begin{equation}\label{eq:hessian-lower-bound}
    \begin{aligned}
      \sup_{1/2 \leq \| \blambda \|_{\bK_n} / \| \hat \blambda \|_{\bK_n} \leq 2}  \bK_n^{-1/2} \nabla^2 F_N (  \blambda ) \bK_n^{-1/2} \preceq - c'\delta_0(\eta)
      \frac{s (   \| \hat \blambda \|_{\bK_n})}{ \| \hat \blambda \|_{\bK_n} } \id_n \, .
    \end{aligned}
    \end{equation}
  \end{lemma}

  Using the above lemmas, we are now in position to prove our main result, Theorem \ref{thm:gen-conv-to-pop}.
  \begin{proof}[Proof of Theorem \ref{thm:gen-conv-to-pop}]
    Recall we define $q=q_1\vee q_2$, $Q=Q_{1}\vee Q_2$, $Q'=Q_1\wedge Q_2$. 
    Assume
    $N \geq N_1\vee N_2 $ as in the statement. 
    By Lemmas \ref{lem:gen-uniform-concentration-of-predictor}, \ref{lem:point-wise-concentration},
    \ref{lem:lower-hessian-no-small-ball} events $\cE_1$, $\cE_2$, $\cE_3$ hold with probability 
    \begin{align*}
      \P(\cE_1\cap \cE_2\cap\cE_2)\ge  1 - C' N^{-c' n}\, ,
    \end{align*}
    with constants
    \begin{align*}
      \eps_1=\eps_2&= C'  \tau^{Q} \Big( \sqrt{\frac{n\log N}{N}} \vee \frac{(n\log N)^{(Q/2) \vee 1}}{N}\Big)\, ,\\
      \beta &=c'\eta^{3\vee q}\, .
    \end{align*}
    Further, by Lemma \ref{lem:gen-dual-continuity}, the continuity property of Eq.~\eqref{eq:pop-continuity} holds with
    \begin{align*}
      K=     C' (\tau^{Q}\vee \tau^2\eta^{(Q'-2)})\, .
      \end{align*}

      We can now apply Lemma \ref{lem:gen-deterministic-analysis}. Since $\eps_1=\eps_2$ and, without loss of generality,
      $K/\beta\ge 1$, we obtain that with, probability at least $1-C'N^{-c'n}$,
      \begin{align}
     & \| \hf_N (\,\cdot\,;\hat \blambda_N) - \hf(\,\cdot\,;\hat \blambda)\|_{L_2}\le  \Delta\cdot s ( \| \hblambda \|_{\bK_n } )\, ,\label{eq:AlmostFinal1}\\
     &  \Delta = \frac{3K\eps_1}{\beta} = C'\tau^{Q+2} \frac{(\tau^{Q-2} \vee \eta^{Q'-2})}{\eta^{3\vee q}}\Big( \sqrt{\frac{n\log N}{N}} \vee \frac{(n\log N)^{(Q/2) \vee 1}}{N}\Big)\, \,\label{eq:AlmostFinal2}
    \end{align}
    where we verify that $\eps_1 / \beta \leq 1/4$ by the assumption that $N  \geq N_1\vee N_2$.

    Let us bound $s ( \| \hblambda \|_2)$. We have by the optimality condition $ \bK_n^{-1/2} \by = \E [ \bpsi_n (\bw) s ( \< \bpsi_n (\bw) , \bK_n^{1/2} \hblambda \>) ]$. Therefore, denoting $\bv =  \bK_n^{1/2} \hblambda / \| \hblambda \|_{\bK_n} $ and using PEN,
\begin{align*}
  \| \bK_n^{-1/2} \by \|_{2} = &
  \sup_{\| \bu \|_2 =1} \E [ \< \bu , \bpsi_n ( \bw) \>  s ( \< \bpsi_n (\bw) , \bK_n^{1/2} \hblambda \>) ] \\
  \geq&~ \E [  \< \bv , \bpsi_n ( \bw ) \> s ( \| \hblambda \|_{\bK_n} \< \bpsi_n (\bw) , \bv \>) ] \\
\geq &~ c s ( \| \hblambda \|_{\bK_n} ) \E [ | \< \bv , \bpsi_n ( \bw ) \> |^{q_1} \wedge |  \< \bv , \bpsi_n ( \bw ) \> |^{q_2}  ]\, .
\end{align*}
We next notice that, for any $\delta>0$, and any $C\tau^2$- sub-Gaussian random variable $X$ with $E[X^2]=1$,
$\E[|X|^{q_1}\vee |X|^{q_2}]\ge C(\delta) \tau^{-(2-q_1\wedge q_2+\delta)_+}$. This basic fact is proved in Appendix
\ref{sec:Tools}. We apply this inequality to $X=\<\bv,\bpsi_n(\bw)\>$ to get
\begin{align*}
  \| \bK_n^{-1/2} \by \|_{2} \ge  & c(\delta) s ( \| \hblambda \|_{\bK_n} )\, \tau^{-(2-q_1\wedge q_2+\delta)_+}\, .
\end{align*}
Using this bound together with Eqs.~\eqref{eq:AlmostFinal1}, \eqref{eq:AlmostFinal2} yields the claim of the theorem.
\end{proof}

\section*{Acknowledgements}

T.M.\ thanks Enric Boix-Adsera for helpful discussions about hardness results. This work was supported by NSF through award DMS-2031883 and from the Simons Foundation
through Award 814639 for the Collaboration on the Theoretical Foundations of Deep Learning. We
also acknowledge NSF grants CCF-2006489, IIS-1741162 and the ONR grant N00014-18-1-2729.
M.C.\ was supported by the National Science Foundation Graduate Research Fellowship under grant DGE-1656518.


\bibliographystyle{amsalpha}
\bibliography{RFinterpolation.bbl}

\newpage

\appendix

\section{Verifying the continuity and concentration conditions}

\label{app:verifying_assumptions}

\subsection{Continuity property of the infinite width problem:
  Proof of Lemma \ref{lem:gen-dual-continuity}}
\label{app:continuity}

    By the fundamental theorem of calculus,
    \begin{equation}\label{eq:y0-fund-calc}
    \begin{aligned}
        \hat f(\bx;\blambda) - \hat f(\bx;\hat \blambda)
            &=
            \int_0^1 \nabla_{\blambda} \hat f(\bx;\blambda_t)^\top (\blambda - \hat \blambda) \de t,
    \end{aligned}
    \end{equation}
    where $\blambda_t := t \blambda + (1-t) \hat \blambda$. 
    We will show that
    \begin{equation}\label{eq:exchange-deriv-and-expect}
        \nabla_{\blambda} \hat f(\bx;\blambda_t)^\top (\blambda - \hat \blambda)
            =
            \E_{\bw,\phi}\Big[\ophi (\bx ; \bw ) s'(\langle \bphi_n(\bw), \blambda_t\rangle) \< \bphi_n(\bw) , \blambda - \blambda \> \Big]\, ,
    \end{equation}
    (i.e., we may exchange integration and differentiation),
    and we will bound the right-hand side. In what follows, we will use the shorthand $\bphi_n = \bphi_n(\bw) =
    (\phi(\bx_1;\bw), \dots, \phi(\bx_n;\bw))^{\sT}$ for the evaluation of the featurization map at the $n$ datapoints.

    For $e \geq 1$ and $r_1 , r_2 , r_3 >1$ such that $\frac{1}{r_1} + \frac{1}{r_2} + \frac{1}{r_3} = 1$, we have by H\"older's inequality, 
    \begin{equation}\label{eq:der_y0_bound}
    \begin{aligned}
        &\E_{\bw,\phi}\big[\big|\ophi(\bx;\bw)s'(\langle \bphi_n, \blambda_t\rangle) \< \bphi_n , \blambda - \hat \blambda \> \big|^{e}\big]
        \\
            &\qquad\qquad\leq 
            \E[|\ophi(\bx;\bw)|^{r_1e}]^{1/r_1}\,\E[|s'(\langle \bphi_n,\blambda_t\>)|^{r_2e}]^{1/r_2}\, \E[|\langle \bphi_n , \blambda - \hat \blambda\rangle|^{r_3e}]^{1/r_3} \, .
    \end{aligned}
    \end{equation}
    Denote $Q' = Q_1 \wedge Q_2$. If $Q' < 2$, we set $r_1 = r_3 = \frac{2(3- Q')}{Q'-1}$, $r_2 = \frac{(3-Q')}{2(2-Q')}$, and $\ev = \frac{5-Q'}{2(3-Q')}$. 
    In this case, one can check that $\ev > 1$ and $r_2 (Q'-2) \ev = - \frac{5-Q'}4 > -1 $.
    Otherwise, we set $r_1 = r_2 = r_3 = 3$ and $\ev = 2$. 
    By FEAT1,
    \[
    \begin{aligned}
     \E[|\ophi(\bx;\bw)|^{r_1\ev}]^{1/r_1} & \leq C \tau^\ev \, , \\
     \E[|\langle \bphi_n, \blambda - \hat \blambda\rangle|^{r_3\ev}]^{1/r_3} & \leq C \tau^\ev \| \blambda - \hat \blambda \|_{\bK_n}^\ev\, ,
    \end{aligned}
    \]
    where $C$ depends only on $Q ' $.    
    Furthermore, by either \emph{(i)} PEN and FEAT1 in the case $Q' \geq 2$ or \emph{(ii)} PEN, FEAT 1, and FEAT3' in the case $Q' < 2$, we have
    \[
    \begin{aligned}
        \E[|s'(\< \bphi_n, \blambda_t \> )|^{r_2\ev}]^{1/r_2} 
            &\leq 
            C \frac{s(\|\hat \blambda\|_{\bK_n})^\ev}{\|\hat \blambda\|_{\bK_n}^\ev}\E\Big[ \big| \< \bphi_n ,\blambda_t \> /\|\hat \blambda\|_{\bK_n}\big |^{r_2\ev(Q_1-2)}\vee \big|\< \bphi_n , \blambda_t \>/\|\hat \blambda\|_{\bK_n}\big|^{r_2\ev(Q_2-2)}  \Big]^{1/r_2}
        \\
            &\leq
            C [\tau^{\ev(Q-2)}\vee \eta^{\ev(Q'-2)}] \frac{s(\|\hat \blambda\|_{\bK_n})^\ev}{\|\hat \blambda\|_{\bK_n}^\ev}\, ,
    \end{aligned}
    \]
    where we denoted $Q = Q_1 \vee Q_2$, the constant changes in the last line, and we used that $\|\hat \blambda\|_{\bK_n}/2 \leq \| \blambda_t \|_{\bK_n} \leq 2 \| \hat \blambda \|_{\bK_n}$ and that either \emph{(i)} FEAT1 and $r_2\ev(Q'-2) \geq 0$ in the case that $Q' \geq 2$ or \emph{(ii)} FEAT1, FEAT3', and $r_2 \ev (Q'-2) > -1$ in the case that $Q' < 2$.
    We see that the expectation $\E_{\bw,\phi}\big[\big|\ophi(\bx;\bw)s'(\langle \bphi_n, \blambda_t\rangle) \< \bphi_n , \blambda - \hat \blambda \> \big|^{\ev}\big]$ is bounded for some $\ev > 1$ and all $t \in [0,1]$,
    whence $\ophi(\bx;\bw)s'(\langle \bphi_n(\bw), \blambda_t\rangle) \bphi_n(\bw) $ is uniformly integrable. 
    We conclude that we may exchange differentiation and expectation, justifying Eq.~\eqref{eq:exchange-deriv-and-expect}.

    We may now replace $\ev$ by 1 in Eq.~\eqref{eq:der_y0_bound} to get
    \begin{equation}
    \begin{aligned}
        |\nabla_{\blambda} \hat f(\bx;\blambda_t)^\top(\blambda - \hat \blambda)|
            &\leq 
            \E[|\ophi(\bx;\bw)|^{r_1}]^{1/r_1}\,\E[|s'(\langle \bphi_n,\blambda_t\>)|^{r_2}]^{1/r_2}\, \E[|\langle \bphi_n , \blambda - \hat \blambda\rangle|^{r_3}]^{1/r_3} 
        \\
            &\leq 
            C (\tau^{Q}\vee \tau^2\eta^{Q'-2})\frac{s(\| \hat \blambda \|_{\bK_n})}{\| \hat \blambda \|_{\bK_n}} \| \blambda - \hat \blambda \|_{\bK_n} \, ,
    \end{aligned}
    \end{equation}
    which is independent of $\bx$ and $t$.
    Combining this bound with Eq.~\eqref{eq:y0-fund-calc},
    we conclude that
    \begin{equation}
        \E_{\bx}[(\hat f(\bx;\blambda) - \hat f(\bx;\hat \blambda))^2]
            \leq
            C (\tau^{2Q}\vee \tau^4\eta^{2(Q'-2)}) \frac{s(\|\hat \blambda\|_{\bK_n})^2}{\|\hat \blambda\|_{\bK_n}^2}\| \blambda - \hat \blambda \|_{\bK_n}^2 \, .
    \end{equation}
    Thus, we have established Eq.~\eqref{eq:pop-continuity_2}.

\subsection{Uniform concentration of the predictor: Proof of Lemma \ref{lem:gen-uniform-concentration-of-predictor}}
\label{app:concentration_predictor}

Throughout the proof, we will denote by $c, c', C, C'$ constants that depend on the constants in assumptions FEAT1, FEAT2 and PEN, but not on $d,n,N$ and $\tau$. The value of these constants is allowed to change from line to line.

    \noindent
    {\bf Step 1. Decoupling. } 
    
    Let $G: \reals \rightarrow \reals_{\geq 0}$ be convex with $G(x) = 0$ for $x \leq 0$ and satisfy for some $\alpha_1,\alpha_2 > 0$,
    \begin{equation}\label{eq:add_psi}
        G(x + y) \leq \alpha_1(G(\alpha_2 x) + G(\alpha_2 y)).
    \end{equation}
    Below denote by $\E$ the expectation with respect to the $\bw_j$'s and the randomization in $\bphi_{n} (\bw)$ (for notational simplicity, we will omit below the subscripts from the expectations). Further, we will denote $\bphi_{n,j}=\bphi_n(\bw_j)=(\phi(\bx_1;\bw_j),\dots,\phi(\bx_n;\bw_j))^{\sT}$
    for the evaluation of the feature map at the $n$ data points, and $\bpsi_{n,j}= \bpsi_n(\bw_j) = \bK^{-1/2}_n\bphi_n(\bw_j)$ for the corresponding
    isotropic random vector.
    Then 
    \begin{align*}
        \E G(\bar S_1)
            &\stackrel{(1)}{=} 
            \E \sup_{\substack{\|\bx\|_2\leq \rz\sqrt{d}\\ \|\bv\|_2 \leq 2 \| \hblambda \|_{\bK_n}}}
                G\Big(  \Big| 
                    \frac1N \sum_{j=1}^N \bar \phi(\bx;\bw_j) s( \< \bpsi_{n,j} , \bv \> )
                    -
                    \E [\bar \phi(\bx;\bw) s( \< \bpsi_n(\bw) , \bv \> )]
               \Big|  \Big)
        \\
        & \stackrel{(2)}{\leq} 
            \E\sup_{\substack{\|\bx\|_2\leq \rz\sqrt{d}\\ \|\bv\|_2 \leq 2 \| \hblambda \|_{\bK_n}}}
            G\Big( 
            \Big|
                \frac1N \sum_{j=1}^N 
                \Big[
                \bar \phi(\bx;\bw_j) s( \< \bpsi_{n,j} , \bv \> )
                -
                \bar \phi(\bx;\bw_j') s( \< \bpsi_{n,j} ' , \bv \> )]
                \Big] \Big|
            \Big)
        \\
        & \stackrel{(3)}{=} 
             2 \E\sup_{\substack{\|\bx\|_2\leq \rz\sqrt{d}\\ \|\bv\|_2 \leq 2 \| \hblambda \|_{\bK_n}}}
            G\Big(
                \frac1N \sum_{j=1}^N 
                \sigma_j
                \Big[
                \bar \phi(\bx;\bw_j) s( \< \bpsi_{n,j} , \bv \> )
                -
                \bar \phi(\bx;\bw_j') s( \< \bpsi_{n, j} ' , \bv \> )]
                \Big]
            \Big)
        \\
        & \stackrel{(4)}{\leq}
            4\alpha_1 \E G\Big(\alpha_2 \sup_{\substack{\|\bx\|_2\leq \rz\sqrt{d}\\ \|\bv\|_2 \leq 2 \| \hblambda \|_{\bK_n}}} \underbrace{\frac1N \sum_{j=1}^N 
                    \sigma_j
                    \bar \phi(\bx;\bw_j) s( \< \bpsi_{n,j} , \bv \> )}_{=:T(\bx,\bv)}\Big),
    \end{align*}
    where (1) uses that $G$ is nondecreasing; (2) uses that $G (| \, \cdot\,  |)$ is convex and Jensen's inequality; in (3), we denoted $\sigma_j,j=1,\ldots,N$ independent Rademacher random variables and used that the distribution is symmetric and $G ( x ) = 0$ for $x\leq 0$; (4) uses Eq.~\eqref{eq:add_psi} and that $T(\bx,\bv)$ is a symmetric random variable and $G$ is a nondecreasing function.

    \noindent
    {\bf Step 2. Concentration of $T(\bx , \bv)$. } 
    
    We bound the tail of $T$ for fixed $\bx,\bv$ with $\|\bx\|_2\leq \rz\sqrt{d}$, $\|\bv\|_2 \leq R$.
    Denote by $X_j$ the terms in the sum in $T(\bx , \bv)$, i.e., $X_j = \sigma_j \ophi(\bx;\bw_j) s( \< \bpsi_{n,j} , \bv \> ) $.
    By Eq.~\eqref{eq:poly-growth},
    there exists $C$ such that $|s(x_1)| \leq C|s(x_2)|( 1 + |x_1/x_2|)^{Q-1}$, where $Q = Q_1 \vee Q_2$ (for simplicity we have loosened the bound for when $|x_1/x_2| < 1$).
    By H\"older inequality, we have for any $k \geq 1$ (potentially non-integer) 
    \begin{equation}\label{eq:X_i_moment_bound}
    \begin{aligned}
        \E[|X_i|^k]
            &\leq 
            C\E[|\bar \phi(\bx;\bw_j)|^k \cdot|s(\tau R)|^k\cdot (1 + |\< \bv , \bpsi_{n,j} \>|^{{Q}-1} / (\tau R)^{{Q}-1} )^k]
        \\        
            &\leq 
            C \tau^k |s(\tau R)|^k\E\Big[|\bar\phi(\bx;\bw_j)/\tau|^{Qk} \Big]^{1/Q} \E\Big[ (1 + |\< \bv , \bpsi_{n,j} \>|^{{Q}-1} / (\tau R)^{{Q}-1} )^{Qk/(Q-1)}\Big]^{(Q-1)/Q}
        \\
            &\leq 
            \tau^ks(\tau R)^k(C' k)^{{Q}k/2},
    \end{aligned}
    \end{equation}
    where $C'$ depends only on ${Q}$, and we used FEAT1, i.e., $\bar \phi(\bx;\bw_j)$ and $\bpsi_{n,j}$ are $\tau^2$-sub-Gaussian. 
    
    For $Q\leq 2$, we have directly by Eq.~\eqref{eq:X_i_moment_bound} and \cite[Theorem 2.10]{boucheron2013concentration},
    \begin{equation}\label{eq:bound_Txv_1}
     \P\Big(T ( \bx , \bv ) \geq t\Big) \leq  \exp\Big\{-c' N \Big(\frac{t^2}{\tau^2 s(\tau R)^2} \wedge \frac{t}{\tau s(\tau R)}\Big) \Big\} \, .
    \end{equation}
    For $Q > 2$, define for $M > 0$ the truncated random variable 
     $X_i^M = \sign(X_i)(|X_i| \wedge \tau s(\tau R) M^{Q} )$.
    Then setting $\ell = 2k + 2{Q} - 4$,
    we have
    \begin{equation}\label{eq:X_i_M_moment_bound}
    \begin{aligned}
        \E[|X_i^M /N|^k]
            &\leq 
            N^{-k}\E[|X_i|^{\ell/{Q}}(\tau s(\tau R) M^{Q})^{k-\ell/{Q}}]
            \leq 
            N^{-k} \tau^k s(\tau R)^k (C'\ell/{Q})^{\ell/2} M^{({Q}-2)(k-2)}
        \\
            &\leq
            (C'k)^k \frac{\tau^2s(\tau R)^2}{N^2} \Big(\frac{\tau s(\tau R) M^{{Q}-2}}{N}\Big)^{k-2} \, .
    \end{aligned}
    \end{equation}
    We deduce again by \cite[Theorem 2.10]{boucheron2013concentration} with  Eq.~\eqref{eq:X_i_M_moment_bound}, that 
    \[
     \P\Big(\frac1N \sum_{i=1}^N X_i^M \geq t\Big) \leq  \exp\Big\{-c' N \Big(\frac{t^2}{\tau^2 s(\tau R)^2} \wedge \frac{t}{\tau s(\tau R)M^{{Q}-2}}\Big) \Big\} \, .
    \]
    Furthermore, using the bound~\eqref{eq:X_i_moment_bound}, we have
    \[
    \P \Big(|X_1| \geq \tau s(\tau R)M^{Q} \Big) \leq \inf_{k\ge 1} M^{-kQ} \E \Big[(|X_1|/ \tau s (\tau R) )^k \Big] = \exp \Big\{ - c ' M^2 \Big\} \, .
    \]
    Combining the above two displays and taking $M = (Nt)^{1/{Q}}/(\tau s(\tau R))^{1/{Q}}$, we conclude that for $Q >2$ and fixed $\|\bx\|_2\leq \rz \sqrt{d} $, $\|\bv\|_2 \leq R$, we have
    \begin{equation}\label{eq:bound_Txv_2}
    \begin{aligned}
        \P\Big(T(\bx,\bv) \geq t\Big)
            \leq&~ \P\Big(\frac1N \sum_{i=1}^N X_i^M \geq t\Big)
            +
            N \P\Big(|X_1| \geq \tau s(\tau R)M^{Q} \Big) \,  \\
            \leq &~
            N \exp\Big\{ -c' \Big(\frac{Nt^2}{\tau^2 s(\tau R)^2} \wedge \frac{(Nt)^{2/{Q}}}{\tau^{2/{Q}} s(\tau R)^{2/{Q}}} \Big)\Big\} \, .
            \end{aligned}
    \end{equation}

    \noindent
    {\bf Step 3. Uniform concentration of $T(\bx , \bv)$ on $\| \bx \|_2 \leq \rz\sqrt{d}$, $\| \bv \|_2 \leq R$. } 
    
    We now evaluate the concentration of $T$ uniformly over $\|\bx\|_2\leq \rz \sqrt{d} $, $\|\bv\|_2 \leq R$.
    By the tail bound on $L(\bw)$ in FEAT2 and that $\bpsi_{n,j}$ is a $\tau^2$-sub-Gaussian random vector by FEAT1, 
    there exists constants $C,c>0$ independent of $\tau$, such that for all $\Delta \geq 1$ \cite{hsu2012},
    \begin{equation}\label{eq:gen-feat-norm-bound}
        \P \Big( \max_{j \in [N]} \| \bpsi_{n,j} \|_2 \vee L(\bw_j) \leq C \tau \sqrt{n} \Delta \Big) \geq 1 - CN \exp ( - c n \Delta^2 ).
    \end{equation}
    Let $0 < \zeta_2 \leq 1 \leq \zeta_1$ and $q' = q_1 \wedge q_2$.
    Then, by assumption PEN,
    for any $|x_1|,|x_2| \leq \tau R \Delta \zeta_1$, we have
    \begin{align*}
        |s(x_1) - s(x_2)| 
            &\leq 
            ( |s(\zeta_2 \tau R \Delta)| + |s(- \zeta_2\tau R \Delta)|)
            +
            \Big( \sup_{|x| \in [\zeta_2 \tau R \Delta, \zeta_1 \tau R \Delta]} s'(x) \Big)|x_1-x_2| 
        \\
            &\leq 
            C s(\tau R \Delta) \zeta_2^{q'-1} 
                + 
                C \frac{s(\tau R \Delta)}{\tau R \Delta}\Big(\zeta_1^{(Q - 2) \vee 0} \vee \zeta_2^{(q' - 2) \wedge 0} \Big) |x_1 - x_2|
        \\ 
            &\leq 
            C s(\tau R \Delta) \Big( \zeta_2^{q' - 1} + (\zeta_1^{(Q - 2) \vee 0} \vee \zeta_2^{(q' - 2) \wedge 0} ) |x_1 - x_2| / (\tau R \Delta) \Big).
    \end{align*}
    Let $\zeta_1 = C \sqrt{n}$. On the event \eqref{eq:gen-feat-norm-bound} and for $\| \bv \|_2 \leq R$, we have $| \< \bpsi_{n,j} , \bv \> | \leq  C \tau \sqrt{n} R\Delta = \tau R \Delta \zeta_1$. Furthermore, by FEAT2, we have $| \ophi ( \bx ; \bw ) | \leq  L (\bw) (1 + \| \bx \|_2)$. For convenience, let us denote $\tilde \bx = \bx / (\rz\sqrt{d})$ so that $\| \tilde \bx \|_2 \leq 1$.
    On the event \eqref{eq:gen-feat-norm-bound}, 
    whenever $\|\tilde \bx_1\|_2,\|\tilde \bx_2\|_2,\|\tilde \bv_1\|_2,\|\tilde \bv_2\|_2 \leq 1$, we have (the constants $C$ below may depend on $\rz$ but not on $\tau$, $\eta$)
    \begin{align*}
        &\,\,\,\, \,\,\,\, |T(\bx_1,R\tilde \bv_1) - T(\bx_2,R\tilde \bv_2)| \\
            &\leq 
            \frac1N\sum_{j=1}^N 
                L(\bw_j) |s(\< \bpsi_{n,j}, R \tilde \bv_1 \>)| \|\bx_1-\bx_2\|_2 +
            \frac{C}N\sum_{j=1}^N 
                L(\bw_j)(1 + \| \bx_2 \|_2) | s(\< \bpsi_{n,j}, R \tilde \bv_1\>) - s( \< \bpsi_{n,j} , R \tilde \bv_2 \>)|
        \\
            &\leq 
            \frac{C}N\sum_{j=1}^N 
                \tau \sqrt{nd} \Delta s(\tau R \Delta)(1 + (C\sqrt{n})^{Q-1})\|\tilde \bx_1 - \tilde \bx_2\|_2
        \\
            &\qquad\qquad+
            \frac{C}N\sum_{j=1}^N 
                \tau \sqrt{nd} \Delta s(\tau R \Delta)\big(\zeta_2^{q' - 1} + (\zeta_1^{(Q - 2) \vee 0} \vee \zeta_2^{(q' - 2) \wedge 0} ) \sqrt{n}\|\tilde\bv_1 - \tilde\bv_2\|_2 \big)
        \\
            &\leq 
            C\tau \Delta s(\tau R \Delta)
            \Big(
                n^{Q/2}\sqrt{d}\|\tilde \bx_1 - \tilde\bx_2\|_2  + n\sqrt{d} (\zeta_1^{(Q - 2) \vee 0} \vee \zeta_2^{(q' - 2) \wedge 0} ) \| \tilde \bv_1 - \tilde \bv_2 \|_2
                + \sqrt{nd} \zeta_2^{q'-1}
            \Big).
    \end{align*}
    Fix $\eps > 0$.
    Let $\zeta_2 = (\eps/(3 C\sqrt{nd}))^{1/(q'-1)}$, 
    and define
    \[
    \eps ' = \frac{\eps}{3 C n^{Q/2} \sqrt{d} +  3C n\sqrt{d} (\zeta_1^{(Q - 2) \vee 0} \vee \zeta_2^{(q' - 2) \wedge 0} ) } \, ,
    \]
    so that $|T(\bx_1,R\tilde \bv_1) - T(\bx_2,R\tilde \bv_2)| \leq \tau \Delta s(\tau R \Delta) \eps$ as soon as $\| \tilde\bx_1 - \tilde\bx_2 \|_2 \leq \eps '$ and $\| \tilde \bv_1 - \tilde \bv_2 \|_2 \leq \eps '$. Let $\cN_1$ and $\cN_2$ be $\eps'$-net of $\Ball_2^d (1)$ and $\Ball_2^n (1)$ respectively, where we define $\Ball_2^k (r) := \{ \bu \in \R^k : \| \bu \|_2 \leq r \}$.
    Note that $\log (|\cN_1||\cN_2|) \leq (d+n) \log\Big(\frac{3}{\eps'}\Big)$.
    Then taking $\eps = t s(\tau R) / (2\Delta s(\tau R \Delta)) \geq C ' t \Delta^{-Q} $ and recalling Eqs.~\eqref{eq:bound_Txv_1} and \eqref{eq:bound_Txv_2}, we have
    \begin{align*}
        &\P\Big(
            \sup_{\substack{\|\tilde \bx\|_2\leq 1\\ \|\bv\|_2\leq R}} T(\tilde \bx,\bv) \geq t \tau s(\tau R) 
        \Big)    
        \\
        &\qquad\qquad\leq 
            \P\Big(
                \sup_{\tilde \bx \in \cN_1,\tilde \bv \in \cN_2} T(\tilde \bx,R\tilde \bv) \geq t \tau s(\tau R)/2
            \Big)
            + 
            \P \Big( \max_{j \in [N]} \| \bpsi_{n,j} \|_2 \vee L(\bw_j)  \geq C \tau \sqrt{n} \Delta\Big)
        \\
        &\qquad\qquad\leq
           \exp\Big\{
                \log N 
                +
                2n \log\Big(\frac{3}{\eps'}\Big)
                -
                c' \Big(Nt^2 \wedge (Nt)^{(2/Q) \wedge 1 } \Big)
            \Big\} + CNe^{-cn \Delta^2}
        \\
        &\qquad\qquad\leq 
             \exp\Big\{
                \log N
                +
                C n \log(n)
                + 
                C n \log(1/\eps)
                -
                c' \Big(Nt^2 \wedge (Nt)^{(2/Q) \wedge 1} \Big)
            \Big\} + Ne^{-cn\Delta^2}
        \\
        &\qquad\qquad\leq 
             \exp\Big\{
                \log N
                +
                C n \log(n)
                +
                C n \log(\Delta^{Q}/t) 
                -
                c' \Big(Nt^2 \wedge (Nt)^{(2/Q) \wedge 1 } \Big)
            \Big\} + Ne^{-cn\Delta^2},
    \end{align*}
    where in the second to last inequality we have used the definition of $\eps'$ and adjusted constants appropriately,
    and in the last inequality we have used the definition of $\eps$. Consider $t \geq 1 / N$ and set $\Delta = (tN)^{1/Q} \geq 1$. Then
\begin{equation}\label{eq:tail_bound_T}
\begin{aligned}
  &~  \P\Big(
            \sup_{\substack{\|\tilde \bx\|_2\leq 1\\ \|\bv\|_2\leq R}} T(\tilde \bx,\bv) \geq t \tau s(\tau R) 
        \Big)  \\
        \leq&~ \exp \Big\{ C n \log(N) -  c' \Big(Nt^2 \wedge (Nt)^{(2/{Q}) \wedge 1} \Big)
            \Big\} + N \exp \Big\{ - c n N^{2/Q} t^{2 /Q} \Big\}.
            \end{aligned}
\end{equation}

        \noindent
    {\bf Step 4. Concluding the proof. } 
    
    For a constant $\Tilde C >0$ that will be set sufficiently large, define
    \begin{align*}
        \eps_1 = \Tilde C \tau s (\tau R) \Big( \sqrt{\frac{n\log N}{N}} \vee \frac{(n\log N)^{({Q}/2) \vee 1}}{N}\Big) \, .
    \end{align*}
    Consider $G ( x) = ( x - \eps_1 / 2)_+$ which is convex and verifies Eq.~\eqref{eq:add_psi} with $\alpha_1 = 1$ and $\alpha_2 = 2$. Then we have
    \[
    \begin{aligned}
    \P ( \bar S_1 \geq \eps_1 ) \leq \frac{2}{\eps_1} \E [G (  \bar S_1 )] \leq&~  \frac{8}{\eps_1} \E G\Big(2 \sup_{\substack{\|\bx\|_2\leq \rz\sqrt{d}\\ \|\bv\|_2 \leq R}}T (\bx , \bv ) \Big) \\ \leq&~ \frac{16 \tau s( \tau R) }{\eps_1} \int_{0}^\infty \P \Big(
            \sup_{\substack{\|\bx\|_2\leq \rz\sqrt{d} \\ \|\bv\|_2\leq R}} T(\bx,\bv) \geq  t\tau s( \tau R)+ \eps_1/4 
        \Big) \de t \, .
    \end{aligned}
    \]
    Using Eq.~\eqref{eq:tail_bound_T} and the inequalities $N (\eps_1 / \tau s(\tau R))^2 \geq \Tilde C^2 n \log N $ and $(N \eps_1 / \tau s(\tau R) )^{(2/Q)\wedge 1} \geq \Tilde C^{(2/Q)\wedge 1} n \log N $, we get
    \[
    \P ( \bar S_1 \geq \eps_1 ) \leq C \exp \Big\{ C n \log(N) - c \Tilde C^{(2/Q)\wedge 1} n \log (N) \Big\} + C N \exp \Big\{ - c \Tilde C^{2/Q} n (n \log (N))^{1 \vee (2/Q) } \Big\}.
    \]
    Taking $\Tilde C$ sufficiently large, $R = 2 \| \hblambda \|_{\bK_n}$ and using that $s(2\tau \| \hblambda \|_{\bK_n}) \leq C s(\| \hblambda \|_{\bK_n}) \tau^{Q-1}$ (where we recall that we assumed $\tau \geq 1)$, we deduce that there exists constants $c', C' >0$ that depend only on the constants of FEAT1, FEAT2 and PEN (except $\tau$) such that
    \[
    \bar S_1  \leq C' \tau^Q s ( \| \hblambda \|_{\bK_n}) \Big( \sqrt{\frac{n\log N}{N}} \vee \frac{(n\log N)^{({Q}/2) \vee 1}}{N}\Big) \, ,
    \]
    with probability at least $1 - C' N^{-c' n}$.
    The proof is complete.

\subsection{Point-wise concentration of dual gradient: Proof of Lemma \ref{lem:point-wise-concentration}}
\label{app:concentration_gradient}

   The proof follows from the same argument as in the proof of Lemma \ref{lem:gen-uniform-concentration-of-predictor} and we will only highlight the differences. 
   
    \noindent
    {\bf Step 1. Decoupling. } 
    
    First notice that we can rewrite 
    \[
    \bar S_2 = \sup_{\| \bb \|_2 \leq 1} \Big\{ \frac1N\sum_{j=1}^N \< \bb , \bpsi_{n,j} \> s(\<\bphi_{n,j}, \hblambda\>) - \E[ \< \bb , \bpsi_n(\bw) \> s(\< \bphi_n ,\hblambda\>)] \Big\} \, .
    \]
    Denote 
    \[
    T( \bb ) = \frac{1}{N} \sum_{j = 1}^N \sigma_j \< \bb , \bpsi_{n,j} \> s(\<\bphi_{n,j} ,\hblambda\>)\, .
    \]
     For $G$ with the properties listed in step $1$ in the proof of Lemma \ref{lem:gen-uniform-concentration-of-predictor}, we have
    \[
    \E G (\bar S_2 ) \leq 4 \alpha_1 \E G \Big( \alpha_2 \sup_{\|\bb \|_2 \leq 1} T ( \bb ) \Big).
    \]

\noindent
    {\bf Step 2. Concentration of $T(\bb)$. } 

Denote $X_j$ the terms in the sum in $T (\bb )$, i.e., $X_j = \sigma_j \< \bb , \bpsi_{n,j} \> s(\<\bphi_{n,j} ,\hblambda\>)$. By FEAT1 and PEN, and denoting $Q = Q_1 \vee Q_2$, $R =\| \hblambda \|_{\bK_n}$, we have by H\"older inequality, for any $k \geq 1$ and $\| \bb \|_2 \leq 1$,
    \[
    \begin{aligned}
        \E[|X_j|^k]
            &\leq 
            C \tau^{k} s(\tau R )^{k} \E\Big[|\< \bb , \bpsi_{n,j} \> / \tau |^{Qk} \Big]^{1/Q} \E\Big[ (1 + |\< \blambda , \bphi_{n,j} \>|^{Q-1} / (\tau R )^{{Q}-1} )^{Qk/(Q-1)}\Big]^{(Q-1)/Q}
        \\
            &\leq 
            \tau^{k} s(\tau R )^{k} (C' k)^{Qk/2} \, .
    \end{aligned}
    \]
    For $Q>2$ and $M > 0$,
    define $X_j^M = \sign(X_j)(|X_j| \wedge \tau s(\tau R) M^{Q} )$.
    Then setting $\ell = 2k + 2{Q} - 4$,
    we have
    \[
    \begin{aligned}
        \E[|X_j^M /N|^k]
            &\leq 
            N^{-k}\E[|X_j|^{\ell/{Q}}(\tau s(\tau R) M^{Q})^{k-\ell/{Q}}]
            \leq 
            (C'k)^k \frac{\tau^2 s(\tau R)^2}{N^2} \Big(\frac{\tau s(\tau R) M^{{Q}-2}}{N}\Big)^{k-2} \, .
    \end{aligned}
    \]
    Following step 2 in the proof of Lemma \ref{lem:gen-uniform-concentration-of-predictor}, in particular taking $M = (Nt)^{1/Q}$ in the case of $Q > 2$, we get
    \begin{equation}
        \P\Big(T(\bb) \geq t \tau s ( \tau R) \Big)
            \leq 
       N \exp\Big\{ -c' \Big(Nt^2 \wedge (N t)^{(2/Q) \wedge 1}\Big)\Big\}  \, .
    \end{equation}

    \noindent
    {\bf Step 3. Uniform concentration of $T(\bb)$ on $\| \bb \|_2 \leq 1$. }

    We consider the event for $\Delta \geq 1$,
    \begin{equation}\label{eq:good_ev_dual_grad}
        \P \Big( \max_{j \in [N]} \| \bpsi_{n,j} \|_2  \leq C \tau \sqrt{n} \Delta \Big) \geq 1 - CN \exp ( - c n \Delta^2 ) \, ,
    \end{equation}
    where we used that the $\bpsi_{n,j}$ are $\tau^2$-sub-Gaussian by FEAT1. On this event, we have $| \< \bphi_{n,j} , \hblambda \> | \leq \| \bpsi_{n,j} \|_2 \| \hblambda \|_{\bK_n} \leq CR \tau \sqrt{n} \Delta$ and by PEN, we get
    \[
    \begin{aligned}
    |T (\bb_1 ) - T ( \bb_2 )| \leq&~ \frac{1}{N} \sum_{j = 1}^N  \| \bpsi_{n,j} \|_2 |s ( \< \bphi_{n,j} , \blambda \> ) | \| \bb_1 - \bb_2 \|_2 
    \leq  C \tau \Delta s ( \tau R \Delta ) n^{Q/2}\| \bb_1 - \bb_2 \|_2  \, .
    \end{aligned}
    \]
    Fix $\eps > 0 $ and $\eps ' = \eps / (C n^{Q/2 } )$, so that $| T (\bb_1) - T ( \bb_2 ) | \leq \tau \Delta s (\tau R \Delta) \eps$ as soon as $\| \bb_1 - \bb_2 \|_2  \leq \eps '$. Taking $\eps = t s(\tau R) / (2 \Delta s(\tau R \Delta ) ) \geq C' t \Delta^{-Q}$ and $\Delta = (tN )^{1/Q}$, we have for $t \geq 1 / N$,
    \begin{equation}
\begin{aligned}
  &~  \P\Big(
            \sup_{\|\bb\|_2\leq 1} T(\bb) \geq t \tau s(\tau R)
        \Big)  \\
        \leq&~ \exp \Big\{ C n \log(N) -  c' \Big(Nt^2 \wedge (Nt)^{(2/{Q})\wedge 1} \Big)
            \Big\} + N \exp \Big\{ - c n N^{2/Q} t^{2 /Q} \Big\}.
            \end{aligned}
\end{equation}
    
        \noindent
    {\bf Step 4. Concluding the proof. }
    
    Following the same argument as in Lemma \ref{lem:gen-uniform-concentration-of-predictor}, there exists constants $c ', C' > 0$ that depend only on the constants of PEN such that
    \[
    \bar S_2 \leq  C' \tau^Q s(\|\blambda\|_{\bK_n}) \Big( \sqrt{\frac{n\log N}{N}} \vee \frac{(n\log N)^{Q/2\vee 1}}{N}\Big), 
    \]
    with probability at least $1 - C' e^{-c ' n}$.

\subsection{Uniform lower bound on the dual Hessian: Proof of Lemma \ref{lem:lower-hessian-no-small-ball}}
\label{app:lower_bound_Hessian}

  Define $q = 3 \vee q_1 \vee q_2$, where $q_1$ and $q_2$ are as in assumption PEN, and
  define $s_0$ as
  \begin{equation}
    s_0(t) 
      = 
      1 \wedge t \wedge t^{q_1 - 2} \wedge t^{q_2 - 2}.
  \end{equation}
  The function $s_0(t)$ is $(1 \vee |q_1-2| \vee |q_2 - 2|)$-Lipschitz.
  By assumption PEN, there exists $c > 0$ depending only on the constants in that assumption such that for all $t$
  \begin{equation}
    s'\big(\| \hat \blambda\|_{\bK_n} t\big) \geq c \frac{s(\| \hat \blambda\|_{\bK_n})}{\| \hat \blambda \|_{\bK_n}}  s_0(t).
  \end{equation}
  Thus, 
  for $\blambda \in \R^n$ such that $1/2 \leq \| \blambda \|_{\bK_n} / \| \hat \blambda \|_{\bK_n}\leq 2$ and denoting $\bv = \bK_n^{1/2} \blambda / \| \hat \blambda \|_{\bK_n} $, we have 
  \[
  \begin{aligned}
  - \bK_n^{-1/2} \nabla^2 F_N (  \blambda ) \bK_n^{-1/2} =&~ \frac 1N \sum_{j = 1}^N s ' ( \< \bK_n^{1/2} \blambda , \bpsi_{n,j} \>) \bpsi_{n,j} \bpsi_{n,j}^\sT 
  \succeq  c \frac{s (   \| \hat \blambda \|_{\bK_n})}{ \| \hat \blambda \|_{\bK_n} } \bH_N(\bv) \, ,
  \end{aligned}
  \]
  where
  \[
  \bH_N(\bv) = \frac 1N \sum_{j = 1}^N s_0 ( \< \bv , \bpsi_{n,j} \>) \bpsi_{n,j} \bpsi_{n,j}^\sT \, .
  \]

  Define $\obH(\bv) := \E_{\bw,\phi} [ s_0 ( \< \bv , \bpsi_n \>)
  \bpsi_n \bpsi_n^\sT ] $.
  As already explained in Remark \ref{rmk:FEAT3}, the fact that $\bpsi_n$ is isotropic, together with
  assumption FEAT3 imply that, for any two unit vectors $\bu,\bu'$, we have
   $ \P\big(|\<\bu,\bpsi_n\>|\in [\eta,C] ; \; |\<\bu',\bpsi_n\>|\in [\eta,C]\big)\ge c/2$\, .
  Using the fact that  $\| \bv \|_2 \in [1/2,2]$, we deduce that
  \begin{align*}
    \P\big(|\<\bu,\bpsi_n\>|\in [\eta,C] ; \; |\<\bv,\bpsi_n\>|\in [\eta/2,2C]\big)\ge \frac{c}{2}\, .
  \end{align*}
  Thus,
  \begin{align*}
    \<\bu, \obH(\bv) \bu\>
      = 
      \E_{\bw,\phi} [ s_0 ( \< \bv , \bpsi_n \>) \< \bu , \bpsi_n \>^2 ]
      \geq 
      C \eta^{2} \min\{\eta,\eta^{q_1-2},\eta^{q_2-2}\}.
  \end{align*}
  Thus, $\lambda_{\min} (\obH(\bv) ) \geq \delta_0(\eta)$ for $\delta_0(\eta):= C\eta^{3\vee q_1\vee q_2}$. 
  Furthermore, by FEAT1 and that $0\leq s_0 (x ) \leq 1$, 
  we have $s_0 ( \< \bv , \bpsi_{n,j} \>)^{1/2} \bpsi_{n,j}$ are $C\tau^2$-sub-Gaussian. 
  Applying Lemma \ref{lem:matrix_concentration} to the vectors $s_0 ( \< \bv , \bpsi_{n,j} \>)^{1/2} \bpsi_{n,j}/(C\tau)$ with $\delta = \delta_0(\eta) / (2C^2 \tau^2)$ and $t = \sqrt{N/n}\,\delta_0(\eta) / (2C'C^2\tau^2) - 1$, and noting that $\delta_0(\eta)/(2C^2\tau^2)\le 1$, we conclude that there exists a constant $c>0$ such that for $N \geq 16 C'^2 C^4 \tau^4 n / \delta_0(\eta)^2$
  \begin{align*}
    \P ( \lambda_{\min} ( \bH_N(\bv) ) \leq \delta_0 (\eta) /2 ) \leq \P ( \| \bH_N(\bv)  - \obH(\bv)  \|_{\op} \geq \delta_0 (\eta) /2 )
    \leq 2 \exp \big( - c N\delta_0(\eta)^2/\tau^4 \big).
  \end{align*}
  
  Denote $\bPsi_n = [ \bpsi_{n,1} , \ldots, \bpsi_{n,N} ] \in \R^{n \times N}$. 
  For constant $C>0$ sufficiently large, consider the event 
  \[
    \cE_0 = \Big\{ \{ \max_{j \in [N]} \| \bpsi_{n,j} \|_2 \leq C \sqrt{n}\,\tau \} \cup \{ \| \bPsi_n \|_\op^2 \leq CN\tau^2 \} \Big\} \, .
  \]
  We have for $N \geq n$,
  \[
  \begin{aligned}
    \P (\cE_0^c ) \leq &~ \P \big(  \max_{j \in [N]} \| \bpsi_{n,j} \|_2 \geq C \sqrt{n}\,\tau \big) +
    \P \big( \| \bPsi_n \|_\op^2 \geq CN\tau^2 \big) \leq \exp ( - cN )\, ,
  \end{aligned}
  \]
  where we used that $ \bpsi_{n,j}$ are $\tau^2$-subgaussian by FEAT1 to bound the first term, and that $\E_{\bw,\phi} [ \bpsi_n \bpsi_n^\sT ] = \id_n$ and Lemma \ref{lem:matrix_concentration} to bound the second term.

  On the event $\cE_0$, we have
  \[
  \begin{aligned}
  \| \bH_N(\bv_1) - \bH_N(\bv_2) \|_{\op} \leq&~ \frac 1N  \| \bPsi_n \|_{\op}^2 \max_{j \in [N]} | s_0 ( \< \bpsi_{n,j} , \bv_1 \> ) - s_0 ( \< \bpsi_{n,j} , \bv_2  \> )| \\
  \leq&~ C \tau^2\max_{j \in [N]} |  \< \bpsi_{n,j} , \bv_1 - \bv_2  \> | \leq C \sqrt{n} \, \tau^3 \| \bv_1 - \bv_2 \|_2 \, ,
  \end{aligned}
  \]
  where we used that $s_0$ is Lipschitz. Let $\cN_n$ be a minimal $\eps$-net of $\{ \bv \in \R^n : \| \bv \|_2 \in [1/2,2] \}$, with $\eps = \delta_0(\eta) / (4 C\sqrt{n} \tau^3)$, so that $\| \bH_N(\bv_1) - \bH_N(\bv_2) \|_{\op} \leq \delta_0(\eta)/4$ as soon as $\| \bv_1 - \bv_2 \|_2 \leq \eps$. 
  We have
  \begin{align*}
    \P ( \min_{\| \bv \|_2 \in [1/2,2]} \lambda_{\min} ( \bH_N(\bv) ) \leq \delta_0 / 4 ) \leq &~ | \cN_n|
                                                                                         \max_{\bv\in\cN_n}\P ( \lambda_{\min} ( \bH_N(\bv) ) \leq \delta_0/2 ) + \P( \cE_0^c )\\
    &
      \le C \exp \big\{ C n \log (1/\eps) - cN\delta_0(\eta)^2/\tau^4 \big\}\\
                                                                                       &\le C \exp \big\{ C n \log ( n)+C n\log \tau+C n\log (1/\eta) - cN\delta_0(\eta)^2/\tau^4 \big\}\\
    &\le  C \exp \big\{ C n \log ( n) - cN\delta_0(\eta)^2/\tau^4 \big\}\, ,
  \end{align*}
  where the last inequality follows since by assumption $\tau\le n^C$, $\eta\ge n^{-C}$.
  By taking $N \geq C'(\tau^4/\delta_0(\eta)^2) n \log(N)$ for a sufficiently large constant $C'$ we obtain
  $ \min_{\| \bv \|_2 \in [1/2,2]} \lambda_{\min} ( \bH_N(\bv) ) \leq \delta_0 / 4$ with probability
  at least $1-C'N^{-c'n}$ as claimed.

\section{The latent linear model: Proof of Proposition \ref{prop:latent-RF-infinite}}
\label{app:proof_latent}

We consider the latent linear model presented in the main text (Section \ref{sec:latent}). We have $\bphi_{n,j} = \bphi_n ( \bw_j) = \bX \bw_j + \bz_j$, where $\bz_j = (z_{1j} , \ldots, z_{nj} ) \in \R^n$ are the iid features noise. Denote $\bW = [ \bw_1 , \ldots , \bw_N ]^\sT \in \R^{N \times d}$ and $\bZ = [ \bz_1 , \ldots , \bz_N ] \in \R^{n \times N}$. The features matrix is given by
\begin{equation}\label{eq:feature_matrix_latent}
\bPhi_n = [ \bphi_{n,1} , \ldots , \bphi_{n,N} ] = \bX \bW^\sT + \bZ  \in \R^{n \times N} \, .
\end{equation}
The random features and infinite-width predictors are given by 
\[
\hat f ( \bx ; \hblambda ) = \E_{\bw, \phi} [ \< \bx , \bw \> s ( \< \bphi_n ( \bw) , \hblambda \>)], \qquad \hat f_N ( \bx ; \hblambda_N ) =  \frac{1}{N} \< \bx , \bW^\sT s ( \bPhi_N^\sT \hblambda_N ) \> \, ,
\]
where $s$ is applied component-wise.
The distance between $\hat f$ and $\hat f_N$ is therefore given explicitly by
\begin{equation}\label{eq:dist_latent}
\begin{aligned}
\| \hat f_N ( \cdot ; \hblambda_N ) - \hat f (\cdot ; \hblambda ) \|_{L^2}^2 =&~  \E_{\bx} \Big[ \Big( N^{-1}  \< \bx , \bW^\sT s ( \bPhi_N^\sT \hblambda_N )  - \E_{\bw, \phi} [ \< \bx , \bw \> s ( \< \bphi_n ( \bw) , \hblambda \>)] \Big)^2 \Big] \\
=&~ \| N^{-1} \bW^\sT s ( \bPhi_N^\sT \hblambda_N ) - \E_{\bw, \phi} [ \bw s ( \< \bphi_n ( \bw) , \hblambda \>)] \|_{\bSigma_x}^2\, .
\end{aligned}
\end{equation}

The following lemma is the key result that allows us to improve on Theorem \ref{thm:gen-conv-to-pop} and gain a factor $1/\sqrt{n}$.

\begin{lemma}\label{lem:interpolation-trick}
    Assume $\lambda_{\max} (\bSigma_x ) \leq \kappa^2$ and $\sigma_{\min} ( \bX ) \geq c_0 \sqrt{n}$ (the minimum singular value of $\bX$). 
    Then if $n \geq 2d /c_0^2$, we have
    \begin{equation}\label{eq:interpolation-trick}
        \| \hat f_N ( \cdot ; \hblambda_N ) - \hat f (\cdot ; \hblambda ) \|_{L^2} \leq \frac{2 \kappa}{ c_0 \sqrt{n}}  \| N^{-1} \bZ s(\bPhi_n^\top \hat \blambda_N) - \E_{\bw,\bz} [\bz s(\< \bphi_n ( \bw , \bz ), \hblambda \>)]\|_2.        
    \end{equation}
\end{lemma}
We remark that Lemma \ref{lem:interpolation-trick} is entirely deterministic. 
This may seem surprising because $\bW^\top$ and $\bZ$ are random.
In fact, the proof of Lemma \ref{lem:interpolation-trick} relies on a deterministic argument which uses the fact that both the infinite-width and random features predictors interpolate the training data, i.e., for $i\leq n$, 
\begin{equation*} 
\E_{\bw, \phi} [ \phi ( \bx_i ; \bw) s ( \< \bphi_n ( \bw) , \hblambda ) ] =  \frac{1}{N} \sum_{ j =1}^N \phi ( \bx_i ; \bw_j) s ( \< \bphi_{n,j} , \hblambda_N \> ) = y_i \, ,
\end{equation*}
or equivalently, 
\begin{equation}\label{eq:interpol_app}
  \E_{\bw,\bz}[(\bX \bw + \bz)s(\< \bphi_n(\bw) , \hat \blambda \> )]
  =
  \frac1N(\bX \bW^\top + \bZ)s( \bPhi_N^\top \hat \blambda_N )
  =
  \by
\end{equation}
Note that here we have used the form of the infinite-width primal problem (see Section \ref{sec:Primal}).

\begin{proof}[Proof of Lemma \ref{lem:interpolation-trick}]
By Eq.~\eqref{eq:dist_latent} and the bound $\lambda_{\max} (\bSigma_x ) \leq \kappa^2$, we have
\[
 \| \hat f_N ( \cdot ; \hblambda_N ) - \hat f (\cdot ; \hblambda ) \|_{L^2} \leq \kappa \Big\| N^{-1} \bW^\sT s ( \bPhi_N^\sT \hblambda_N ) - \E_{\bw, \phi} [ \bw s ( \< \bphi_n ( \bw) , \hblambda \>)] \Big\|_2\, .
\]
Denote $\bM = \id_n + \frac1d \bX \bX^\top$. We have
    \begin{equation*}
        \bW^\sT s ( \bPhi_N^\sT \hblambda_N )
             =
             \frac1d \bX^\top\bM^{-1} \bPhi_N s ( \bPhi_N^\sT \hblambda_N )
             +
             \Big(\bW^\top - \frac1d \bX^\top\bM^{-1} \bPhi_N \Big) s ( \bPhi_N^\sT \hblambda_N ) \, ,
    \end{equation*}
    and
    \begin{equation*}
    \begin{aligned}
       \E_{\bw, \phi} [ \bw s ( \< \bphi_n ( \bw) , \hblambda \>)] 
             =&~
             \E_{\bw, \phi} \Big[\frac1d \bX^\top\bM^{-1} \bphi_n ( \bw)  s ( \< \bphi_n ( \bw) , \hblambda \>)\Big] \\
             &~
             +
             \E_{\bw,\phi} \Big[\Big(\bw - \frac1d \bX^\top\bM^{-1} \bphi_n ( \bw) \Big) s ( \< \bphi_n ( \bw) , \hblambda \>) \Big].
             \end{aligned}
    \end{equation*}
    By the interpolation constraints \eqref{eq:interpol_app} and recalling the expression of the features matrix in Eq.~\eqref{eq:feature_matrix_latent}, we write
    \begin{equation*}
        \bW^\top s ( \bPhi_N^\sT \hblambda_N )
            = 
             \frac Nd \bX^\top\bM^{-1} \by 
            + 
            \Big( \id_d - \frac1d \bX^\top \bM^{-1}  \bX \Big) \bW^\top s ( \bPhi_N^\sT \hblambda_N ) - \frac1d \bX^\top\bM^{-1} \bZ s ( \bPhi_N^\sT \hblambda_N ) \, ,
    \end{equation*}
    and 
    \begin{equation*}
    \begin{aligned}
        \E_{\bw, \phi} [ \bw s ( \< \bphi_n ( \bw) , \hblambda \>)]
            = &~
            \frac1d \bX^\top\bM^{-1} \by 
            + 
            \Big( \id_d - \frac1d \bX^\top \bM^{-1}  \bX \Big) \E_{\bw, \phi} [ \bw s ( \< \bphi_n ( \bw) , \hblambda \>)] \\
            &~ - \frac1d \bX^\top\bM^{-1} \E_{\bw, \bz} [ \bz s ( \< \bphi_n ( \bw, \bz) , \hblambda \>)]  \, .
            \end{aligned}
    \end{equation*}
    The first terms on the right-hand sides of the preceding two expressions are the same (up to a factor $N$). We deduce that
    \begin{equation}\label{eq:decompo_interpol_trick}
    \begin{aligned}
        &~ \Big\| N^{-1} \bW^\top s ( \bPhi_N^\sT \hblambda_N ) -  \E_{\bw, \phi} [ \bw s ( \< \bphi_n ( \bw) , \hblambda \>)] \Big\|_2 \\
            \leq &~  
            \Big\|\id_d - \frac1d \bX^\top \bM^{-1} \bX \Big\|_{\op} \Big\| N^{-1} \bW^\top s ( \bPhi_N^\sT \hblambda_N ) -  \E_{\bw, \phi} [ \bw s ( \< \bphi_n ( \bw) , \hblambda \>)] \Big\|_2 
        \\
            &~
            +
            \Big\|\frac1d \bX^\top \bM^{-1}\Big\|_{\op} \Big\| N^{-1} \bZ s ( \bPhi_N^\sT \hblambda_N )  - \E_{\bw, \bz} [ \bz s ( \< \bphi_n ( \bw, \bz) , \hblambda \>)] \Big\|_2.
    \end{aligned}
    \end{equation}
    From the expression of $\bM$, we have $\id_d - \frac1d \bX^\top \bM^{-1} \bX  =  (\id_d + \frac{1}{d} \bX^\sT \bX)^{-1} $. By the assumption that $\sigma_{\min} ( \bX)  \geq \sqrt{n} c_0$ and $n \geq 2d / c_0^2$, we have $\| \id_d - \frac1d \bX^\top \bM^{-1} \bX  \|_{\op} \leq \frac{d}{n} c_{0}^{-2}  \leq 1/2$.
    Similarly, we have $\|\frac1d \bX^\top \bM^{-1}\|_{\op} \leq 2/ (3 c_0 \sqrt{n} )$. Rearranging the terms in Eq.~\eqref{eq:decompo_interpol_trick} implies Eq.~\eqref{eq:interpolation-trick}.
\end{proof}

\begin{proof}[Proof of Proposition \ref{prop:latent-RF-infinite}]
  By Proposition \ref{prop:latant_assumptions}, FEAT1, FEAT2, FEAT3 and FEAT3' are satisfied with probability $1 - e^{-c'n}$. Hence, we can use the results proved in the proof of Theorem \ref{thm:gen-conv-to-pop}. Replace the event $\cE_1$ by the event $\cE_1 '$ that for all $\|\blambda\|_{\bK_n}/\|\hblambda\|_{\bK_n} \in [1/2,2]$,
      \begin{equation*}
            \frac{\Big\|\frac1{N}\sum_{j=1}^N \bz_j s(\langle \bphi_{n,j} , \blambda \rangle) - \E_{\bw,\bz}[\bz s(\< \bphi_n (\bw, \bz), \blambda\>)]\Big\|_{2}}{s(\|\hblambda\|_{\bK_n})}
            \leq 
            \eps_1\, .
    \end{equation*}
    Using the same proof as in Lemma \ref{lem:point-wise-concentration}, where the vector $\bz$ is a $\kappa^2$-sub-Gaussian random vector by A3, there exists $C', c'$ such that for $N \geq n \geq c' d$, taking $\eps_1 = C' \Big( \sqrt{\frac{n\log N}{N}} \vee \frac{(n\log N)^{(Q/2)\vee 1}}{N} \Big)$, we have $\P (\cE_1') \geq 1 - C' N^{-c' n}$. Consider the same events $\cE_2$ (with same $\eps_2$) and $\cE_3$ as in the proof of Theorem \ref{thm:gen-conv-to-pop}, except with $\tau,\eta$---which do not depend on $d,n$---absorbed into the constnast $C$. By the same proof as in Lemma \ref{lem:gen-dual-continuity}, we can show that there exists a constant $K >0$ such that for all $\|\blambda - \hblambda\|_{\bK_n} \leq \|\hblambda\|_{\bK_n}/2 $,
    \begin{equation}\label{eq:avg-z-continuity}
        \frac{\|\E_{\bw, \bz} [\bz s(\< \bphi_n (\bw , \bz) , \blambda \>)] - \E_{\bw, \bz} [\bz s(\< \bphi_n (\bw , \bz) , \hblambda \>)]\|_2}{s(\|\hblambda\|_{\bK_n})}
            \leq 
            K\frac{\|\blambda - \hblambda\|_{\bK_n}}{\|\hblambda\|_{\bK_n}}.
    \end{equation}
    The same argument as in Lemma \ref{lem:gen-deterministic-analysis} implies that on events $\cE_1' , \cE_2, \cE_3$, we have
    \begin{equation}
        \| N^{-1} \bZ s(\bPhi_n^\top \hat \blambda_N) - \E_{\bw,\bz} [\bz s(\< \bphi_n ( \bw , \bz ), \hblambda \>)]\|_2
            \leq 
            \Big(
                \eps_1 + \frac{K\eps_2}{\beta}
            \Big) s(\|\hblambda\|_{\bK_n})\, .
    \end{equation}
    Recalling the argument in the proof of Theorem \ref{thm:gen-conv-to-pop}, we have $s(\|\hblambda\|_{\bK_n}) \leq C \| \bK_n^{-1/2} \by \|_2 \leq C' \| \by \|_2$, where we used that by A3, $\bK_n \succeq \gamma^2 \id_n$. Hence with probability at least $1 - C' N^{-c'n}$, we have
        \begin{equation}
        \| N^{-1} \bZ s(\bPhi_n^\top \hat \blambda_N) - \E_{\bw,\bz} [\bz s(\< \bphi_n ( \bw , \bz ), \hblambda \>)]\|_2
            \leq 
            C ' \Big( \sqrt{\frac{n\log N}{N}} \vee \frac{(n\log N)^{Q/2}}{N} \Big) \| \by \|_{2 } \, .
    \end{equation}
    Furthermore, from assumption A1 and Lemma \ref{lem:matrix_concentration}, there exists $c>0$ such that $\sigma_{\min} ( \bX) \geq c \sqrt{n} $ with probability at least $1 - e^{-cn}$. Hence, we can use Lemma \ref{lem:interpolation-trick}, which concludes the proof.
\end{proof}

\section{Small ball property under fast decay of Hermite coefficients}
\label{sec:small_ball}

In this section, we show FEAT3 for a deterministic feature map $\phi ( \bx ; \bw ) = \sigma ( \< \bx , \bw \>)$ with the activation function $\sigma : \R \to \R$ verifying a decay condition on its Hermite coefficients.  

Recall that for any function $g \in L^2 (\R , \gamma)$ with $\gamma (\de x ) = e^{-x^2/2} \de x / \sqrt{2 \pi}$ the standard Gaussian measure, we have the decomposition
\[
g(x) = \sum_{k = 0}^\infty \frac{\mu_k (g)}{k!} \He_k (x), \qquad \mu_k (g) \equiv \E_{G \sim \gamma} \left\{ g(G) \He_k (G) \right\} \, ,
\]
where $\{ \He_k \}_{k \geq 0}$ are the Hermite polynomials with standard normalization $\E \{ \He_k (G) \He_j (G) \} = k! \delta_{jk}$ and we call $\mu_k (g)$ the $k$-th Hermite coefficient of $g$.

Throughout this section, we will consider $\bw \sim \normal ( 0 , \id_d /d )$ and $\| \bx_i \|_2 = \sqrt{d}$ for all $i\leq n$. For an integer $m \geq 0$, consider the decomposition of the activation $\sigma = \sigma_{\leq m} + \sigma_{>m}$ into a low-degree and a non-polynomial parts
\begin{align}
\sigma_{\leq m} ( x) = \sum_{k = 0}^m \frac{\mu_k (\sigma)}{k!} \He_k (x), \qquad \sigma_{> m} ( x) = \sum_{k = m+1}^\infty \frac{\mu_k (\sigma)}{k!} \He_k (x)\, .
\end{align}
For $\bw \sim \normal ( 0 , \id_d /d )$ and taking $\| \bx \|_2 = \| \bx ' \|_2 = \sqrt{d}$, we can decompose the kernel function into
\begin{align}
K ( \bx , \bx ' ) = \E_{\bw} \{ \sigma ( \< \bw , \bx \> ) \sigma ( \< \bw , \bx ' \>) \} = K^{\leq m} ( \bx , \bx ' ) +  K^{> m} ( \bx , \bx ' ) \, ,
\end{align}
where
\begin{align}
K^{\leq m} ( \bx , \bx ' ) =&~ \E_{\bw} \{ \sigma_{\leq m} ( \< \bw , \bx \> ) \sigma_{\leq m} ( \< \bw , \bx ' \>) \} \, ,\\
K^{> m} ( \bx , \bx ' ) =&~ \E_{\bw} \{ \sigma_{> m} ( \< \bw , \bx \> ) \sigma_{> m} ( \< \bw , \bx ' \>) \} \, .
\end{align}
Therefore the empirical kernel matrix can be decomposed into $\bK_n = \bK_n^{\leq m} + \bK_n^{>m}$ where $\bK_n^{\leq m} = (K^{\leq m} (\bx_i , \bx_j))_{ij \in [n]}$ and $\bK_n^{> m} = (K^{> m} (\bx_i , \bx_j))_{ij \in [n]}$. 

Similarly, we introduce
\begin{align}
\bphi_n^{\leq m} (\bw) =&~ \big( \sigma_{\leq m} ( \< \bx_1 , \bw\>) , \ldots , \sigma_{\leq m} ( \< \bx_n , \bw\>) \big)\, , \\
 \bphi_n^{> m} (\bw) = &~\big( \sigma_{> m} ( \< \bx_1 , \bw\>) , \ldots , \sigma_{> m} ( \< \bx_n , \bw\>) \big) \, .
\end{align}
Notice that $\E_\bw \big\{ \bphi_n^{\leq m} (\bw) \bphi_n^{\leq m} (\bw)^\sT \big\} = \bK_n^{\leq m}$ and  $\E_\bw \big\{ \bphi_n^{> m} (\bw) \bphi_n^{> m} (\bw)^\sT \big\} = \bK_n^{> m}$. Denote $\bpsi_n^{\leq m} = \bK_n^{-1/2} \bphi_n^{\leq m}$ and $\bpsi_n^{> m} = \bK_n^{-1/2} \bphi_n^{> m}$. 

With these notations, we can now introduce our result.

\begin{proposition}\label{prop:small_ball_Hermite_decay}
Assume that $\| \bx_i \|_2 = \sqrt{d}$ for all $i \leq n$ and $\bw \sim \normal ( 0 , \id_d /d )$, and that $\sigma \in L^2 (\R , \gamma)$. There exists an absolute constant $C_0> 0$ such that the following holds. If for an integer $m \geq 1$,
\begin{align}\label{eq:condition_hermite}
\lambda_{\max} \left( \bK_n^{>m} \right) \leq \lambda_{\min} \left( \bK_n \right) \cdot \Big\{ (C_0m)^{-(2m+1)} \wedge (1/4) \Big\} \, ,
\end{align}
then we have
\begin{align}
\sup_{\| \bv \|_2 = 1} \P \big( | \< \bv , \bpsi_n \> | \leq \eta^* \big) \leq \frac{1}{4} \, ,
\end{align}
where $\eta^*  = 4\lambda_{\max } \big(\bK_n^{>m}\big)  / \lambda_{\min} \big( \bK_n \big)$.
\end{proposition}

\begin{proof}[Proof of Proposition \ref{prop:small_ball_Hermite_decay}]
Throughout this proof, we will denote $C>0$ a generic absolute constant. In particular, $C$ is allowed to change from line to line.

Consider $\eta >0$, $\| \bv \|_2 =1$, and an integer $m$ such that condition \eqref{eq:condition_hermite} is satisfied with a constant $C_0$ that will be fixed later, and decompose
\begin{equation}\label{eq:decompo_proba}
\P \big( | \< \bv , \bpsi_n \> | \leq \eta \big) \leq \P \big( | \< \bv , \bpsi_n^{\leq m} \> | \leq 2\eta \big) + \P \big( | \< \bv , \bpsi_n^{>m} \> | \geq \eta \big)\, .
\end{equation}
The first term $\< \bv, \bpsi_n^{\leq m} (\bw ) \>$ is a polynomial of degree $m$ in $\bw \sim \normal (0 , \id_d / d)$. From Carbery-Wright inequality \cite{carbery2001distributional}, we have the following anti-concentration bound
\[
\P \big( | \< \bv , \bpsi_n^{\leq m} \> | \leq 2\eta \big) \leq C m \left( \frac{2\eta}{ \E \{ \< \bv , \bpsi_n^{\leq m} \>^2 \} } \right)^{1/m} \, .
\]
Note that
\[
\begin{aligned}
 \E \big\{ \< \bv , \bpsi_n^{\leq m} \>^2 \big\} = &~ \< \bv , \bK_n^{-1/2} \bK_n^{\leq m } \bK_n^{-1/2} \bv \> \\
 =&~ 1 -  \< \bv , \bK_n^{-1/2} \bK_n^{> m } \bK_n^{-1/2} \bv \> \geq 1 - \frac{\lambda_{\max} \big(\bK_n^{>m}\big) }{\lambda_{\min} \big(\bK_n\big)} \geq 3/4 \, ,
 \end{aligned}
\]
where in the last inequality, we used $\lambda_{\max} \big(\bK_n^{>m}\big) \leq  \lambda_{\min} \big(\bK_n\big)/4$ from condition \eqref{eq:condition_hermite}. Hence
\begin{equation}\label{eq:hermite_bound_1}
\P \big( | \< \bv , \bpsi_n^{\leq m} \> | \leq 2\eta \big)  \leq C m \cdot  \eta^{1/m} \, .
\end{equation}

By Markov's inequality, the second term in Eq.~\eqref{eq:decompo_proba} is bounded by
\begin{equation}\label{eq:hermite_bound_2}
\P \big( | \< \bv , \bpsi_n^{>m} \> | \geq \eta \big) \leq \eta^{-2} \E \big\{ \< \bv , \bpsi_n^{>m} \>^2 \big\} \leq \eta^{-2} \cdot \frac{\lambda_{\max} \big(\bK_n^{>m}\big) }{\lambda_{\min} \big(\bK_n\big)}\, .
\end{equation}
Hence combining bounds \eqref{eq:hermite_bound_1} and \eqref{eq:hermite_bound_2} in Eq.~\eqref{eq:decompo_proba} yields
\[
\P \big( | \< \bv , \bpsi_n \> | \leq \eta \big) \leq  C m \cdot \eta^{1/m}  + \eta^{-2} \cdot \frac{\lambda_{\max} \big(\bK_n^{>m}\big) }{\lambda_{\min} \big(\bK_n\big)}\, .
\]
Setting $\eta^* = \left( \lambda_{\max} (\bK_n^{>m})  / (C m \lambda_{\min} (\bK_n) ) \right)^{m/(2m+1)}$, we get
\[
\P \big( | \< \bv , \bpsi_n \> | \leq \eta^* \big) \leq  C m \left( \frac{\lambda_{\max} \big(\bK_n^{>m}\big) }{\lambda_{\min} \big(\bK_n\big)} \right)^{1/(2m+1)} \leq \frac{1}{4} \, ,
\]
where we set up $C_0 =  4C$ to obtain the last inequality. Noticing that condition \eqref{eq:condition_hermite} implies
\[
\begin{aligned}
\eta^* =&~ \left( \lambda_{\max} (\bK_n^{>m})  / (C m \lambda_{\min} (\bK_n) ) \right)^{m/(2m+1)} \\
\geq&~ \left( 4 \lambda_{\max} (\bK_n^{>m})  /  \lambda_{\min} (\bK_n)  \right)^{m(2m+2)/(2m+1)^2}  \geq 4 \lambda_{\max} (\bK_n^{>m})  /  \lambda_{\min} (\bK_n)\, ,
\end{aligned}
\]
concludes the proof.
\end{proof}

Note that for $\bx_i\sim \Unif (\S^{d-1} (\sqrt{d}))$, we have $\< \bx_i , \bx_j \> = O_{d,\P}(d^{-1/2})$ and therefore we expect the off-diagonal elements of $\bK_n^{>m}$ to have negligible operator norm for $m$ sufficiently large. In fact, it was shown in \cite{ghorbani2019linearized} (see \cite{mei2021generalization} for a generalization), that for any constant $\delta >0$ and integer $\ell$, if $d^{\ell+\delta} \leq n \leq d^{\ell +1 -\delta}$, then for any $m \geq \ell$, we have $\| \bK_n^{>m} - \kappa_{>m} \id_n \|_{\op} = \kappa_{>m} \cdot o_{d,\P} (1)$, where
\[
\kappa_{>m} = \sum_{k = m+1}^\infty \frac{\mu_k (\sigma)^2}{k!}\, .
\]
Hence for $d^{\ell+\delta} \leq n \leq d^{\ell +1 -\delta}$ and $d$ sufficiently large, condition \eqref{eq:condition_hermite} is verified with high probability over the data $\bX$, if there exists $m > \ell$ such that 
\[
\kappa_{>m}  \leq \kappa_{>\ell}  \cdot \Big\{ (C_0m)^{-(2m+1)} \wedge (1/4) \Big\} \, .
\]

 \section{Proof of Theorem \ref{thm:F1_NPhard}: $\NP$-hardness of learning with $\cF_1$ norm}
\label{sec:proof_hardness}

We borrow some notations and terminology from \cite{grotschel2012geometric}. We consider the convex set $K \subset \R^n$ define by
\begin{equation}
\begin{aligned}
K =&~ \Big\{ \bz \in \R^n \, : \, \exists \tau \in \cM(\cV), | \tau | (\cV)\leq 1 \text{ and }\bz = \int_{\cV} \bphi_n (\bw) \tau(\de \bw)  \Big\} \\
= &~ \mathsf{ConvexHull} \Big\{ \bphi_n (\bw) , - \bphi_n (\bw) : \bw \in \cV \Big\} \, .
\end{aligned}
\end{equation}

From our choice of the truncated ReLu activation, we have $\| \bphi_n (\bw) \|_2 \leq \sqrt{n}$ and 
$K \subseteq \Bsf (\bzero , \sqrt{n})$, where we denoted $\Bsf ( \ba , R )$ the ball of 
center $\ba$ and radius $R$, i.e., $\Bsf ( \ba , R ) = \{ \bz \in \R^n : \| \bz - \ba \|_2 \leq R\}$. 
In our reductions, we will further meed to assume that there exists $r >0$ such that 
$\Bsf ( \bzero , r) \subseteq K$. We will check that indeed we can choose $r >0$ during our
 reduction. We will denote $(K,r)$ the set $K$ such that $\Bsf ( \bzero , r) \subseteq K $.
  For $\delta >0$, we let $S(K,\delta)$ denote the $\delta$-neighborhood of $K$, i.e., 
\[
S ( K , \delta) := \big\{ \bz \in \R^n  : \| \bu - \bz \|_2 \leq \delta \text{ for some } \bu \in K \big\} \, .
\]
Similarly, we will denote $S(K,-\delta)$ the interior $\delta$-ball of $K$ defined by
\[
S ( K , - \delta) := \big\{ \bz \in K  : \Bsf (\bz , \delta) \subseteq K \big\} \, .
\]

For convenience, we recall here the different problems of interest. The weak version of the $\cF_1$-problem is given by
\begin{equation}\label{eq:F1_problem_weak_app}%
  \begin{aligned}
    \mbox{\rm minimize}\;\;\;& | \tau | (\cV) \, ,\\
     \mbox{\rm subj. to}\;\;\;& \hf(\bx_i;\tau)=\hy_i,\;\;\; \forall i\le n\, , \\
     & \| \by - \hat \by \|_2 \leq \eps \, ,
    \end{aligned}
\end{equation}
while the intermediary optimization problem reads with the new notations
\begin{equation}\label{eq:inter_problem_app}%
  \begin{aligned}
    \mbox{\rm maximize}\;\;\;& \< \by , \bz \> \, ,\\
     \mbox{\rm subj. to}\;\;\;& \bz \in K \, .
    \end{aligned}
\end{equation}
 
 We will consider the following problems:
 \begin{description}

\item[$\boxed{\texttt{W-F1-PB}(\eps)}$]: given $\by \in \Q^n$ and $\gamma \in \Q$. Denote 
$L^*$ the value of the weak $\cF_1$-problem \eqref{eq:F1_problem_weak_app}. Either
\begin{itemize}
\item[(1)] assert that $L^* \leq \gamma + \eps$; or,
\item[(2)] assert that $L^* \geq \gamma - \eps$. 
\end{itemize}

\item[$\boxed{\texttt{W-MEM}(\delta,K,r)}$]: given $\blambda \in \Q^n$. Either
\begin{itemize}
\item[(1)] assert that $\blambda \in S(K,\delta)$; or,
\item[(2)] assert that $\blambda \not\in S(K,-\delta)$. 
\end{itemize}

\item[$\boxed{\texttt{W-VAL}(\delta,K,r)}$]: given $\bz \in \Q^n, \gamma \in \Q$. Either 
\begin{itemize}
\item[(1)] assert that $\< \bz , \blambda \> \leq \gamma + \delta$ for all $\blambda \in S(K,-\delta)$; or,
\item[(2)] assert that $\< \bz , \blambda \> \geq \gamma - \delta$ for some $\blambda \in S(K,\delta)$. 
\end{itemize}

\item[$\boxed{\texttt{HS-MA}(\eps)}$]: distinguish the following two cases
\begin{itemize}
\item[(1)] there exists a half space $\< \bw , \bx \>   > a$ such that $M( \bw , a)  \geq (n_+ + n_- ) (1 - \eps)$; or,
\item[(1)] for any half space $\< \bw , \bx \>   > a$, we have $M( \bw , a)  \leq (n_+ + n_- ) (1/2 + \eps)$. 
\end{itemize}

\end{description}

\texttt{W-F1-PB} corresponds to a \textit{weak validity} problem associated to the weak $\cF_1$-problem \eqref{eq:F1_problem_weak_app}; \texttt{W-MEM} is the \textit{weak membership} problem associated to the convex set $K$; \texttt{W-VAL} is the \textit{weak validity} problem associated to the intermediary optimization problem \eqref{eq:inter_problem_app}; and \texttt{HS-MA} is the \textit{Maximum Agreement for Halfspaces} problem.

We will use the following hardness result on \texttt{HS-MA}: 

\begin{theorem}[Theorem 8.1 in \cite{guruswami2009hardness}] 
For all $0 < \eps < 1/4$, the problem \texttt{HS-MA}$(\eps)$ is $\NP$-hard. \label{thm:HSMA_NPhard} 
\end{theorem}

Let us first prove that there exists a polynomial time randomized reduction from \texttt{W-VAL} to \texttt{W-F1-PB}.

\begin{lemma}\label{lem:F1-VAL}
There exists an absolute constant $C>0$ and an oracle-polynomial time randomized algorithm that solves the weak validity problem \texttt{W-VAL}$(\delta,K,r)$ given a \texttt{W-F1-PB}$(\eps)$ oracle, where $\eps \geq \Omega ( (\delta r /n)^C )$.
\end{lemma}

\begin{proof}[Proof of Lemma \ref{lem:F1-VAL}]
Let us first show that one can use \texttt{W-F1-PB}$(\eps)$ to solve $\texttt{W-MEM}(\delta,K,r)$.
 
Consider $\blambda \in \Q^n$ and call \texttt{W-F1-PB} with $\by := \blambda$, $\gamma := 1$ 
and $\eps := \frac{\delta}{1 + 2 \delta + 4\sqrt{n}}$. First, we know that 
$K \subseteq \Bsf( \bzero , \sqrt{n})$, hence if $\| \blambda \|_2 \geq  2\sqrt{n}$, we can 
directly assert that $\blambda \not\in S(K,-\delta)$. We assume from now on that 
$\| \blambda \|_2 \leq  2\sqrt{n}$. If \texttt{W-F1-PB} asserts that $L^* \leq 1 +\eps$, 
then it means there exists $\bz \in \R^n$ with $\| \bz - \blambda \|_2 \leq \eps$ and 
associated measure $| \tau | (\cV) \leq 1 + \eps$. Consider $\bz' = \bz / (1 +\eps)$, 
then $\bz'$ has associated measure $| \tau ' | (\cV ) \leq 1$ and $\| \bz ' - \blambda\|_2 \leq 
\frac{\eps}{1+\eps}(\| \blambda \|_2 +\eps) + \eps \leq \delta$ with our choice of $\eps$. 
Hence we can assert that $\blambda \in S(K,\delta)$.
If \texttt{W-F1-PB} asserts that $L^* \geq 1 - \eps$, then it means in particular that 
$\blambda$ has associated measure $| \tau | (\cV) \geq 1 - \eps$. Consider $\bz = \blambda / (1 - 2\eps)$, then $\bz$ has associated measure $| \tau ' | (\cV ) \leq \frac{1-\eps}{1-2\eps}>1$ and $\| \bz ' - \blambda\|_2 \leq \frac{2\eps}{1-2\eps}\| \blambda \|_2  \leq \delta$ with our choice of $\eps$. Hence we can assert that $\blambda \not\in S(K,-\delta)$.

Now that we saw we can implement a weak membership oracle $\texttt{W-MEM}(\delta,K,r)$ using \texttt{W-F1-PB}$(\eps)$ with $\eps = \Omega (\delta / \sqrt{n})$, we can use the results in \cite{lee2018efficient}, for example their Theorem 21 (using the sequence of reductions $\texttt{MEM}$ to $\texttt{SEP}$ to $\texttt{OPT}$ to $\texttt{VAL}$), which shows that there exists an absolute constant $C>0$ and a randomized reduction from the weak validity problem \texttt{W-VAL}$(\delta,K,r)$ to the weak membership problem $\texttt{W-MEM}(\delta ',K,r)$ with $\delta ' = \Omega ( (\delta r / n)^C )$. 
\end{proof}

\begin{lemma}\label{lem:VAL-HSMA}
There exists an oracle-polynomial time algorithm that solves the problem \texttt{HS-MA}$(\eps)$ given a weak validity \texttt{W-VAL}$(\delta,K,r)$ oracle, where $\delta = \Omega ((1-\eps)^2/\sqrt{n})$ and $\eta = \Omega ((1-\eps)/\sqrt{n})$ .
\end{lemma}

\begin{proof}[Proof of Lemma \ref{lem:VAL-HSMA}]
First let us show that 
\begin{equation}\label{eq:equa_opt}
\sup_{\blambda \in K} \< \ones , \blambda \> = \sup_{\bw \in \cV} \< \ones , \bphi_n (\bw) \>\, .
\end{equation}
Notice that with our truncated ReLu activation, $\bphi_n(\bw)$ has all its coordinates non-negative for any $\bw \in \cV$. Hence,
\[
\sup_{\substack{\tau = \tau_+ - \tau_-, \\ \tau_+ (\cV) + \tau_- (\cV) \leq 1}} \int_{\cV} \< \ones , \bphi_n(\bw) \> \tau_+ (\de \bw) -  \int_{\cV} \< \ones , \bphi_n(\bw) \> \tau_-  (\de \bw) = \sup_{\substack{\tau = \tau_+,\\ \tau_+ (\cV)  \leq 1}} \int_{\cV} \< \ones , \bphi_n(\bw) \> \tau_+ (\de \bw)\, ,
\]
where we denoted $\tau_+$ and $\tau_-$ non-negative measures on $\cV$. Hence, we directly have $\sup_{\blambda \in K} \< \ones , \blambda \> \leq \sup_{\bw \in \cV} \< \ones , \bphi_n (\bw) \>$ and the converse inequality is obtained by taking the supremum over $\tau = \delta_{\bw}$ with $ \bw \in \cV$. 

Let us now prove the reduction from \texttt{HS-MA} to \texttt{W-VAL}.
Consider $(n_+,n_-,d) \in \N^2$ and vectors $\{ \bx_1 , \ldots , \bx_{n_+} , \bz_1 , \ldots , \bz_{n_-} \} \subset \{ - 1 , 1 \}^d$. Denote $n = n_+ + n_-$, and $\Tilde \bx_i = (\bx_i, - 1)$ for $i \in [n_+]$ and $\Tilde \bx_{n_+ +i} = (- \bz_i,  1)$ for $i \in [n_-]$. Writing now $(\bw,a)$ as $\bw \in \R^{d+1}$, we see that 
\[
\Tilde M(\bw) = \sum_{ i \leq n} \mathbbm{1}[ \< \bw , \Tilde \bx_i \> > 0]\, .
\]
Denote now $\bphi_n (\bw) = ( \bphi (\Tilde \bx_1 ; \bw) , \ldots , \bphi (\Tilde \bx_n ; \bw) )$. 
Recall that $\bphi (\bx ; \bw) = \min ( \max ( \< \bw , \bx \> , 0 ) ,1)$: if we take 
$\cV = \R^{d+1}$, we can always rescale $\bw$ by a constant large enough such that 
$\bphi_n (\bw)$ only as $0$ or $1$ values, and
\[
\sup_{\bw \in \R^{d+1} } \< \ones , \bphi_n (\bw) \> = \sup_{\bw \in \R^{d+1} } \Tilde M(\bw) \, .
\]
Using Eq.~\eqref{eq:equa_opt}, we can find easily a reduction from \texttt{HS-MA} to the strong version of \texttt{W-VAL} (with $\delta = 0$). However, we will do a slightly more complicated reduction, in order to insure that there exists $r>0$ such that $\Bsf (\bzero , r ) \subseteq K$ and we can take $\delta >0$.

Consider $\ubx_i = (\Tilde \bx_i, \be_i ) \in \R^{d+1+n}$ where $\be_i$ is the canonical vector basis in $\R^n$ (vector with $1$ at the $i$'th coordinate and $0$ otherwise). Consider $\cV = \R^{d+1} \times [-\eta , \eta]^{n}$. In this case, by taking $\bw = \eta (\bzero, \be_i)$, we have $\bphi_n(\bw) = \eta \be_i$. Hence $K$ contain all $\pm \be_i$ and therefore contain $\Bsf ( \bzero, \eta/\sqrt{n})$. Denote $\underline{M}$ the $M$ function associated to the $\ubx_i$. By $1$-Lipschitzness of truncated ReLu, we have
\begin{equation}\label{eq:lip_ineq_M}
\sup_{\bw \in \R^{d+1} } \Tilde M(\bw) \leq \sup_{\bw \in \R^{d+1+n} } \underline{M} (\bw)  \leq \sup_{\bw \in \R^{d+1} } \Tilde M(\bw) + n \eta \, .
\end{equation}
Let us call \texttt{W-VAL}$(\delta,K,\eta/\sqrt{n})$ with $\gamma := \frac{3}{4}n$, $\bz := \ones$ and $\delta$ and $\eta$ that will fixed sufficiently small later. Consider the case $\sup_{\bw \in \R^{d+1}} \Tilde M (\bw) \geq n (1 -\eps)$, i.e., there exists an halfspace that makes less than $n\eps$ errors. In that case, we show that (for $\delta, \eta >0$ sufficiently small), the oracle must assert that $\< \ones , \blambda \> \geq \frac{3}{4}n - \delta$ for some $\blambda \in S(K,\delta)$. Let us construct $\blambda\in S(K , - \delta)$ such that $\< \bz , \blambda \> > \gamma + \delta$. Denote $\blambda^*$ such that $\< \ones , \blambda^* \> \geq n(1-\eps)$, and consider $\lambda_s = (1 - s) \blambda^*$. Because $K$ is convex and contains $\blambda^*$ and $\Bsf (\bzero, \eta / \sqrt{n})$, it must contain the cone with apex $\blambda^*$ and base the section of  $\Bsf (\bzero, \eta / \sqrt{n})$ perpendicular to $\blambda_*$. In particular, $\Bsf( \blambda_s , \eta s/\sqrt{n}) \subseteq K$ (for $n$ sufficiently large).  
Furthermore, $\< \blambda_s , \ones \> = (1 - s) \< \blambda^* , \ones \> \geq (1 - s)n(1-\eps)$. Setting $s = \frac{1/4 - \eps}{2(1- \eps)}$ and $\delta \leq \eta s / \sqrt{n}$, we have $\blambda_s \in S(K, - \delta)$ such that $\< \blambda_s , \ones \> \geq (3/4+ \kappa)n > 3n/4 + \delta$ for some $\kappa >0$.

Consider now the case $\sup_{\bw \in \R^{d+1}} \Tilde M (\bw) \leq n (1/2 +\eps)$, i.e., halfspaces classify at most $n (1/2 +\eps)$ vectors correctly. In that case, we show that (for $\delta, \eta >0$ sufficiently small), the oracle must assert that $\< \ones , \blambda \> \leq \frac{3}{4}n + \delta$ for all $\blambda \in S(K,-\delta)$. To do so, we show that for any $\blambda \in S(K, \delta)$, we must have $\< \ones , \blambda \> < 3n/4 - \delta$. Consider an arbitrary $\blambda \in S(K, \delta)$, there exists $\bz \in K$ such that $\| \blambda - \bz \|_2 \leq \delta$. Using Eqs.~\eqref{eq:equa_opt} and \eqref{eq:lip_ineq_M}, we have $\< \ones , \bz \> \leq n(1/2 +\eps +\eta)$. Hence $\< \ones , \blambda \> \leq n(1/2 +\delta +\eta) + \delta \sqrt{n}$. Taking $\eta = \frac{1 - 4\eps}{8}>0$ and $\delta \leq C$, we have $\< \ones , \blambda \> <  3n/4 - \delta$.

Combining the above conditions, we see that there exists an absolute constant $c>0$ such that the \texttt{W-VAL}$(\delta,K,\eta/\sqrt{n})$ oracle with $\delta = c (1 - 4\eps)^2 / \sqrt{n}$ and $\eta = c (1 - 4\eps)$ allows to distinguish \texttt{HS-MA}$(\eps)$.
\end{proof}

We are now ready to prove Theorem \ref{thm:F1_NPhard}.

\begin{proof}[Proof of Theorem \ref{thm:F1_NPhard}] 
From Lemma \ref{lem:VAL-HSMA}, there is a polynomial-time reduction from \texttt{HS-MA}$(1/10)$ to \texttt{W-VAL}$(c/\sqrt{n},K,c/\sqrt{n})$ for some absolute constant $c>0$. From Lemma \ref{lem:F1-VAL}, there is a polynomial-time randomized reduction from \texttt{W-VAL}$(c/\sqrt{n},K,c/\sqrt{n})$ to \texttt{W-F1-PB}$(n^{-C})$ for some absolute constant $C>0$. Using Theorem \ref{thm:HSMA_NPhard} concludes the proof. 
\end{proof}

\section{Additional properties of the minimum interpolation problem}

\subsection{The infinite-width primal problem for randomized features}
\label{sec:Primal}

In the case of randomized features, we wrote the infinite-width dual problem   (Eq.~\eqref{eq:dual_infinite}), 
but  we did not write the infinite-width primal problem to which it corresponds.
Indeed the dual problem entirely defines the predictor via Eq.~\eqref{eq:InfWidthPredictorDual}.
 For the sake of completeness, we present the infinite-width primal problem here.

 It is convenient to make the randomization mechanism explicit by drawing $z \sim \nu$ (with $(\cZ,\nu)$ a probability space)
independent of $\bw$,
    and writing $\phi(\bx; \bw) = \phi(\bx;\bw, z)$ (without loss of generality, the reader can assume $z\sim \nu=\Unif([0,1])$).
Therefore, we can rewrite Eq.~\eqref{eq:dual-infinite} as
\begin{align*}
F ( \blambda) &=~\< \blambda , \by \> - \int_{\cV\times\cZ} \rho^* ( \< \bphi_{n}(\bw,z) ,  \blambda \> )  \mu( \de \bw)\nu(\de z)\\
               & =~\< \blambda , \by \>-\sup_a \int_{\cV\times\cZ} \big[ a(\bw,z)  \< \bphi_{n}(\bw,z) ,  \blambda \> -\rho ( a(\bw,z) )\big]  \mu( \de \bw)\nu(\de z)\\
  & =~\inf_a \cL(\blambda,a)\, ,
\end{align*}
where $a:\cV\times\cZ\to\reals$ is a general measurable function.
This expression shows immediately that the primal variable is a function of both $\bw$ and $z$.
Further, maximizing the Lagrangian  $\cL(\blambda,a)$ over $\blambda$, we obtain the following primal problem that
generalizes \eqref{eq:MinCplx} to the case of randomized features:
  \begin{align}\label{eq:MinCplx2}
    \mbox{\rm minimize}\;\;\;& \int_{\cV\times\cZ}  \rho(a(\bw,z)) \mu (\de \bw)\nu(\de z)\, ,\\
     \mbox{\rm subj. to}\;\;\;& \int_{\cV \times \cZ} a(\bw,z) \phi (\bx_i;\bw, z) \mu(\de \bw) \nu(\de z)  =y_i,\;\;\; \forall i\le n\, . 
    \end{align}
    This is a natural limit of the finite-width primal problem  \eqref{eq:opt_finite}, whereby we replace the weights
    $a_i$ by evaluations of the function $a$, $a_i=a(\bw_i,z_i)$.

\begin{remark}
  Note that, under assumptions FEAT1, FEAT3 (or FEAT3' for $Q_1\wedge Q_2<2$) and PEN, the minimizer
  to problem \eqref{eq:MinCplx}  (and its generalization \eqref{eq:MinCplx2})
  exists and is unique. First of all notice that problem \eqref{eq:MinCplx2} is feasible.
  Indeed, choosing $a(\bw,z) = \<\bphi_n(\bw,z),\bxi\>$ for $\bxi\in\reals^n$, the interpolation constraint takes the form
  \begin{align}
    \bK_n\bxi = \by\, ,
  \end{align}
  which has solutions since $\bK_n = \E_{\bw,z}\{\bphi_n(\bw,z)\bphi_n(\bw,z)^{\sT}\}\succ \bzero$
  is strictly positive definite by condition FEAT3. Let $a_0$ be such a feasible point.

  Denote by $U(a):= \int_{\cV\times\cZ} \rho(a(\bw,z))\,\mu(\de \bw)\nu(\de z)$ the cost function.
  Notice that, by assumption PEN, $\rho(x)\ge c|x|^{1+\delta}-C$ for some constants $c,\delta>0$, and
  $C\in\reals$, and therefore $U(a)\le U(a_0)$ only if $\|a\|_{L^{1+\delta}(\mu\otimes\nu)}\le M<\infty$.
  It is therefore sufficient to focus on these functions $a$. Further notice that the map
  $a\mapsto f(\bx_i;a)$ is continuous under weak convergence  in $L^{1+\delta}(\mu\otimes\nu)$,
  since $\phi(\bx_i;\cdot)\in L^{k}(\mu\otimes\nu)$ for
  every $k$ by assumption FEAT1. It follows that the set of feasible solutions of \eqref{eq:MinCplx2} satisfying
  $\|a\|_{L^{1+\delta}(\mu\otimes\nu)}\le C$  is closed and bounded and hence weakly sequentially  compact
  by Banach-Alaoglu.
  Finally,  $a\mapsto U(a)$
  is weakly lower semicontinuous by Fatou's lemma, and this implies existence of minimizers.
 
  Uniqueness follows from the fact that $\rho$ is strictly convex,
  which implies that $a\mapsto U(a)$ is also strictly convex.
  \end{remark}

\subsection{Representer theorem for strictly convex and differentiable penalty}
\label{sec:Representer}

  In the case of deterministic features (i.e.\ $a(\bw,z) = a(\bw)$ and $\phi(\bx;\bw,z)=\phi(\bx;\bw)$), 
  we present here a generalization of the representer theorem to a broad class of penalties $\rho$.
  Recall that in the case of $\rho (x) = \frac{1}{2} |x|^2$, 
  the representer theorem states that the solution $a_* :\R^d \to \R$ belongs to the class of functions 
  \begin{equation}\label{eq:subspace_representer_RKHS}
  \Big\{ \bw \mapsto \< \blambda , \bphi_n ( \bw ) \>  : \blambda \in \R^n \Big\} \, .
  \end{equation}
  In words, the solution $a_*$ is contained in the (at most) $n$-dimensional linear subspace spanned by $\{ \bw \mapsto \phi ( \bx_i ; \bw ) : i \leq n \}$. 

  The following proposition generalizes this result to a penalty $\rho$ that is strictly convex.

  \begin{proposition}[Representer theorem for penalty $\rho$]\label{prop:representer_rho} 
    Let $\rho :\R \to \R$ be strictly convex and differentiable, 
    and consider the minimum complexity interpolation problem \eqref{eq:MinCplx2} in the case of deterministic features. 
    The solution $a_* : \R^d \to \R$ of Problem \eqref{eq:MinCplx2} belongs to the class of functions (in the almost sure sense)
    \begin{equation}\label{eq:subspace_representer_rho}
    \Big\{ \bw \mapsto s(\< \blambda , \bphi_n ( \bw ) \>)  : \blambda \in \R^n \Big\} \, ,
    \end{equation}
    where $s(x) = (\rho ' )^{-1} (x)$.
  \end{proposition}
Note that we recover the representer theorem for RKHS when $\rho = \frac{1}{2} |x|^2$ and $s(x) = x$. 
However for general $\rho$, the loss function cannot be simplified by evaluating $(K(\bx_i,\bx_j))_{i,j \leq n}$ once.

\begin{proof}[Proof of Proposition \ref{prop:representer_rho}]
We have
\[
\begin{aligned}
&~ \inf_{a:\R^d \to \R} \Big\{  \int_{\cV}  \rho (a(\bw)) \mu (\de \bw) : \int_{\cV} a(\bw) \phi (\bx_i;\bw) \mu(\de \bw)   =y_i,\;\;\; \forall i\le n \Big\} \\
\stackrel{(1)}{=} &~\inf_{a:\R^d \to \R} \sup_{\blambda \in \R^n} \Big\{  \< \blambda , \by \> + \int_{\cV}  \rho (a(\bw)) \mu (\de \bw) - \int_{\cV} a(\bw) \< \blambda , \bphi_n (\bw ) \> \mu (\de \bw ) \Big\} \\
\stackrel{(2)}{=} &~  \sup_{\blambda \in \R^n} \Big[ \< \blambda , \hat \by \> + \inf_{a:\R^d \to \R} \int_{\cV} \big\{ \rho (a(\bw)) - a(\bw) \< \blambda , \bphi_n (\bw ) \> \big\} \mu (\de \bw) \Big]  \\
\stackrel{(3)}{=} &~ \sup_{\blambda \in \R^n} \Big[ \< \blambda , \by \> -  \int_{\cV}  \rho^* (\< \blambda , \bphi_n (\bw ) \> ) \mu (\de \bw) \Big] \, , \\
\end{aligned}
\]
where we introduced the Lagrange multiplies $\blambda$ in (1), we used strong duality in (2), and we used the definition of the convex conjugate in (3). At the optimum, we must have $a_* (\bw) = (\rho^*)' ( \< \blambda_* , \bphi_n (\bw ) \>  ) = (\rho ' )^{-1} ( \< \blambda_* , \bphi_n (\bw ) \>  ) = s (( \< \blambda_* , \bphi_n (\bw ) \>  ))$ a.s. with respect to $\mu$.
\end{proof}

\section{Useful technical facts}
\label{sec:Tools}

We recall the following basic result on concentration of the empirical covariance of
independent sub-Gaussian random vectors.
This can be found, for instance, in  \cite[Exercise
4.7.3]{vershynin2018high} or \cite{vershynin2010introduction} (Theorem 5.39 and Remark 5.40(1)).
\begin{lemma}\label{lem:matrix_concentration}
 Let $(\ba_i)_{i\le N}$ be independent $\tau^2$-subgaussian random
  vectors, with common covariance $\E\{ \ba_i\ba_i^{\sT}\} =\bSigma$.
 Denote by $\hbSigma:=N^{-1}\sum_{i=1}^N\ba_i\ba_i^{\sT}$
  the empirical covariance. Then there exist absolute constants $C,c>0$ such that, for all $s\ge C(\sqrt{n/N}\vee (n/N))$, we have
  \begin{align}
    \P\big(\big\|\hbSigma-\bSigma\big\|_{\op}\ge s\tau^2\big)\le \exp\big(-cN (s\wedge s^2)\big)\, .
  \end{align}
\end{lemma}

We also state and prove two simple lemmas about moments of sub-Gaussian random variables.
\begin{lemma}\label{lemma:SimpleSubg}
   For any $a>0$
and any $\delta>0$ there exists $C(a,\delta)$ such that the following holds.
  For any random variable $X$ that is $\tau^2$-sub-Gaussian with $\E\{X^2\}=1$, we have
  \begin{align}
    \E\{|X|^a\}\le C(a,\delta)\, \tau^{(a-2+\delta)_+}\, .
  \end{align}
  This inequality holds for $a\le 2$, with $\delta=0$ and $C=1$.
  \end{lemma}
  \begin{proof}
    Without loss of generality, we can assume $X\ge 0$.
    For $a\le 2$, this is just Jensen's inequality.

    For $a>2$, by H\"older for any $q\in [0,a]$ and any $r>1$, we have
    \begin{align*}
      \E[X^a] \le \E[X^{rq}]^{1/r}\E[X^{r(a-q)/(r-1)}]^{(r-1)/r}\, .
    \end{align*}
    Setting $q=2/r$, $r\ge 2/a$ and using the fact that $\E[X^2] = 1$ and $X$ is sub-Gaussian, we get
    that there exist constants $C_0(a,r)$ finite as long as $r>1$, such that
    \begin{align*}
      \E[X^a] \le \E[X^{(ra-2)/(r-1)}]^{(r-1)/r} \le C_0(a,r) \tau^{a-2/r}\, .
    \end{align*}
    The claim follows by setting $2/r=2-\delta$ or $r=(1-\delta/2)^{-1}>1$. 
  \end{proof}

\begin{lemma}
   For any $q_1,q_2>0$
and any $\delta>0$ there exists $C(q_1,q_2,\delta)$ such that the following holds.
  For any random variable $X$ that is $\tau^2$-sub-Gaussian with $\E\{X^2\}=1$, we have
  \begin{align}
    \E\{|X|^{q_1}\wedge |X|^{q_2}\}\ge C(q_1,q_2,\delta)\, \tau^{-(2-q_1\wedge q_2+\delta)_+}\, .
  \end{align}
  This inequality holds for $q_1= q_2\ge 2$, with $\delta=0$ and $C=1$.
  \end{lemma}
  \begin{proof}
    Without loss of generality, we can assume $X\ge 0$.
    For $q_1=q_2\ge 2$, this is just Jensen's inequality.

    For the other cases, notice that, for $a_1,a_2\in [0,2]$, 
    $(x^{a_1}\wedge x^{a_2}) (x^{2-a_1}\vee x^{2-a_2})=x^2$. Hence, by H\"older inequality,
    for all $r>1$:
    \begin{align*}
      1=\E[X^2]\le \E[X^{a_1r}\wedge X^{a_2r}]^{1/r}  \E[X^{(2-a_1)r/(r-1)}\vee
      X^{(2-a_2)r/(r-1)}]^{(r-1)/r}\, ,
    \end{align*}
    We then set $a_i=q_i/r$, and invert this relationship to get
    \begin{align*}
      \E[X^{q_1}\wedge X^{q_2}]&\ge  \E[X^{(2r-q_1)/(r-1)}\vee
                                 X^{(2r-q_2)/(r-1)}]^{-(r-1)}\\
                               &\ge  \Big(\E[X^{(2r-q_1)/(r-1)}]+   \E[ X^{(2r-q_2)/(r-1)}]\Big)^{-(r-1)}\, .
    \end{align*}
    Taking $r\ge 1\vee (q_1/2)\vee(q_2/2)$, we can apply Lemma \ref{lemma:SimpleSubg},
    to get
    \begin{align*}
      \E[X^{q_1}\wedge X^{q_2}]&\ge
                                 C\Big(\tau^{(\frac{2r-q_1}{r-1}-2+\delta')_+(r-1)}+\tau^{(\frac{2r-q_2}{r-1}-2+\delta')_+(r-1)}\Big)^{-1}\\
                               &\ge C \Big(\tau^{(2-q_1+\delta)_+}+\tau^{(2-q_2+\delta )_+}\Big)^{-1}\\
                               &\ge C \tau^{-(2-q_1\wedge q_2+\delta)_+}\, .
    \end{align*}
    This proves the claim.
    \end{proof}

\end{document}